\documentclass[onefignum,onetabnum]{siamonline171218}

\def\jmpu{{\lbrack\!\lbrack \widetilde{\gamma}> \\
  <g,\widetilde{\Balpha}> \\
<g,\widetilde{\Bbeta}> \end{bmatrix}$, whereu\rbrack\!\rbrack}}

\def\sC{{\mathcal C}}

\def\sN{{\mathcal N}}
\def\sF{{\mathcal F}}

\def\sJ{{\mathcal J}}

\def\sD{{\mathcal D}}

\def\sA{{\mathcal A}}
\def\sQ{{\mathcal Q}}
\DeclareMathOperator*{\argmax}{argmax}
\def\Balpha{\mbox{\boldmath$\alpha$}}
\def\Bbeta{\mbox{\boldmath$\beta$}}

\def\Btheta{\mbox{\boldmath$\theta$}}

\def\btheta{\boldsymbol{\theta}} 
\def\bphi{\boldsymbol{\phi}} 
\def\Bbeta{\boldsymbol{\beta}}

\newcommand{\LRc}[1]{\left\{ #1 \right\}}
\newcommand{\LRp}[1]{\left( #1 \right)}
\newcommand{\LRs}[1]{\left[ #1 \right]}
\newcommand{\norm}[1]{\left \Vert #1 \right \Vert}
\newcommand{\real}{{\mathbb{R}}}
\newcommand{\mb}[1]{\mathbf{#1}}
\newcommand{\Ex}{{\mathbb{E}}}
\newcommand{\expect}{\Ex}

\def\bC{\mb{C}}

\def\bff{\mb{f}}
\def\bg{\mb{g}}

\def\bQ{\mb{Q}}

\def\bW{\mb{W}}
\def\bQ{\mb{Q}}
\def\bX{\mb{X}}

\def\bb{\mb{b}}
\def\bc{\mb{c}}

\def\bff{\mb{f}}
\def\bg{\mb{g}}

\def\bp{\mb{p}}

\def\bu{\mb{u}}
\def\bv{\mb{v}}

\def\bx{\mb{x}}
\def\by{\mb{y}}
\def\bz{\mb{z}}

\newcommand{\bs}{\boldsymbol}
\newcommand{\vb}{\bs{v}}
\renewcommand{\epsilon}{\varepsilon}

\usepackage{lipsum}
\usepackage{amsfonts}
\usepackage{graphicx}
\usepackage{epstopdf}
\ifpdf
  \DeclareGraphicsExtensions{.eps,.pdf,.png,.jpg}
\else
  \DeclareGraphicsExtensions{.eps}
\fi

\usepackage{enumitem}
\setlist[enumerate]{leftmargin=.5in}
\setlist[itemize]{leftmargin=.5in}

\newsiamremark{remark}{Remark}
\newsiamremark{hypothesis}{Hypothesis}
\crefname{hypothesis}{Hypothesis}{Hypotheses}
\newsiamthm{claim}{Claim}

\headers{Deep Neural Network Architecture Adaptation}{C G Krishnanunni, Tan Bui-Thanh, and Clint Dawson}

\title{%A Sequential Sparse Layerwise %Sparsifying Training and Sequential 
 Topological derivative approach for deep neural network architecture adaptation
}

\author{C G Krishnanunni\thanks{Dept of Aerospace Engineering $\&$ Engineering Mechanics, UT  Austin
  (\email{krishnanunni@utexas.edu}).}
\and Tan Bui-Thanh\thanks{Dept of Aerospace Engineering $\&$ Engineering Mechanics, Oden Institute for Computational Engineering $\&$ Sciences, UT  Austin 
  (\email{tanbui@utexas.edu}, \email{clint.dawson@austin.utexas.edu}).}
\and Clint Dawson\footnotemark[2]}

\usepackage{amsopn}

\makeatletter
\newcommand*{\addFileDependency}[1]{% argument=file name and extension
  \typeout{(#1)}% latexmk will find this if $recorder=0 (however, in that case, it will ignore #1 if it is a .aux or .pdf file etc and it exists! if it doesn't exist, it will appear in the list of dependents regardless)
  \@addtofilelist{#1}% if you want it to appear in \listfiles, not really necessary and latexmk doesn't use this
  \IfFileExists{#1}{}{\typeout{No file #1.}}% latexmk will find this message if #1 doesn't exist (yet)
}
\makeatother

\newcommand*{\myexternaldocument}[1]{%
    \externaldocument{#1}%
    \addFileDependency{#1.tex}%
    \addFileDependency{#1.aux}%
}
\usepackage{algorithm} %for writing pseudoalgorithms
\usepackage{algpseudocode} %for writing pseudoalgorithms
\usepackage{ntheorem}
\usepackage{tablefootnote}
\usepackage{mathtools}
\usepackage{amssymb}
\usepackage{tikz}
\usetikzlibrary{patterns}
\usetikzlibrary{arrows,shapes}
\usetikzlibrary{decorations.pathreplacing,calligraphy}

\usepackage{comment}

\newenvironment{proof}{\paragraph{\textit{Proof:}}}{}

\myexternaldocument{ex_supplement}

\newcounter{inlineenum}
\renewcommand{\theinlineenum}{\alph{inlineenum}}
\newenvironment{inlineenum}
  {\unskip\ignorespaces\setcounter{inlineenum}{0}%
   \renewcommand{\item}{\refstepcounter{inlineenum}{\textit{\theinlineenum})~}}}
  {\ignorespacesafterend}

\begin{document}
\maketitle
\begin{abstract}
 This work presents a novel algorithm for progressively adapting neural network architecture along the depth. In particular, we attempt to address the following questions in a mathematically principled way: i)  Where to add a new capacity (layer) during the training process? ii) How to initialize the new capacity? At the heart of our approach are two key ingredients: i) the introduction of a ``shape functional" to be minimized, which depends on neural network topology, and ii) the introduction of %The algorithm defining the 
a topological derivative of the shape functional %It is conceptually the derivative of a shape functional 
with respect to %infinitesimal changes in 
the neural network topology. Using an optimal control viewpoint, we show that the network topological derivative exists under certain conditions, and its closed-form expression is derived. In particular,  we explore, for the first time, the connection between the topological derivative from a topology optimization framework with the Hamiltonian from optimal control theory. Further, we show that 
%the derived expression for the topological derivative 
 the optimality condition for the shape functional
 leads to an eigenvalue problem for deep neural architecture adaptation. %The algorithm we derived 
 Our approach thus determines the most sensitive location along the depth where a new layer needs to be inserted during the training phase and the associated parametric initialization for the newly added layer.  We also demonstrate that our layer insertion strategy can be derived from an optimal transport viewpoint as a solution to maximizing a topological derivative in $p$-Wasserstein space, where $p\geq 1$. Numerical investigations with fully connected network, convolutional neural network, and vision transformer on various regression and classification problems demonstrate that our proposed approach can outperform an ad-hoc baseline network and other architecture adaptation strategies. Further, we also demonstrate other applications of topological derivative in fields such as transfer learning.
\end{abstract}

% REQUIRED

% REQUIRED

\section{Introduction}
It has been observed that deep neural networks (DNNs) create increasingly
simpler but more useful  representations  of the learning problem layer by layer \cite{hinton2007learning, montavon2010layer,
  montavon2011kernel, zeiler2014visualizing}. Furthermore, 
empirical evidence supports the paradigm that  depth of a network is of crucial
importance \cite{russakovsky2015imagenet, simonyan2014very,
  szegedy2015going,krishnanunni2022layerwise}.
 Some of the problems associated with training such deep networks include: %\cite{low2019stacking}:  
i) a possible large training set is needed to overcome the over-fitting issue; ii) the architecture adaptability problem, e.g., any amendments to a pre-trained DNN, requires retraining even with transfer learning; iii) GPU employment is almost  mandatory due to %the %huge size of the
massive network and data sizes. 
Most importantly, it is often unclear on the number of layers and number of neurons to be used in each layer while training
a neural network for a specific task.  Therefore, there is a critical need for rigorous adaptive principles to guide the architecture design of a neural network.

\subsection{Related work}

Neural architecture adaptation algorithms can be broadly classified into two categories: i) neural architecture search algorithms and ii) principled adaptive strategies. 
Neural architecture search (NAS) algorithms rely on metaheuristic optimization,  reinforcement learning strategy, or Bayesian hyperparameter optimization to arrive at a reasonable architecture \cite{zoph2016neural, stanley2002evolving, suganuma2017genetic, elsken2018efficient, real2019regularized, balaprakash2018deephyper, miikkulainen2019evolving, liu2021survey,li2020random,kandasamy2018neural}. However, these strategies involve training and evaluating many candidate architectures (possibly deep) in the process and are thus computationally expensive. {One key issue with NAS algorithms is that most of the methods report the performance of the best-found architecture, presumably resulting from a single run of the search process \cite{li2020random}. However,  random initializations of each candidate architecture could drastically influence the choice of best architecture prompting the use of multiple runs to select the best architecture, and hence prohibitively expensive. }
On the other hand, sensible adaptive strategies  are algorithms for growing neural networks where one
starts by training a small network and progressively increasing the size of the network (width/depth) \cite{wu2019splitting, wynne1993node,wu2020firefly,evci2022gradmax,krishnanunni2022layerwise,chen2015net2net}.

In particular, existing works have considered growing the width gradually by adding neurons for a fixed depth neural network \cite{wu2019splitting, wynne1993node,maile2022and,evci2022gradmax}. Liu et al. \cite{wu2019splitting} developed a simple criterion for
deciding the best subset of neurons to split and a splitting gradient for optimally updating the off-springs.  Wynne-Jones \cite{wynne1993node}  considered splitting the neurons (adding neurons) based on a principal component analysis on the oscillating weight vector.  Chen et al. \cite{chen2015net2net} showed that replacing a network with an equivalent network that is wider (has more neurons in
each hidden layer) allows the equivalent network to inherit the knowledge from the existing one and can be trained to further improve the performance. Firefly algorithm by Wu et al. \cite{wu2020firefly} generates candidate neurons that either split existing
neurons with noise or are completely new and selects those with the highest gradient norm.

On the other hand, many efforts have been proposed for growing neural architecture along the depth \cite{hettinger2017forward,kulkarni2017layer,wen2020autogrow}. Layerwise training of neural networks is an approach that addresses the issue of the choice of depth of a neural network and the computational complexity involved with training \cite{xu1999training}.   Hettinger et al. \cite{hettinger2017forward} showed that layers can be trained one at a time and the resulting DNN can generalize better. Bengio et al.  \cite{bengio2007greedy} proposed a greedy layerwise unsupervised learning algorithm where the initial layers of a network are supposed to represent more abstract concepts that explain the input observation, whereas subsequent layers extract low-level features. Recently, we devised a manifold regularized greedy layerwise training approach for adapting a neural architecture along its depth \cite{krishnanunni2022layerwise}. Net2Net algorithm \cite{chen2015net2net} introduced the concept of function-preserving transformations for rapidly transferring the information stored in one
neural network into another neural network. Wei et al. \cite{wen2020autogrow} proposed the AutoGrow algorithm to automate depth discovery in deep neural networks where new layers are inserted progressively
if the accuracy improves; otherwise, stops growing and
thus discovers the depth.

It is important to note that any algorithm for growing neural networks (width and/or depth) should adequately address the following questions in a mathematically principled way rather than a heuristic approach \cite{evci2022gradmax}: {\bf{When}} to add new capacity (neurons/layers)?; {\bf{Where}} to add new capacity?; {\bf{How}} to initialize the new capacity? 
Most existing efforts to address these questions focus on growing networks in width, i.e., by adding neurons \cite{wu2019splitting, wu2020firefly, wynne1993node, evci2022gradmax, maile2022and}. {When it comes to growing neural architectures in depth, relatively little work adequately addresses these questions. Sensli \cite{kreis2023sensli} extends the Net2Net framework \cite{chen2015net2net} by addressing not only how to initialize a new layer, but also where to insert it within the network. 
However, in Sensli \cite{kreis2023sensli}, the initialization of a new layer is independent of both the data and the location at which the new layer is inserted in the network. We hypothesize that data-dependent, location-dependent initialization of the added layer is a crucial component of such adaptation strategies, contributing to improved generalization rather than focusing solely on where the new layer is inserted.}  

\subsection{Our contributions}

In this work, we derive an algorithm for progressively increasing a
neural network’s depth inspired by topology optimization.
%in a mathematically principled fashion. 
In particular, we provide solutions to the following questions: i) Where to add a new layer; ii)  When to add a new layer; and iii) How to initialize the new layer? 
{We will show that our approach leads to a data-dependent, position dependent initialization of an added layer.} We shall present two versions of the algorithm: i) a semi-automated version where a predefined scheduler is used to decide when a new layer needs to be added during the training phase  \cite{evci2022gradmax}; ii) a fully automated growth process in which a validation metric is employed to automatically detect when a new layer needs to be added. Our method does not have the limitation that the loss function needs to plateau before growing such as those in \cite{wu2019splitting,kilcher2018escaping}. Our algorithm is based on deriving a closed-form expression for the topological derivative for a shape functional with respect to a neural network (\cref{exist_th}). To that end, in \cref{admissible_pert} and \cref{prop_admissible} we introduce the concept of admissible perturbation and suggest ways for constructing an admissible perturbation. In \cref{deriv_algo}, we show that the first-order optimality condition leads to an eigenvalue problem whose solution determines where a new layer needs to be inserted along with the associated parametric initialization for the added layer. 
The efficiency of our proposed algorithm is explored in \cref{experim} by carrying out extensive numerical investigations with different architectures on prototype regression and classification problems.  

\section{Mathematical framework}

In this work, all the matrices and vectors are represented in boldface. 
Consider a regression/classification task where one is provided with $S$ training data points, input data dimension $n_0$, and label dimension $n_T$. Let the inputs $\bx_i \in \mathbb{R}^{n_0}$ 
for $i \in \{1,2,...S\}$ be organized
row-wise into a matrix $\bX \in \mathbb{R}^{S\times {n_0}}$ and let the corresponding true labels be denoted as $\bc_i\in \mathbb{R}^{n_T}$ and  stacked row-wise as $\bC \in \mathbb{R}^{S\times {n_T}}$.
We start by considering the following empirical risk minimization problem (neural network training) with some loss function\footnote{More precisely the notation $\Phi_s$ may be used. } $\Phi$  for a typical feed-forward neural network:
\begin{equation}
    \begin{aligned}
        \hspace{0 cm} \min_{\boldsymbol{\theta}\in \Theta} \sJ(\Btheta)= \frac{1}{S}\sum_{s=1}^S \Phi\LRp{\bx_{s,T}},\quad \text{subject to:} \ \ \bx_{s,t+1}=\bff_{t+1}(\bx_{s,t};\ \Btheta_{t+1}), \ \ \ \bx_{s,0}=\bx_s,\\
        %  & \text{subject to:} \ \ \bx_{s,t+1}=f_{t+1}(\bx_{s,t},\ \Btheta_{t+1}),  \quad  t=0,\dots T-1, \quad s\in \{1,\dots S\},
    \end{aligned}
    \label{training_problem}
\end{equation}
where  $t=0,\dots T-1$, $ s\in \{1,\dots S\}$,  $\bff_{t+1}: \real^{n(t)}\times \Theta_{t+1}\mapsto \real^{n(t+1)}$ 
denotes the forward propagation function, $n(t)$ denotes the number of neurons in the $t$-th  layer, the parameter set $\Theta_{t+1}$ is a subset of a Euclidean space and  $\Theta=\Theta_1\times \dots \times \Theta_T$,
$\Btheta_{t+1}$ represents the network parameters corresponding to $(t+1)^{th}$ layer, $\boldsymbol{\theta}=\LRp{\Btheta_1,\dots,\Btheta_T}$ such that $\Btheta \in \Theta$ implies  $\Btheta_{t+1}\in \Theta_{t+1}$, $\bx_{s,t}$ represents the hidden states of the network,  $T-1$ is the total number of hidden layers in the network.  In the case of a fully connected network (FNN), one  has:
\[\bff_{t+1}(\bx_{s,t};\ \Btheta_{t+1})=\sigma_{t+1} \LRp{\bW_{t+1}\bx_{s,t}+\bb_{t+1}},\]
where $\bW_{t+1} \in \real^{n(t+1)\times n(t)}$, $\bb_{t+1} \in \real^{n(t+1)}$, and $\sigma_{t+1}$ is a nonlinear activation function acting component-wise on its arguments. In the following sections, for clarity, let us limit our discussion to the case of a fully-connected neural network.

\subsection{Motivation}
\label{motiv}

{Note that during the training of the network in \eqref{training_problem}, it is often desirable to dynamically add new layers (with  specific initialization for the parameters), at specific locations in the architecture \cite{chen2015net2net,hettinger2017forward}. Our objective in this work is to develop a mathematically principled criterion that guides such layer additions in a network during the training process.}

{Our framework is based on the concept of topological derivative which
was formally introduced  as the "bubble method" for the
optimal design of structures \cite{eschenauer1994bubble}, and further studied in detail by Sokolowski and Zochowski \cite{sokolowski1999topological}. The topological derivative is, conceptually, a derivative of a shape functional with respect to infinitesimal changes in its topology, such as adding an infinitesimal hole or crack.
The concept finds enormous application in the field of structural mechanics, image processing and inverse problems \cite{amstutz2022introduction}. In the field of structural mechanics, the concept has been used to detect and locate cracks for a simple model problem: the steady-state heat equation with the heat flux imposed and the temperature measured on the boundary \cite{amstutz2005crack}. In the field of image processing, the topological derivative has been used to derive a non-iterative algorithm to perform edge detection and image restoration by studying the impact of an insulating crack in the domain \cite{belaid2006image}. In the field of inverse problems, the topological derivative approach has been applied to tomographic reconstruction problem \cite{auroux2010application}. 
}

\def\layersep{1.4cm}
\def\nodeinlayersep{0.7cm}
\begin{figure}[H]
\centering
\begin{tikzpicture}[
    node distance=\layersep,
    edge/.style={-stealth,shorten >=1pt, draw=black!50,thin},
    neuron/.style={circle,fill=black!25,minimum size=10pt,inner sep=0pt},
    operator/.style={rectangle,fill=green!,minimum height= \nodeinlayersep, minimum width= 0.8 * \layersep, inner sep=0pt, rounded corners},
    input neuron/.style={neuron, fill=green!50,minimum size=12pt},
    output neuron/.style={neuron, fill=green!50,minimum size=12pt},
    hidden neuron/.style={neuron, fill=blue!50},
    Forward map/.style={operator, fill=red!50},
    annot/.style={text width=4em, text centered},
    every node/.style={scale=1.0},
    node1/.style={scale=2.0}
]
    % Draw the input layer nodes
    \foreach \name / \y in {1,...,4}
        {\ifnum \y=3
            \node (I-\name) at (0,-\nodeinlayersep * \y - \nodeinlayersep * 0.5) {$\vdots$};
            % \node[input neuron] (I-\name) at (0,-\y-0.4) {$y_{\y}$};
        \else
            \ifnum \y=4
                \node[input neuron] (I-\name) at (0,-\nodeinlayersep * \y - \nodeinlayersep * 0.5) {$x_{n}$};
            \else
                \node[input neuron] (I-\name) at (0,-\nodeinlayersep *\y - \nodeinlayersep * 0.5 ) {$x_{\y}$};
            \fi
        \fi}
        
    % Draw the output layer node
    \foreach \name / \y in {1,...,4}
        {\ifnum \y=3
            \node (O-\name) at (3*\layersep,-\nodeinlayersep *\y - \nodeinlayersep * 0.5) {$\vdots$};
            % \node[input neuron] (I-\name) at (0,-\y-0.4) {$y_{\y}$};
        \else
            \ifnum \y=4
                \node[input neuron] (O-\name) at (3*\layersep,-\nodeinlayersep *\y - \nodeinlayersep * 0.5) {$y_{p}$};
            \else
                \node[input neuron] (O-\name) at (3*\layersep,-\nodeinlayersep *\y - \nodeinlayersep * 0.5) {$y_{\y}$};
            \fi
        \fi}

    % set number of hidden layers
    \newcommand \Nhidden{2}
    % Draw the hidden layer nodes
    \foreach \N in {1,...,\Nhidden} {
        \foreach \y in {1,...,5} { %%% MODIFIED (1,...,12 -> 1,...,5, and the next five lines)
            \ifnum \y=4
                \node at (\N*\layersep,-\y*\nodeinlayersep) {$\vdots$};
            \else
                \node[hidden neuron] (H\N-\y) at (\N*\layersep,-\y*\nodeinlayersep ) {$\sigma$};
        \fi
      }
    }

    %%% <-- MODIFIED (from H\Nhidden-6 to H\Nhidden-3) 
    % Connect every node in the input layer with every node in the
    % hidden layer.
    \foreach \source in {1,2,4}
        \foreach \dest in {1,...,3,5} %%% <-- MODIFIED (1,...,12 -> 1...,3,5)
            \draw[edge] (I-\source) -- (H1-\dest);
    
    % connect all hidden stuff
    \foreach [remember=\N as \lastN (initially 1)] \N in {2,...,\Nhidden}
      \foreach \source in {1,...,3,5} %%% <-- MODIFIED (1,...,12 -> 1...,3,5)
          \foreach \dest in {1,...,3,5} %%% <-- MODIFIED (1,...,12 -> 1...,3,5)
              \draw[edge] (H\lastN-\source) -- (H\N-\dest);
              
    % Connect every node in the hidden layer with the output layer
    \foreach \source in {1,...,3,5} %%% <-- MODIFIED (1,...,12 -> 1...,3,5)
        \foreach \dest in {1,2,4}
            \draw[edge] (H\Nhidden-\source) -- (O-\dest);

    \draw [decorate, decoration = {calligraphic brace, mirror}, thick]  (1*\layersep,-.7*\nodeinlayersep) -- (1*\layersep,-.7*\nodeinlayersep) node[pos=1.5,below=3.5cm,black]{$l=1$};

    \draw [decorate, decoration = {calligraphic brace, mirror}, thick]  (2*\layersep,-.7*\nodeinlayersep) -- (2*\layersep,-.7*\nodeinlayersep) node[pos=1.5,below=3.5cm,black]{$l=2$};
       \draw [decorate, decoration = {calligraphic brace, mirror}, thick]  (1.5*\layersep,-.7*\nodeinlayersep) -- (1.5*\layersep,-.7*\nodeinlayersep) node[pos=1.5,below=4cm,black]{Original network $\Omega_0$};

      % \draw [dashed] (5.5,-4.8) -- (5.5,0.5);

\end{tikzpicture}
\hspace{2 cm}
\begin{tikzpicture}[
    node distance=\layersep,
    edge/.style={-stealth,shorten >=1pt, draw=black!50,thin},
    neuron/.style={circle,fill=black!25,minimum size=10pt,inner sep=0pt},
    operator/.style={rectangle,fill=green!,minimum height= \nodeinlayersep, minimum width= 0.8 * \layersep, inner sep=0pt, rounded corners},
    input neuron/.style={neuron, fill=green!50,minimum size=12pt},
    output neuron/.style={neuron, fill=green!50,minimum size=12pt},
    hidden neuron/.style={neuron, fill=blue!50},
    Forward map/.style={operator, fill=red!50},
    annot/.style={text width=4em, text centered},
    every node/.style={scale=1.0},
    node1/.style={scale=2.0}
]
    % Draw the input layer nodes
    \foreach \name / \y in {1,...,4}
        {\ifnum \y=3
            \node (I-\name) at (0,-\nodeinlayersep * \y - \nodeinlayersep * 0.5) {$\vdots$};
            % \node[input neuron] (I-\name) at (0,-\y-0.4) {$y_{\y}$};
        \else
            \ifnum \y=4
                \node[input neuron] (I-\name) at (0,-\nodeinlayersep * \y - \nodeinlayersep * 0.5) {$x_{n}$};
            \else
                \node[input neuron] (I-\name) at (0,-\nodeinlayersep *\y - \nodeinlayersep * 0.5 ) {$x_{\y}$};
            \fi
        \fi}
        
    % Draw the output layer node
    \foreach \name / \y in {1,...,4}
        {\ifnum \y=3
            \node (O-\name) at (4*\layersep,-\nodeinlayersep *\y - \nodeinlayersep * 0.5) {$\vdots$};
            % \node[input neuron] (I-\name) at (0,-\y-0.4) {$y_{\y}$};
        \else
            \ifnum \y=4
                \node[input neuron] (O-\name) at (4*\layersep,-\nodeinlayersep *\y - \nodeinlayersep * 0.5) {$y_{p}$};
            \else
                \node[input neuron] (O-\name) at (4*\layersep,-\nodeinlayersep *\y - \nodeinlayersep * 0.5) {$y_{\y}$};
            \fi
        \fi}

    % set number of hidden layers
    \newcommand \Nhidden{3}
    % Draw the hidden layer nodes
    \foreach \N in {1,...,\Nhidden} {
        \foreach \y in {1,...,5} { %%% MODIFIED (1,...,12 -> 1,...,5, and the next five lines)
            \ifnum \y=4
                \node at (\N*\layersep,-\y*\nodeinlayersep) {$\vdots$};
            \else
                \node[hidden neuron] (H\N-\y) at (\N*\layersep,-\y*\nodeinlayersep ) {$\sigma$};
        \fi
      }
    }

    %%% <-- MODIFIED (from H\Nhidden-6 to H\Nhidden-3) 
    % Connect every node in the input layer with every node in the
    % hidden layer.
    \foreach \source in {1,2,4}
        \foreach \dest in {1,...,3,5} %%% <-- MODIFIED (1,...,12 -> 1...,3,5)
            \draw[edge] (I-\source) -- (H1-\dest);
    
    % connect all hidden stuff
    \foreach [remember=\N as \lastN (initially 1)] \N in {2,...,\Nhidden}
      \foreach \source in {1,...,3,5} %%% <-- MODIFIED (1,...,12 -> 1...,3,5)
          \foreach \dest in {1,...,3,5} %%% <-- MODIFIED (1,...,12 -> 1...,3,5)
              \draw[edge] (H\lastN-\source) -- (H\N-\dest);
              
    % Connect every node in the hidden layer with the output layer
    \foreach \source in {1,...,3,5} %%% <-- MODIFIED (1,...,12 -> 1...,3,5)
        \foreach \dest in {1,2,4}
            \draw[edge] (H\Nhidden-\source) -- (O-\dest);

    \draw [decorate, decoration = {calligraphic brace, mirror}, thick]  (1*\layersep,-.7*\nodeinlayersep) -- (1*\layersep,-.7*\nodeinlayersep) node[pos=1.5,below=3.5cm,black]{$l=1$};

    \draw [decorate, decoration = {calligraphic brace, mirror}, thick]  (2*\layersep,-.7*\nodeinlayersep) -- (2*\layersep,-.7*\nodeinlayersep) node[pos=1.5,below=3.5cm,black]{$l=1.5$};

        \draw [decorate, decoration = {calligraphic brace, mirror}, thick]  (3*\layersep,-.7*\nodeinlayersep) -- (3*\layersep,-.7*\nodeinlayersep) node[pos=1.5,below=3.5cm,black]{$l=2$};
        
    \draw [decorate, decoration = {calligraphic brace, mirror}, thick]  (1.5*\layersep,-.7*\nodeinlayersep) -- (1.5*\layersep,-.7*\nodeinlayersep) node[pos=1.5,above=0.5cm,black]{Added layer with };

    \draw [decorate, decoration = {calligraphic brace, mirror}, thick]  (1.5*\layersep,-.7*\nodeinlayersep) -- (1.5*\layersep,-.7*\nodeinlayersep) node[pos=1.5,above=0.1cm,black]{ parameters $\epsilon \boldmath\phi$};

    \draw [decorate, decoration = {calligraphic brace, mirror}, thick]  (2.0*\layersep,-.7*\nodeinlayersep) -- (2.0*\layersep,-.7*\nodeinlayersep) node[pos=1.8,below=4cm,black]{ Perturbed network $\Omega_{\epsilon}$};

%\draw [decorate,decoration={brace,amplitude=5pt,mirror,raise=4ex}]  (2.7,-1.1) -- (1.5,-1.1);
\draw[->, thick] (2,-0.3) -- (2,-0.65);
%\draw [|<->|] (1.5,.4) -- node[above=1mm] {5mm} (2,.4);
\end{tikzpicture}
\caption{Schematic view of the topological derivative approach: A new layer with parameters $\epsilon \bphi$ is inserted between the $1^{st}$ and $2^{nd}$ layer. If one views $\epsilon$ as the magnitude of perturbation, then for $\epsilon=0$, the network $\Omega_\epsilon$ should behave exactly the same way as $\Omega_0$ under the standard training process (Residual connections are not shown in the figure).} 
\label{fig:topo}
\end{figure}
{In the topological derivative approach, one often studies how the presence of a circular hole (or inclusion of new material) of radius $\epsilon$  at a specific location $\bz$ in the domain $\Omega_0\subset \real^2$ affects
the solution of a given partial differential equation (PDE) with prescribed boundary conditions \cite{amstutz2022introduction}.}

{Analogously, in the context of neural networks, we will examine  how the addition of a new layer initialized with parameters $\epsilon \bphi$,  at a specific location $l\in \{1,2,\dots T-1\}$ in the neural network $\Omega_0$ affects the training equations. Here, $\epsilon \in \real$ and $\bphi$ is the vector of parameters of the new layer. 
\cref{fig:topo} shows an example where a new layer with parameters $\epsilon \bphi$ is inserted at the location between layers 1 and 2 of the  network $\Omega_0$. This raises the following question:}
\vspace{0.2 cm}

{{\it{In what sense does adding a layer constitute a perturbation of the neural network graph?}} }

\vspace{0.2 cm}

{We answer this question by drawing analogy with the concept of ``topological derivative" from mechanics\footnote{{The analysis resulting from this viewpoint differs from other works on perturbation analysis in the machine learning community, where the focus is either (i) on how perturbations to input data affect the output of a neural network \cite{fawzi2016robustness,rahnama2020robust}, or (ii) on how perturbations to network parameters influence the network's performance or output \cite{cheney2017robustness,tsai2021formalizing,buchel2021network,yu2023adversarial}.}}. Since we are interested in understanding how the addition of a layer, as illustrated in \cref{fig:topo}, affects the neural network training, it is important to first understand the gradient flow/backpropagation equations for training a neural network. 
To that end, in \cref{opt_top_full} we present the neural network training problem from an optimal control viewpoint \cite{li2018optimal, benning2019deep}. A key quantity that arises in this framework is the Hamiltonian $H_t$ (see \eqref{hamiltonian}). Later in \cref{exist_th}, we unveil an interesting connection between the network topological derivative, formally defined in \cref{topo_defin}, and the Hamiltonian $H_t$ in \eqref{hamiltonian}.}

\subsection{Optimal control viewpoint for neural network training}
\label{opt_top_full}

Recently, there has been an interest in formulating deep learning as a  discrete-time
optimal control problem \cite{li2018optimal, benning2019deep}. Note that in the context of popular architectures such as residual neural networks (ResNet), equation \eqref{training_problem} can be interpreted as discretizations of an optimal control problem subject to an ordinary differential equation constraint \cite{benning2019deep}.
The Hamiltonian  for the $t$-th layer, $H_t: \real^{n(t)}\times \real^{n(t+1)}\times \Theta_{t+1} \mapsto \real$ corresponding to the loss in \eqref{training_problem} is defined as \cite{todorov2006optimal}: 
\begin{equation}
    H_t\LRp{\bx_{s,t};\ \bp_{s,t+1};\ \Btheta_{t+1}}=\bp_{s,t+1} \cdot \bff_{t+1}(\bx_{s,t};\ \Btheta_{t+1}),
    \label{hamiltonian}
\end{equation}
where, $\{ \bp_{s,1},\ \bp_{s,2},\dots \bp_{s,T} \}$ denote the adjoint variables computed during a particular gradient descent iteration (backpropagation) as follows \cite{todorov2006optimal}:
\begin{equation}
    \begin{aligned}
   & \underline{\mathrm{Forward\ propagation}}\\ 
   & \bx_{s,t+1}=\nabla_\bp H_t\LRp{\bx_{s,t};\bp_{s,t+1};\Btheta_{t+1}}=\bff_{t+1}(\bx_{s,t};\ \Btheta_{t+1}),\\
   & \underline{\mathrm{Computing \ adjoints}}\\
       & \bp_{s,T}=-\frac{1}{S}\nabla \Phi\LRp{\bx_{s,T}},\\
         & \bp_{s,t}=\nabla_{\bx} H_t\LRp{\bx_{s,t};\ \bp_{s,t+1};\ \Btheta_{t+1}}= \LRs{\nabla_{\bx}\bff_{t+1}(\bx_{s,t};\Btheta_{t+1})}^T\ \bp_{s,t+1},\ \ t=0,\dots T-1,\\
       & \underline{\mathrm{Updating \ parameters \ via\ gradient\ descent}}\\
& \Btheta_{t+1}\leftarrow\Btheta_{t+1}-\ell\ \nabla_{\btheta} H_t\LRp{\bx_{s,t};\ \bp_{s,t+1};\ \Btheta_{t+1}}=\Btheta_{t+1}-\ell\ \LRs{\nabla_{\btheta}\bff_{t+1}(\bx_{s,t};\Btheta_{t+1})}^T\bp_{s,t+1},
\end{aligned}
    \label{set_of}
\end{equation}
where, $\nabla_{\bx}\bff_{t+1}$ denotes the gradient of $\bff_{t+1}$ with respect to the first argument (i.e. the state $\bx$) and $\nabla_{\btheta}\bff_{t+1}$ denotes the gradient of $\bff_{t+1}$ with respect to the second argument (i.e. the parameter $\btheta$), $\ell$ denotes the step-size used in the particular gradient descent iteration, notation {`$\leftarrow$' means that the parameters $\Btheta_{t+1}$ of the $(t+1)-$th layer are being updated.}
From a  Lagrangian viewpoint, the Hamiltonian can be understood as a device to generate the first-order necessary conditions for optimality \cite{todorov2006optimal}. 

{With the concept of the topological derivative introduced in \cref{motiv}, and backpropagation described in \cref{opt_top_full},  we are now ready to introduce key concepts such as ``perturbed network" and ``admissible perturbation"  leading to the definition of a discrete topological derivative which we call as the ``network topological derivative" in our work. The concept of an “admissible perturbation” is based on the idea that, if $\epsilon = 0$ in \cref{fig:topo} (i.e., the perturbation magnitude), the solutions in \cref{set_of} should correspond to those of the original network $\Omega_0$ at any gradient descent iteration, as if the added layer were not present in the graph. This is analogous to the topological derivative approach in mechanics, where setting $\epsilon = 0$ (the radius of the hole) recovers the solution of the PDE on the original domain $\Omega_0$.}

\begin{definition}[Network perturbation]
\label{def_pert}
Consider $\sA$ as the set of all feed-forward neural networks with constant width `n' ($n \in \mathbb{N}$) in each hidden layer. %across  hidden layers  and arbitrary depth. 
Let $\Omega_0 \in \sA$ be a given neural network with $(T-1)$ hidden layers, and 
$\sigma\LRp{.}:\real\mapsto \real$ be a non-linear activation function. A perturbation $\Omega_\epsilon$  from  $\Omega_0$ at $l\in \{1,2,\dots T-1\}$ along the ``direction"   $\boldsymbol{\phi}$ and with magnitude $\epsilon$ is defined as:
\[ \Omega_{\epsilon}={\Omega_0\oplus (l,\ \epsilon \boldsymbol{\phi}, \ \sigma)}, \quad s.t.\ \ \Omega_\epsilon \in \sA,\]
where $\oplus$ represents the operation of adding a layer between  the $l^{th}$ and $(l+1)^{th}$ layer of network $\Omega_0$ initialized with vectorized parameters (weights/biases) $\epsilon \boldsymbol{\phi}$  and activation $\sigma$ (refer \cref{fig:topo}). The added layer is denoted as $\LRp{l+\frac{1}{2}}$ in \cref{fig:topo}. Note that $\Omega_{\epsilon}$ is a also a function of $
l, \boldsymbol{\phi}$ and $\sigma$, but for the simplicity of the notation, we omit them.
\label{perturbedNetworkDef}
\end{definition}
In the rest of the paper (and in \cref{def_pert}), the notations $\Omega_0,\ \Omega_{\epsilon}$ denote the neural network with the list of all parameters and activation functions.
 In addition, we adopt the notation $\Omega_0(\bx_s):\real^{n_0}\mapsto \real^{n_T}$, and  $\Omega_\epsilon(\bx_s; \  \epsilon \bphi):\real^{n_0}\times \real^{n_p}\mapsto \real^{n_T}$ to denote the neural network functions (forward propagation constraint given in \eqref{training_problem}) where the second argument $ \epsilon\bphi$ denotes the parameters of the added layer for $\Omega_\epsilon$. Now it is necessary that by setting the magnitude of perturbation $\epsilon=0$ in \cref{def_pert}, 
the network $\Omega_\epsilon \Big |_{\epsilon=0}$ should behave exactly the same as
$\Omega_0$ under the gradient-based training process. We formalize this notion  in \cref{admissible_pert} below.

\begin{definition}[Admissible perturbation]
\label{admissible_pert}
We say that $\Omega_\epsilon$ in \cref{def_pert}  is an admissible perturbation  if:
\begin{equation}
\Omega_{\epsilon}\Big |_{\epsilon=0}=\Omega_0,
\label{admit_per}
\end{equation}
 where `=' in \eqref{admit_per} is used to denote the fact that  the network $\Omega_\epsilon\Large |_{\epsilon=0}$  behaves exactly the same  as $\Omega_0$ under gradient based training process. In particular,
     the added layer in $\Omega_{\epsilon}\Big |_{\epsilon=0}$ is redundant and only acts as a message-passing layer. Furthermore,  the solutions  given by \eqref{set_of} for $t=\{0,1,\dots l,(l+1),\dots T\}$ and the loss $\sJ$ in \eqref{training_problem} together with its gradient  coincide for both $\Omega_0$ and $\Omega_\epsilon \Big |_{\epsilon=0}$ at every gradient descent iteration.
     \end{definition}    
\begin{proposition}[Construction of an admissible perturbation]
\label{prop_admissible}
Consider \cref{admissible_pert}. The following two steps produce an admissible perturbation:

\begin{enumerate}
   \item  The hidden layer propagation equation in \eqref{training_problem} for  $\Omega_\epsilon$ in \cref{def_pert} satisfies: 
   \begin{equation}
   \begin{aligned}
             & \bx_{s,t+1} =\bff_{t+1}(\bx_{s,t};\ \Btheta_{t+1})=\bx_{s,t}+\bg_{t+1}(\bx_{s,t};\ \Btheta_{t+1}),\ t=1,.., l-1, l+1,..T-2,\\
             & \bx_{s,l+\frac{1}{2}} =\bff_{l+\frac{1}{2}}\LRp{\bx_{s,l};\ \Btheta_{l+\frac{1}{2}}}=\bx_{s,l}+\bg_{l+\frac{1}{2}}\LRp{\bx_{s,l};\ \Btheta_{l+\frac{1}{2}}},\\
              & \bx_{s,l+1} =\bff_{l+1}\LRp{\bx_{s,l+\frac{1}{2}};\ \Btheta_{l+1}}=\bx_{s,l+\frac{1}{2}}+\bg_{l+1}\LRp{\bx_{s,l+\frac{1}{2}};\ \Btheta_{l+1}}.
   \end{aligned}
    \label{res_o}
\end{equation}
   \label{one}
   \item Choose the activation function $\sigma$ such that $\bg_{t}\LRp{.;\ .}$ is continuously differentiable  w.r.t both the arguments and:
   \[ \bg_{t}(\bx_{s,t};\ {\bf{0}})={\bf{0}}, \ \ \nabla_{\bx}\bg_{t}(\bx_{s,t};\ {\bf{0}})={\bf{0}}, \ \ \nabla_{\btheta}\bg_{t}(\bx_{s,t};\ {\bf{0}})={\bf{0}},\ \  t=2,\dots l,l+\frac{1}{2},l+1,\dots {T-1}.\]
    \label{two}
\end{enumerate}
\end{proposition}

\begin{proof}{}{}
 Let us compute the states and adjoints in \eqref{set_of} for network $\Omega_{\epsilon}$ corresponding to the first gradient descent iteration. The forward propagation of residual neural network $\Omega_\epsilon$ from $l^{th}$ layer to $(l+1)^{th}$ layer can be  written as:
\begin{subequations}
\begin{gather}
\label{forr_a}
\bx_{s,l+\frac{1}{2}} =\bff_{l+\frac{1}{2}}(\bx_{s,l};\ \epsilon \bphi)= \bx_{s,l} + \bg_{l+\frac{1}{2}}(\bx_{s,l};\ \epsilon \bphi),\\
\bx_{s,l+1} = \bff_{l+1}\LRp{ \bx_{s,l+\frac{1}{2}};\ \Btheta_{l+1}}= \bx_{s,l+\frac{1}{2}} + \bg_{l+1}\LRp{ \bx_{s,l+\frac{1}{2}};\ \Btheta_{l+1}}.
\end{gather}
\label{forr}
\end{subequations}
Now, for $\epsilon=0$ \eqref{forr_a}  gives
%and assuming $\bg_{l+\frac{1}{2}}\LRp{.,\ .}$ to be continuous everywhere w.r.t second argument yields:
\[   \bx_{s,l+\frac{1}{2}} =\bx_{s,l} , \implies \bx_{s,l+1} = \bx_{s,l} +\bg_{l+1}\LRp{ \bx_{s,l} ;\ \Btheta_{l+1}}=\bff_{l+1}\LRp{ \bx_{s,l};\ \Btheta_{l+1}}, \]
where we have used the fact that $\bg_{l+\frac{1}{2}}(\bx_{s,l};\ {\bf{0}})={\bf{0}}$.
Therefore, adding a new layer recovers the forward propagation equations of the original network $\Omega_0$ when $\epsilon = 0$ for the first gradient descent iteration. Now, the backward propagation of the adjoint using \eqref{set_of}  can be written as:
\begin{equation}
\begin{aligned}
   \bp_{s,l+\frac{1}{2}}&= \LRs{\nabla_{\bx}\bff_{l+1}(\bx_{s,l+\frac{1}{2}};\Btheta_{l+1})}^T\bp_{s,l+1}=\bp_{s,l+1} + \LRs{\nabla_{\bx}\bg_{l+1}\LRp{ \bx_{s,l+\frac{1}{2}};\ \Btheta_{l+1}}}^T\bp_{s,l+1} ,\\
   \bp_{s,l} & = \LRs{\nabla_{\bx}\bff_{l+\frac{1}{2}}\LRp{ \bx_{s,l};\ \epsilon\bphi}}^T\bp_{s,l+\frac{1}{2}} =\bp_{s,l+\frac{1}{2}} + \LRs{\nabla_{\bx}\bg_{l+\frac{1}{2}}\LRp{ \bx_{s,l};\ \epsilon\bphi}}^T\bp_{s,l+\frac{1}{2}}.
\end{aligned}
\label{for}
\end{equation}
Again using $ \nabla_{\bx}\bg_{l+\frac{1}{2}}\LRp{ \bx_{s,l};\ \bf{0}}={\bf{0}}$, we conclude
\[\bp_{s,l}=\bp_{s,l+\frac{1}{2}}.\]
Therefore, when $\epsilon = 0$, the adjoint of the perturbed network and the original one are the same for the first gradient descent iteration. 
Finally, the third condition in \cref{two} for the added layer, i.e. $\nabla_{\btheta}\bg_{l+\frac{1}{2}}(\bx_{s,l};\ {\bf{0}})={\bf{0}}$ implies  that $\nabla_{\btheta}\bff_{l+\frac{1}{2}}(\bx_{s,l};\ {\bf{0}})={\bf{0}}$. Therefore,
for $\epsilon= 0$, the gradients for updating the parameters of the added layer given by the last equation in \eqref{set_of} is ${\bf{0}}$, leading to the parameters of the added layer not getting updated at the end of the first gradient descent iteration. Applying the above arguments recursively for subsequent gradient descent iterations concludes the proof.
\end{proof}

\begin{remark}{}{}$ $
\label{act_const}
\begin{enumerate}
    \item Note that for typical neural networks such as FNN's and CNN's,  condition \ref{two} can be easily realized by choosing continuously differentiable activation function $\sigma$ with the property that $\sigma(0)=0,$ and $\sigma'(0)=0$.
    \label{activ_1}
    \item   The activation $\sigma(x)$  for the added layer can be constructed as a linear combination of some existing activation functions, i.e $\sigma(x)=\alpha_1\sigma_1(x)+\sigma_2(x)$, where $\sigma_1(0)=\sigma_2(0)=0$ and $\alpha_1=-\frac{\sigma_2'(0)}{\sigma'_1(0)}$, $\sigma_1'(0)\neq 0$. For instance, the set $\sF=\{\sigma_1,\ \sigma_2 \}$ can be chosen as follows:
\begin{equation}
    \sF=\{ \mathrm{Swish},\ \mathrm{tanh}\},\ \{ \mathrm{Mish},\ \mathrm{tanh}\},\  \{ \mathrm{Swish},\ \mathrm{Mish}\},\ \ etc.
    \label{univer}
\end{equation}
 \label{activ_2}
Note that it can be proved that the constructed function $\sigma(x)$ based on the choices in \eqref{univer}
is universal\footnote{Let $\sN\sN^{\sigma}_{n,m,k}$ represent the class of functions $\real^n \rightarrow \real^m$ described by feedforward neural networks of arbitrary number of hidden layers, each with $k$ neurons and activation
function  $\sigma$. Let $K\subseteq \real^n$ be compact. We say that the activation function $\sigma(x)$ is universal if $\sN\sN^{\sigma}_{n,m,k}$ is dense in $C(K;\real^m)$ with respect to the uniform norm.} by verifying the conditions in \cite{kidger2020universal} or \cite{bui2024unified}.
\end{enumerate}
 
\end{remark}

\subsection{Network topological derivative}
\label{full_topo}

In this section, we define the ``network topological derivative'' for a feed-forward neural network and derive its explicit expression. Note that the admissible perturbation (adding a layer) defined in  \cref{admissible_pert} can be viewed as an infinitesimal change in the neural network topology. 
In order to formally define the ``network topological derivative'' let us first rewrite the loss function \eqref{training_problem}  as follows:
\begin{equation}
     \sJ(\Omega_\epsilon)= \frac{1}{S}\sum_{s=1}^S\Phi\LRp{\Omega_\epsilon\LRp{\bx_{s,0};\ \epsilon\bphi}},
     \label{for_later}
\end{equation}
where $\epsilon \bphi$ is the initialization of the added layer in  $\Omega_\epsilon$ (see \cref{def_pert}).
\begin{definition}[Network topological derivative]
\label{topo_defin}
Consider an admissible perturbation $\Omega_\epsilon$ in \cref{admissible_pert}. We say the loss functional $\sJ$ in \eqref{for_later} admits a network topological derivative\textemdash denoted as $d\sJ(\Omega_0;\ (l,\  \boldsymbol{\phi},\ \sigma))$\textemdash at $\Omega_0\in \sA$ and at the location $l\in \{1,2\dots T-1\}$ along the direction $\boldsymbol{\phi}$  if there exists a  function $q: \real^{+}\mapsto \real^+$ with $\lim\limits_{\epsilon \downarrow  0}q(\epsilon)=0$ such that the following ``topological asymptotic expansion" holds:
\begin{equation}
    \sJ(\Omega_\epsilon)=\sJ(\Omega_0)-q(\epsilon)d\sJ(\Omega_0;\ (l,\  \boldsymbol{\phi},\ \sigma))+o(q(\epsilon)),
    \label{asymptotic_exp}
\end{equation}
where $o(q(\epsilon))$ is the remainder\footnote{The notation $p(\epsilon)=o(q(\epsilon))$ means that $\lim_{\epsilon \rightarrow 0}\frac{p(\epsilon)}{q(\epsilon)}=0$, where $q(\epsilon)\neq 0$ for all $\epsilon \neq 0$.}. \eqref{asymptotic_exp} is equivalent to the existence of the following limit:
\begin{equation}
    d\sJ(\Omega_0;\ (l,\  \boldsymbol{\phi},\ \sigma))=-\lim_{\epsilon \downarrow 0}\frac{\sJ(\Omega_\epsilon)-\sJ(\Omega_0)}{q(\epsilon)}.
    \label{topo_der}
\end{equation}
\end{definition}
\begin{remark}
Note that the `-' sign is used in \eqref{topo_der} to ensure that 
    a positive derivative favors layer addition and vice versa.  
    Since for an  admissible perturbation $\Omega_\epsilon$ in \cref{admissible_pert}, we have $\sJ\LRp{\Omega_{\epsilon}\Big |_{\epsilon=0}}=\sJ\LRp{\Omega_0}$, 
the ``topological asymptotic expansion" in \eqref{asymptotic_exp},  can  be understood as the following Taylor series expansion:
    \begin{equation}
            \sJ(\Omega_\epsilon)=\sJ\LRp{\Omega_\epsilon\Big |_{\epsilon=0}}+ \epsilon\LRs{\frac{d}{d\epsilon}\LRp{\sJ\LRp{\Omega_\epsilon}}}_{\epsilon=0}+ \frac{\epsilon^2}{2}\LRs{\frac{d^2}{d\epsilon^2}\LRp{\sJ\LRp{\Omega_\epsilon}}}_{\epsilon=0}+\dots,
    \label{eq:Taylor}
    \end{equation}
where $\sJ\LRp{\Omega_\epsilon}$ is given by \eqref{for_later}. In \cref{exist_th}  we will unveil a non-trivial connection between 
the  Hamiltonian $H_l$ introduced in \eqref{hamiltonian} with the network topological derivative  in \eqref{topo_der}.
\end{remark}

\begin{theorem}[Existence of network topological derivative]
\label{exist_th}
Assume the conditions in \cref{prop_admissible} and the loss $\sJ$ in \eqref{training_problem}. Further, let $\mathcal{X}_t$ be the set of all possible states $\bx_{s,t}$ reachable\footnote{We say that a state $\bx_{s,t+1}$ is reachable at layer $(t+1)$ if there exists $\btheta \in \Theta$ and  sample $\bx_{s,0}$ such that the constraint in \eqref{training_problem} holds.}  at layer $t$ from all possible  initial sample $\bx_{s,0}$ 
and trainable parameters in $\Theta$.  Assume that
\begin{enumerate}
    \item $\Phi$ has bounded second and third order derivatives on $\mathcal{X}_T$; 
   \label{cond_one}
    \item $\bff_{t+1}(.;\ .)$ has bounded second and third order derivatives on $\mathcal{X}_t \times \Theta_{t+1}$.
    \label{cond_two}
\end{enumerate}
Then, the network topological derivative given by \eqref{topo_der} exists with $q(\epsilon)=\epsilon^2$ and is given by:
\begin{equation}
     d\sJ(\Omega_0;\ (l,\  \boldsymbol{\phi},\ \sigma))=\frac{1}{2}\sum_{s=1}^S \boldsymbol{\phi}^T \ \nabla^2_{\boldsymbol{\theta}} H_l\LRp{\bx_{s,l};\bp_{s,l};\Btheta}\Big |_{\boldsymbol{\theta}={\bf{0}}}\ \boldsymbol{\phi},
     \label{topo_de}
\end{equation}
where $H_l$ is the Hamiltonian defined in \eqref{hamiltonian}\footnote{Note the change in subscript  for the adjoint variable $\bp$ while evaluating the Hamiltonian.}.

\end{theorem}

\begin{proof} 
Let us represent the parameters of the perturbed network $\Omega_\epsilon$ as $\btheta^\epsilon$ and the parameters of the original network $\Omega_0$ as  $\btheta^0$. For an admissible perturbation $\Omega_\epsilon$ in \cref{admissible_pert},   $\Omega_0$ is equivalent to  $\Omega_\epsilon\Large |_{\epsilon=0}$. We thus identify $\Omega_0$ with $\Omega_\epsilon\Large |_{\epsilon=0}$. 
By  \cref{perturbedNetworkDef} $\btheta^\epsilon_i=\btheta^0_i,\ \forall i\neq l+\frac{1}{2}$.  Define
\[ \delta \bx_{s,t}=\bx_{s,t}^{{\btheta}^\epsilon}-\bx_{s,t}^{\btheta^0},\quad \delta \bp_{s,t}=\bp_{s,t}^{\btheta^\epsilon}-\bp_{s,t}^{\btheta^0},
\]
and we have
\[  \delta \bx_{s,t}=0,\ \ \forall t\leq l.\]
Applying definition \eqref{hamiltonian} for both $\Omega_{\epsilon}$ and $\Omega_0$ gives
\begin{subequations}
\begin{gather}
\label{hamil_cases_a}
\sum_{t=0}^{T-1} \LRp{H_t\LRp{\bx_{s,t}^{\theta^0}+\delta \bx_{s,t};\ \bp_{s,t+1}^{\btheta^0}+\delta \bp_{s,t+1};\ \btheta_{t+1}^\epsilon}-(\bp_{s,t+1}^{\btheta^0}+\delta \bp_{s,t+1})^T(\bx_{s,t+1}^{\theta^0}+\delta \bx_{s,t+1})}=0,\\
\label{hamil_cases_b}
\sum_{t=0}^{T-1} \LRp{H_t\LRp{\bx_{s,t}^{\theta^0};\ \bp_{s,t+1}^{\btheta^0};\ \btheta_{t+1}^0}-(\bp_{s,t+1}^{\btheta^0})^T(\bx_{s,t+1}^{\theta^0})}=0.
\end{gather}
\label{hamil_cases}
\end{subequations}
Note that the  summation from $t=0$ to $t=T-1$ also includes $t=\LRp{l+\frac{1}{2}}$. Subtracting   \eqref{hamil_cases_b} from \eqref{hamil_cases_a} yields:
\begin{equation}
    \begin{aligned}
  &  \sum_{t=0}^{T-1} H_t\LRp{\bx_{s,t}^{\theta^0}+\delta \bx_{s,t};\ \bp_{s,t+1}^{\btheta^0}+\delta \bp_{s,t+1};\ \btheta_{t+1}^\epsilon}- H_t\LRp{\bx_{s,t}^{\theta^0};\ \bp_{s,t+1}^{\btheta^0};\ \btheta_{t+1}^0}=\\
  &  \LRp{\delta \bp_{s,l+\frac{1}{2}}\cdot \delta \bx_{s,l+\frac{1}{2}}}+\LRp{\bp_{s,l+\frac{1}{2}}^{\btheta^0}\cdot \delta \bx_{s,l+\frac{1}{2}}}\\
   &+\underbrace{\sum_{t=l}^{T-1} \LRp{\delta \bp_{s,t+1}\cdot \delta \bx_{s,t+1}}}_{I}+\underbrace{\sum_{t=l}^{T-1} \LRp{\bp_{s,t+1}^{\btheta^0}\cdot \delta \bx_{s,t+1}}}_{II}+\underbrace{\sum_{t=0}^{T-1} \LRp{\delta \bp_{s,t+1}\cdot \bx_{s,t+1}^{\btheta_0}}}_{III}.
\end{aligned}
\label{decomposition}
\end{equation}
Now let's analyze each term on the right-hand side of \eqref{decomposition}. 

\hspace{-0.8 cm} {\bf{\underline{Analyzing Term I in \eqref{decomposition}}}}

    \[  \delta \bx_{s,l+\frac{1}{2}}=\bff_{l+\frac{1}{2}}\LRp{\bx_{s,l}^{\btheta_\epsilon};\ \epsilon \bphi}-\bff_{l+\frac{1}{2}}\LRp{\bx_{s,l}^{\btheta_0};\ \boldsymbol{0}}.\]
  Assuming that $\bff_{l+\frac{1}{2}}(.;.)$ has bounded second order derivative w.r.t the second argument and noting that $\bx_{s,l}^{\btheta_\epsilon}=\bx_{s,l}^{\btheta_0}$,  Taylor series approximation about $\epsilon=0$ gives:
    \[ \delta \bx_{s,l+\frac{1}{2}}=\epsilon \ \nabla_{\btheta}\bff_{l+\frac{1}{2}}\LRp{\bx_{s,l}^{\btheta_0};\ \boldsymbol{0}}\cdot \bphi+\mathcal{O}(\epsilon^2)=0+\mathcal{O}(\epsilon^2),\]
where, $\nabla_{\btheta}\bff_{l+\frac{1}{2}}$ denotes the derivative of $\bff_{l+\frac{1}{2}}$ w.r.t the second argument. Note that one has $\nabla_{\btheta}\bff_{l+\frac{1}{2}}(\bx_{s,l};\ \boldsymbol{0})={\bf{0}}$ by condition \ref{two} of \cref{prop_admissible}. Therefore, we have:

\[  \delta \bx_{s,l+1}= \bff_{l+1}\LRp{\bx_{s,l+\frac{1}{2}}^{\btheta^0}+\mathcal{O}(\epsilon^2);\ \btheta_{l+1}^\epsilon}-\bff_{l+1}\LRp{\bx_{s,l+\frac{1}{2}}^{\btheta^0};\ \btheta_{l+1}^0}.\]
Now, considering the Taylor expansion about $\bx_{s,l+\frac{1}{2}}^{\btheta^0}$ and assuming that $\bff_{l+1}(.;.)$ has bounded second order derivative w.r.t the first argument, it is easy to see that by Taylor's theorem:
\begin{equation}
    \begin{aligned}
       \delta \bx_{s,l+1}&=\nabla_{\bx}\bff_{l+1}\LRp{\bx_{s,l+\frac{1}{2}}^{\btheta^0};\ \btheta_{l+1}^0}\cdot \mathcal{O}(\epsilon^2)+ \mathcal{O}(\epsilon^4)\\
        &=\LRs{\mathcal{I}+\nabla_{\bx}\bg_{l+1}\LRp{\bx_{s,l+\frac{1}{2}}^{\btheta^0};\ \btheta_{l+1}^0}}\cdot \mathcal{O}(\epsilon^2)+\mathcal{O}(\epsilon^4),
    \end{aligned}
\end{equation}
where we have used the fact that $\btheta_{l+1}^\epsilon=\btheta_{l+1}^0$ and used the ResNet propagation equation in \eqref{forr}. Therefore, we have:
\[ \delta \bx_{s,l+1}=\mathcal{O}(\epsilon^2). \]
Applying the above arguments recursively, one can show that:

\begin{equation}
    \delta \bx_{s,i}=\mathcal{O}(\epsilon^2),\quad \ i=\LRp{l+\frac{1}{2}},(l+1),\dots T.
    \label{taylor_x}
\end{equation}
Now from  \eqref{set_of} we have:
\[\bp_{s,T}^{\btheta^\epsilon}=-\frac{1}{S}\nabla \Phi\LRp{\bx_{s,T}^{\btheta^\epsilon}}, \quad \bp_{s,T}^{\btheta^0}=-\frac{1}{S}\nabla \Phi\LRp{\bx_{s,T}^{\btheta^0}},\]
for networks $\Omega_\epsilon$ and $\Omega_0$ respectively. Therefore,
\begin{equation}
    \begin{aligned}
     \delta \bp_{s,T}&= \bp_{s,T}^{\btheta^\epsilon}-\bp_{s,T}^{\btheta^0}\\
        &= \frac{1}{S}\LRs{\nabla \Phi\LRp{\bx_{s,T}^{\btheta^0}}-\nabla \Phi\LRp{\bx_{s,T}^{\btheta^\epsilon}}}=\frac{1}{S}\LRs{\nabla \Phi\LRp{\bx_{s,T}^{\btheta^0}}-\nabla\Phi\LRp{\bx_{s,T}^{\btheta^0}+\delta \bx_{s,T}}}\\
        &=\frac{1}{S}\LRs{\nabla \Phi\LRp{\bx_{s,T}^{\btheta^0}}-\nabla \Phi\LRp{\bx_{s,T}^{\btheta^0}+\mathcal{O}(\epsilon^2)}}.
    \end{aligned}
\end{equation}
Now assuming that $\Phi$ has bounded third order derivative,  Taylor series approximation about $\bx_{s,T}^{\btheta^0}$ gives:
\[ \delta \bp_{s,T}= \mathcal{O}(\epsilon^2). \]
Further, from  \eqref{set_of} we also have:
\begin{equation}
    \begin{aligned}
        \bp_{s,T-1}^{\btheta^\epsilon}&=\LRs{\nabla_{\bx}\bff_{T}\LRp{\bx_{s,T-1}^{\btheta^\epsilon};\btheta_{T}^\epsilon}}^T\bp_{s,T}^{\btheta^\epsilon}=\LRs{\nabla_{\bx}\bff_{T}\LRp{\bx_{s,T-1}^{\btheta^0}+\mathcal{O}(\epsilon^2);\ \btheta_{T}^0}}^T\LRp{\bp_{s,T}^{\btheta^0}+\mathcal{O}(\epsilon^2)},\\
        \bp_{s,T-1}^{\btheta^0}&=\LRs{\nabla_{\bx}\bff_{T}\LRp{\bx_{s,T-1}^{\btheta^0};\btheta_{T}^0}}^T\bp_{s,T}^{\btheta^0}.
    \end{aligned}
    \label{t_one}
\end{equation}
Again, under the assumption that $\bff_T(.;\ .)$ has bounded third order derivative  w.r.t the first argument, it is easy to show that:
\[ \delta  \bp_{s,T-1}= \bp_{s,T-1}^{\btheta^\epsilon}-\bp_{s,T-1}^{\btheta^0}=\mathcal{O}(\epsilon^2),\]
where we have considered the Taylor expansion of $\nabla_{\bx}\bff_{T}$ about $\bx_{s,T-1}^{\btheta^0}$. Now using the above arguments recursively, one can show that:
\begin{equation}
    \delta \bp_{s,i}=\mathcal{O}(\epsilon^2),\quad \ i=1,\dots T.
    \label{taylor_p}
\end{equation}
Therefore, from \eqref{taylor_x} and \eqref{taylor_p} we have:
\begin{equation}
    {\delta \bp_{s,t}\cdot \delta \bx_{s,t}}=\mathcal{O}(\epsilon^4),\quad \ t=\LRp{l+\frac{1}{2}},(l+1),\dots T.
    \label{five_t}
\end{equation}
\end{proof}

\hspace{-0.8 cm} {\bf{\underline{Analyzing Term II and III in \eqref{decomposition}}}}
\begin{equation}
    \begin{aligned}
       &  \sum_{t=l}^{T-1} \LRp{\bp_{s,t+1}^{\btheta^0}\cdot \delta \bx_{s,t+1}}+\sum_{t=0}^{T-1} \LRp{\delta \bp_{s,t+1}\cdot \bx_{s,t+1}^{\btheta_0}}=\\ 
       & \sum_{t=l+1}^{T-1} \LRp{\bp_{s,t}^{\btheta^0}\cdot \delta \bx_{s,t}}+\bp_{s,T}^{\btheta^0} \cdot \delta \bx_{s,T}+ \sum_{t=0}^{T-1} \LRp{\delta \bp_{s,t+1}\cdot \bx_{s,t+1}^{\btheta_0}}. 
    \end{aligned}
     \label{new_t}
\end{equation}
Now from \eqref{set_of} note that:
\[ \bp_{s,T}^{\btheta^0} \cdot \delta \bx_{s,T}=-\frac{1}{S}\nabla \Phi\LRp{\bx_{s,T}^{\btheta^0}}\cdot \delta \bx_{s,T}= -\frac{1}{S}\LRs{\Phi\LRp{\bx_{s,T}^{\btheta^\epsilon}}-\Phi\LRp{\bx_{s,T}^{\btheta^0}}}+\mathcal{O}\LRp{\norm{\delta \bx_{s,T}}_2^2}.\]
where we have assumed $\Phi$ has bounded second order derivative. Therefore,
\begin{equation}
\sum_{s=1}^S \bp_{s,T}^{\btheta^0} \cdot \delta \bx_{s,T}=\sJ(\Omega_0)-\sJ(\Omega_\epsilon)+\mathcal{O}(\epsilon^4).
\label{two_t}
\end{equation}
Further, the remaining terms in \eqref{new_t} can be written in terms of the Hamiltonian as, 
\begin{equation}
  \begin{aligned}
    &\sum_{t=l+1}^{T-1} \LRp{\bp_{s,t}^{\btheta^0}\cdot \delta \bx_{s,t}}+\sum_{t=0}^{T-1} \LRp{\delta \bp_{s,t+1}\cdot \bx_{s,t+1}^{\btheta_0}}=\\
   & \sum_{t=l+1}^{T-1}\nabla_{\bx} H_t\LRp{\bx_{s,t}^{\btheta^0}; \bp_{s,t+1}^{\btheta^0};\btheta_{t+1}^0}\cdot \delta \bx_{s,t}+\sum_{t=0}^{T-1} \nabla_{\bp} H_t\LRp{\bx_{s,t}^{\btheta^0}; \bp_{s,t+1}^{\btheta^0};\btheta_{t+1}^0}\cdot \delta \bp_{s,t+1}.
\end{aligned}  
\label{three_t}
\end{equation}
Finally, substituting \eqref{two_t}, \eqref{three_t}, \eqref{five_t} in \eqref{decomposition} and summing over all the training data points we have:
\begin{equation}
   \begin{aligned}
   & \sJ(\Omega_0)-\sJ(\Omega_\epsilon) +\sum_{s=1}^S\sum_{t=l+1}^{T-1}\nabla_{\bx} H_t\LRp{\bx_{s,t}^{\btheta^0}; \bp_{s,t+1}^{\btheta^0};\btheta_{t+1}^0}\cdot \delta \bx_{s,t}\\
    &+\sum_{s=1}^S\sum_{t=0}^{T-1}\nabla_{\bp} H_t\LRp{\bx_{s,t}^{\btheta^0}; \bp_{s,t+1}^{\btheta^0};\btheta_{t+1}^0}\cdot \delta \bp_{s,t+1}+\sum_{s=1}^S\nabla_{\bx} H_{l+\frac{1}{2}}\LRp{\bx_{s,{l+\frac{1}{2}}}^{\btheta^0}, \bp_{s,l+1}^{\btheta^0},\btheta_{l+1}^0}\cdot \delta \bx_{s,l+\frac{1}{2}}+\mathcal{O}(\epsilon^4)\\
    &=\sum_{s=1}^S\sum_{t=0}^{T-1} \LRs{H_t\LRp{\bx_{s,t}^{\theta^0}+\delta \bx_{s,t};\ \bp_{s,t+1}^{\btheta^0}+\delta \bp_{s,t+1};\ \btheta_{t+1}^\epsilon}- H_t\LRp{\bx_{s,t}^{\theta^0};\ \bp_{s,t+1}^{\btheta^0};\ \btheta_{t+1}^0}}.
\end{aligned} 
\label{final_one}
\end{equation}
The R.H.S of \eqref{final_one} can be simplified by considering the Taylor series expansion about $\LRp{\bx_{s,t}^{\btheta^0},\ \bp_{s,t}^{\btheta^0},\ \btheta_{t+1}^\epsilon}$as:
\begin{align*}
    & H_t\LRp{\bx_{s,t}^{\btheta^0}+\delta \bx_{s,t};\ \bp_{s,t+1}^{\btheta^0}+\delta \bp_{s,t+1};\ \btheta_{t+1}^\epsilon}- H_t\LRp{\bx_{s,t}^{\theta^0};\ \bp_{s,t+1}^{\btheta^0};\ \btheta_{t+1}^0}=\\
     & H_t\LRp{\bx_{s,t}^{\btheta^0};\ \bp_{s,t+1}^{\btheta^0};\ \btheta_{t+1}^\epsilon}-H_t\LRp{\bx_{s,t}^{\theta^0};\ \bp_{s,t+1}^{\btheta^0};\ \btheta_{t+1}^0}+\nabla_{\bx} H_t(\bx_{s,t}^{\btheta^0}; \bp_{s,t+1}^{\btheta^0};\btheta_{t+1}^\epsilon)\cdot \delta \bx_{s,t}+\\
     &\nabla_{\bp} H_t\LRp{\bx_{s,t}^{\btheta^0}; \bp_{s,t+1}^{\btheta^0};\btheta_{t+1}^\epsilon}\cdot \delta \bp_{s,t+1}+\mathcal{O}(\epsilon^4).
\end{align*}
Substituting the above in \eqref{final_one} and simplifying  (note that $\btheta_{t}^0=\btheta_{t}^\epsilon, \ \forall t\neq l+\frac{1}{2}$) we have:
\begin{equation}
    \begin{aligned}
        \sJ(\Omega_0)-\sJ(\Omega_\epsilon)&=\sum_{s=1}^S H_l\LRp{\bx_{s,l}^{\btheta^0};\ \bp_{s,l+\frac{1}{2}}^{\btheta^0};\ \epsilon \bphi}-H_l\LRp{\bx_{s,l}^{\theta^0};\ \bp_{s,l+\frac{1}{2}}^{\btheta^0};\ {\bf{0}}}+\mathcal{O}(\epsilon^4)\\
        &+ \LRs{\nabla_{\bx} H_{l}\LRp{\bx_{s,l}^{\btheta^0}, \bp_{s,l+\frac{1}{2}}^{\btheta^0},\epsilon \bphi}-\nabla_{\bx} H_{l}\LRp{\bx_{s,l}^{\btheta^0}, \bp_{s,l+\frac{1}{2}}^{\btheta^0},{\bf{0}}}}\cdot \delta \bx_{s,l}\\
       & +\LRs{\nabla_{\bp} H_{l}\LRp{\bx_{s,l}^{\btheta^0}, \bp_{s,l+\frac{1}{2}}^{\btheta^0},\epsilon \bphi}-\nabla_{\bp} H_{l}\LRp{\bx_{s,l}^{\btheta^0}, \bp_{s,l+\frac{1}{2}}^{\btheta^0},{\bf{0}}}}\cdot \delta \bp_{s,l}.
    \end{aligned}
    \label{tay}
\end{equation}
Now considering the Taylor expansion about $\epsilon=0$ for the last two terms in \eqref{tay} and noting that $ \nabla_{\bx\btheta}H_{l}\LRp{\bx_{s,l}^{\btheta^0}, \bp_{s,l+\frac{1}{2}}^{\btheta^0},{\bf{0}}}={\bf{0}}$ and $\nabla_{\bp\btheta}H_{l}\LRp{\bx_{s,l}^{\btheta^0}, \bp_{s,l+\frac{1}{2}}^{\btheta^0},{\bf{0}}}={\bf{0}}$ due to assumption \ref{two},
it follows that the last two terms in \eqref{tay} are of order  $\mathcal{O}(\epsilon^4)$. Therefore,  we finally have:
\[\sJ(\Omega_0)-\sJ(\Omega_\epsilon)=\sum_{s=1}^S H_l\LRp{\bx_{s,l}^{\btheta^0};\ \bp_{s,l}^{\btheta^0};\ \epsilon \bphi}-H_l\LRp{\bx_{s,l}^{\theta^0};\ \bp_{s,l}^{\btheta^0};\ {\bf{0}}}+\mathcal{O}(\epsilon^4). 
\]
where we have also used the fact $\bp_{s,l+\frac{1}{2}}^{\btheta_0}=\bp_{s,l}^{\btheta_0}$ since the added layer with zero weights and biases acts as a message-passing layer due to \cref{prop_admissible}. Finally, applying Taylor series expansion about $\epsilon=0$ and assuming that $\bff_t(.;\ .)$  has bounded third order derivative w.r.t second argument we have:
\begin{equation}
    \sJ(\Omega_0)-\sJ(\Omega_\epsilon)=\sum_{s=1}^S \LRp{\epsilon \nabla_{\btheta} H_l\Big |_{\boldsymbol{\theta}=0} \cdot \bphi+\frac{\epsilon^2}{2}\boldsymbol{\phi}^T \ \nabla^2_{\boldsymbol{\theta}} H_{l}\Big |_{\boldsymbol{\theta}=0}\ \boldsymbol{\phi}  +\mathcal{O}(\epsilon^3)}.
    \label{connect}
    \end{equation}
Note that $\nabla_{\btheta} H_l\Big |_{\boldsymbol{\theta}=0}=\bf{0}$  by condition \ref{two} of  \cref{prop_admissible}. Therefore, the topological derivative is computed as:
\[   d\sJ(\Omega_0;\ (l,\ \boldsymbol{\phi},\ \sigma))=-\lim_{\epsilon \downarrow 0}\frac{\sJ(\Omega_\epsilon)-\sJ(\Omega_0)}{\epsilon^2}=\frac{1}{2}\sum_{s=1}^S \boldsymbol{\phi}^T \ \nabla^2_{\boldsymbol{\theta}} H_{l}\Big |_{\boldsymbol{\theta}=0}\ \boldsymbol{\phi},\]
thereby concluding the proof.

\begin{corollary}
Assume the conditions in  \cref{exist_th}. For the expansion \cref{eq:Taylor}  there hold:
      \begin{itemize}
        \item $\LRs{\frac{d}{d\epsilon}\LRp{\sJ\LRp{\Omega_\epsilon}}}_{\epsilon=0}= \nabla_{\epsilon\phi}\sJ\LRp{\Omega_\epsilon}\Big |_{\epsilon=0} \cdot \boldsymbol{\phi} = -\sum\limits_{s=1}^S\nabla_{\btheta} H_l\Big |_{\boldsymbol{\theta}=0} \cdot \bphi=        
        0$, 
        \item  $\frac{1}{2} \LRs{\frac{d^2}{d\epsilon^2}\LRp{\sJ\LRp{\Omega_\epsilon}}}_{\epsilon=0}        = 
        \frac{1}{2}\boldsymbol{\phi}^T\nabla^2_{\epsilon\phi}\sJ\LRp{\Omega_\epsilon}\Big|_{\epsilon=0}\boldsymbol{\phi}=-\frac{1}{2}\sum \limits_{s=1}^S \boldsymbol{\phi}^T \ \nabla^2_{\boldsymbol{\theta}} H_{l}\Big |_{\boldsymbol{\theta}=0}\ \boldsymbol{\phi}$, and        \item $d\sJ(\Omega_0;\ (l,\ \boldsymbol{\phi},\ \sigma))       = 
        -\frac{1}{2}\boldsymbol{\phi}^T\nabla^2_{\epsilon\phi}\sJ\LRp{\Omega_\epsilon}\Big|_{\epsilon=0}\boldsymbol{\phi}=\frac{1}{2}\sum \limits_{s=1}^S \boldsymbol{\phi}^T \ \nabla^2_{\boldsymbol{\theta}} H_{l}\Big |_{\boldsymbol{\theta}=0}\ \boldsymbol{\phi}$.
    \end{itemize}
    \label{coro:Hessian}
\end{corollary}
\begin{proof}
    Comparing equations \cref{connect} and \cref{eq:Taylor} and noting that $\sJ\LRp{\Omega_{\epsilon}\Big |_{\epsilon=0}}=\sJ\LRp{\Omega_0}$ concludes the proof.
    \end{proof}
    
As can be seen from \cref{coro:Hessian}, the directional derivative along the direction $\boldsymbol{\phi}$ of the loss function $\sJ$ at $\boldsymbol{0}$ vanishes. That is, when the newly added layer is a message passing layer, it does not affect the loss function, and this is consistent with \cref{prop_admissible}. Ignoring the factor $1/2$, we can also see that that network derivative is the negative of the Hessian of the loss function at $\boldsymbol{0}$ acting on $\boldsymbol{\phi}$ both on the left and the right. \cref{coro:Hessian} was also the reason we  introduced the generalized notion of network derivative in \cref{topo_der} that allows us to deploy the Hessian as the measure of sensitivity when the gradient vanishes.

\subsection{Derivation of our adaptation algorithm}
\label{deriv_algo}
We now exploit the derived closed-form expression for the network topological derivative \eqref{topo_de} to devise a greedy algorithm for architecture adaptation. Given the network derivative in \cref{topo_de}, the criterion is obvious: {\bf we add a new layer at the depth $l^*$ if the network derivative there is the largest} as that will incur the greatest decrease of the loss function according to \cref{asymptotic_exp}. In fact, \eqref{topo_de} helps us determine not only the best location $l^*$ to add a new layer but also an appropriate initialization for the training process when the new layer is added. %\subsubsection{Deriving our adaptation algorithm via a global perturbation viewpoint}

To begin, we note that the maximal network derivative at a location $l$ and the corresponding direction $\Phi_l$ (and hence the most appropriate initialization of the newly added layer, has it been added) are determined by
%The effect of adding a layer at the $l^{th}$ location is determined by computing the following quantities:
%In particular, we use \eqref{topo_de} to determine the best location (along the depth) to add a new layer along with its initialization for the training process.  To that end, let us define:
\begin{equation}
    \begin{aligned}
       & \Lambda_l= \max_{\bphi} \LRp{d\sJ(\Omega_0;\ (l,\ \bphi,\ \sigma))},\quad s.t. \norm{\bphi}_2^2=1,\\
       & \Phi_l=\underset{\bphi}{\operatorname{argmax}}  \LRp{d\sJ(\Omega_0;\ (l,\ \bphi,\ \sigma))},\quad s.t. \norm{\bphi}_2^2=1,
    \end{aligned}
    \label{quant}
\end{equation}
where we have considered initializing the new layer with unit direction $\bphi$ %. Normalizing $\boldsymbol{\phi}$ 
to avoid an arbitrary scaling in \eqref{quant}.  %The optimization problem \eqref{quant} therefore seeks to find a direction ${\bphi}_l$ along which the topological derivative varies the most for the insertion of a new layer between layer $l$ and layer $l+1$.  
{\it{ A new layer is added at the location $l^*=\argmax\limits_l \{\Lambda_l\}_{l=1}^{T-1}$.}}
%The optimization problem \eqref{quant} therefore seeks to find a direction $\boldsymbol{\phi}$ along which the topological derivative varies the most for the insertion of a new layer between layer $l$ and layer $l+1$.  
It is clear that $\Lambda_l$ is the maximum eigenvalue and $\Phi_l$ is the corresponding eigenvector to the following eigenvalue problem:
\begin{equation}
    Q_l\bphi=\lambda \bphi,\quad Q_l=\frac{1}{2}\sum_{s=1}^S \nabla^2_{\btheta} H_l\LRp{\bx_{s,l};\bp_{s,l};\Btheta}\Big |_{\btheta=0}= -\frac{1}{2}\nabla^2_{\epsilon\boldsymbol{\phi}}\sJ\LRp{\Omega_\epsilon}\Big|_{\epsilon=0}.
    \label{eigen_matrixx}
\end{equation}

\begin{theorem}
    Consider adding a new layer with weights/biases $\epsilon{\bphi}$ at depth $l$. Suppose $\Lambda_l$ in \cref{quant} is positive,
    and the Hessian $\nabla^2_{\epsilon\boldsymbol{\phi}}\sJ\LRp{\Omega_\epsilon}$ is locally Lipschitz at $\epsilon = {0}$ in the neighborhood of radius $\frac{2\Lambda_l}{L}$ with Lipschitz constant $L > 0$. If $0 < \epsilon < \frac{2\Lambda_l}{L}$, then initializing the newly added layer with $\epsilon \Phi_l$ decreases the loss function, i.e., $\sJ\LRp{\Omega_\epsilon} < \sJ\LRp{\Omega_0}$.
\end{theorem}
\begin{proof}
    Terminating the Taylor expansion \cref{eq:Taylor} at the second order term, we have
    \[
    \sJ(\Omega_\epsilon)=\sJ\LRp{\Omega_0}+ \frac{\epsilon^2}{2}{\Phi_l}^T\nabla^2_{\epsilon\boldsymbol{\phi}}\sJ\LRp{\Omega_\epsilon}\Big |_{\epsilon=\gamma}{\Phi_l},
    \]
    for some $0 \le \gamma \le \epsilon$. It follows that
    \begin{multline*}
        \sJ(\Omega_\epsilon)=\sJ\LRp{\Omega_0} - \Lambda_l\epsilon^2  + \frac{\epsilon^2
        }{2}{\Phi_l}^T\LRs{\nabla^2_{\epsilon\boldsymbol{\phi}}\sJ\LRp{\Omega_\epsilon}\Big |_{\epsilon=\gamma}-\nabla^2_{\epsilon\boldsymbol{\phi}}\sJ\LRp{\Omega_\epsilon}\Big |_{\epsilon=0}}{\Phi_l} \\ \le 
        \sJ\LRp{\Omega_0} - \Lambda_l\epsilon^2  + \frac{\gamma L\epsilon^2
        }{2} <  \sJ\LRp{\Omega_0}.
    \end{multline*}
\end{proof}

\begin{remark}{}{}
   We have shown that if $\Lambda_l$, given by \eqref{quant}, is  greater than $0$, then adding a new layer at depth $l$  guarantees a decrease in the loss for sufficiently small $\epsilon$. The question is when  $\Lambda_l > 0$?  \cref{coro:Hessian} shows that $\Lambda_l$ is half of the negative of the smallest eigenvalue value of the  Hessian $\nabla^2_{\epsilon\boldsymbol{\phi}}\sJ\LRp{\Omega_\epsilon}\Big|_{\epsilon=0}$. Thus, the smallest eigenvalue needs to be negative. In that case, initializing the newly added layer with $\epsilon \Phi_l$ corresponds to moving along a negative curvature direction, and thus decreasing the loss. If the smallest eigenvalue is non-negative, and thus $\Lambda_l$ is non-positive, we do not add a new layer at the depth $l$, as it would not improve the loss function.
\end{remark}
\begin{remark}
Note that if the geometric multiplicity of $\Lambda_l$ (the maximum eigenvalue in \eqref{quant}) is greater than 1, then one may choose  $\Phi_l$ (eigenvector in \eqref{quant}) corresponding to any eigenvalue as initialization for the new layer since all these initialization gives the same topological derivative.
\end{remark}

Our algorithm starts by training a small network for some $E_e$ epochs. Then, we compute the quantities in \eqref{quant}  for each layer $l$. At the $l^*$ location where the topological derivative is highest and positive, we insert a new layer there. That is, we add a new layer at the depth at which the loss function is most topologically sensitive. Our formulation \eqref{eigen_matrixx} also suggests that the initial value of the parameters of the newly added layer should be $\epsilon \Phi_l$. 
The new network $\Omega_\epsilon$ is trained again for $E_e$ epochs and the procedure is repeated until a suitable termination criteria is satisfied (see \cref{Algo_full}). The procedure in the case of a fully-connected network is provided in  \cref{Algo_full}. 

\subsection{Application to fully-connected neural network}

In this section, we derive the expression for the topological derivative $ d\sJ(\Omega_0;\ (l,\  \boldsymbol{\phi},\ \sigma))$ for a fully connected network. The Hamiltonian $H_l$ for the $l^{th}$  layer is written as (refer equation \eqref{hamiltonian}):
\begin{equation}
    H_l\LRp{\bx_{s,l};\bp_{s,l};\bW_{l+1};\bb_{l+1} }=\bp_{s,l}\cdot \LRs{\bx_{s,l}+\sigma(\bW_{l+1}^T \bx_{s,l}+\bb_{l+1})},
    \label{hamiltonian_full}
\end{equation}
where,  $\bW_{l+1}$ and $\bb_{l+1}$ denote the weights and biases. Further, without loss of generality let's assume $\bb_{l+1}={\bf{0}}$ 
for the hidden layers and vectorized parameters (weights) as $\btheta_{l+1}=\LRs{\bW_{l+1}^{11},\dots \bW_{l+1}^{n1}, \bW_{l+1}^{12},\dots \bW_{l+1}^{n2},\ \bW_{l+1}^{1N_h},\dots \bW_{l+1}^{n n}}$. Then,  the matrix $\bQ_l$ in \eqref{eigen_matrixx} can be explicitly written as
\begin{equation}
    \bQ_l=\frac{1}{2}\begin{bmatrix}
   \sum_{s=1}^S\bp_{s,l}^{(1)}\bx_{s,l}\bx_{s,l}^T\sigma''(0) & {\bf{0}}&{\bf{0}}  & \dots \\ 
  {\bf{0}} &  \sum_{s=1}^S\bp_{s,l}^{(2)}\bx_{s,l}\bx_{s,l}^T\sigma''(0)  & {\bf{0}} & \dots \\ 
  \vdots &  &  \ddots &  \\ 
  {\bf{0}} & \dots &  {\bf{0}} &  \sum_{s=1}^S\bp_{s,l}^{(n)}\bx_{s,l}\bx_{s,l}^T\sigma''(0)  
 \end{bmatrix},
 \label{block_diagonal}
\end{equation}
where $\bp_{s,l+1}^{(r)}$ denotes the $r^{th}$ component of the vector. Therefore, in the case of a fully connected network, $\bQ_l$ turns out to be a block-diagonal matrix. This structure can be further exploited to devise a strategy to choose the best subset of neurons in the new layer to activate in  \cref{select_neuron}. 

\begin{remark}[Assumptions on activation functions and their practical implications]
   { Note that \cref{activ_1} in \cref{act_const} gives conditions to be satisfied by the activation function $\sigma(x)$ in our framework. \cref{block_diagonal} shows that $\sigma''(0)\neq 0$ is an additional necessary condition to be satisfied for ensuring that $\bQ_l\neq {\bf{0}}$. There are several choices for $\sigma(x)$ that satisfies this condition. In particular, the choice of constructed activation functions mentioned in \cref{activ_2} of \cref{act_const} automatically satisfies this constraint. In \cref{poisson_sec}, \cref{stat_po}, we conducted a neural architecture search (NAS) experiment comparing the performance of our newly designed activation function (linear combination of Swish and tanh) with popular activation functions such as   ReLU and tanh. Our  results summarized in Table 1 shows that our activation function performs on par with, and in some cases slightly better than, ReLU and tanh activations.
Further, when applying our approach to fine-tuning or transfer learning, as demonstrated in \cref{trans_lear}, it is sufficient for the activation function in the newly added layer to satisfy the above constraints. Additionally, the new layer should include residual connections, while no constraints are imposed on the pretrained network. This demonstrates the broader utility of the technique beyond neural architecture adaptation.}
\end{remark}
\begin{remark}[Computational cost associated with solving the eigenvalue problem \eqref{eigen_matrixx}]
{   Let $\bQ_l^{(i)}$ be the $i^{th}$ block (along the diagonal) of $\bQ_l$ given by \eqref{block_diagonal}. Note that the eigenvalues of $\bQ_l$ are just the list of eigenvalues of each $\bQ_l^{(i)}$. \footnote{This follows from the fact that the characteristic polynomial of $\bQ_l$ is the product of the characteristic polynomials of all  $\bQ_l^{(i)}$ \cite{garcia2020block}} Therefore, one only needs to solve the eigenvalue problem for each block $\bQ_l^{(i)}$ separately thereby significantly reducing the computational cost. In particular, for a network with $k$ hidden layers of width $n$, using the QR algorithm \cite{trefethen2022numerical} for eigenvalue decomposition results in $kn$ embarrasingly parallel $\mathcal{O}(n^3)$ floating-point operations. Furthermore, since only the largest eigenvalue of  $\bQ_l^{(i)}$ is of interest, specialized methods (such as power method \cite{trefethen2022numerical}) may be employed that require fewer floating-point operations than the full QR algorithm.}
   
 An important consequence of the block-diagonal structure of $\bQ_l$  is that it enables an extension of our algorithm that activates only the most sensitive neurons in a newly added layer, rather than all $n$ neurons. The details of this extension are discused in \cref{select_neuron}. 
\end{remark}
%On the other hand, we can choose not to add a new layer, but to split the most sensitive neuron to grow the network in the width, and this is the subject of Section \ref{width_growing}.  The combination of both depth and width growing will be discussed in Section \ref{depth_width_growing}.

\subsubsection{Activating the most sensitive neurons in a newly added layer}
\label{select_neuron}
Let $\bQ_l^{(r)} := \sum_{s=1}^S\bp_{s,l}^{(r)}\bx_{s,l}\bx_{s,l}^T\sigma''(0) \in \real^{n \times n}$ be the $r$th diagonal block in $\bQ_l$ corresponding to the $r$th neuron. 
The following result is obvious.

\begin{proposition}
Let $\vb$ be a vector in  $\real^{n}$. Define
\begin{equation}
\Phi_r :=[\underbrace{{{0}},\ \dots }_{{(r-1)\times n.}}, \vb, \ {{0}} \dots {{0}}]^T.
\label{new_rep}
\end{equation}
Then $\vb$ is an eigenvector of $\bQ_l^{(r)}$ if and only if $\Phi_r$ is an eigenvector of $\bQ_l$. Furthermore, the eigenvalues corresponding to $\vb$ and $\Phi_r$ are the same.
\label{subBlockQl}
\end{proposition}

As a direct consequence of \cref{subBlockQl}, if $\vb$ is an eigenvector corresponding to the largest eigenvalue (among all eigenvalues of all diagonal blocks), then the corresponding $\Phi_r$ is associated with the largest eigenvalue of $\bQ_l$.

As a result, the shape functional is most sensitive to the $r$th neuron, and the $r$th neuron should be activated (i.e. initialized with non-zero weights/biases). Alternatively, given an integer $m$, we can identify ``the best" $m$ neurons such that when they are activated, the induced change in the shape functional is largest.

\begin{proposition}
\label{prop_fi}
    Consider the following constrained optimization problem:
\begin{equation}
%\begin{gather}
\label{quant_full_a}
 \max_{\bphi_i, \bphi_j \in \LRc{\Phi_k}_{k=1}^n,
 \bphi_i \ne \bphi_j}  \LRp{d\sJ(\Omega_0;\ (l,\  (\bphi_1+\dots \bphi_{m}),\ \sigma))}, \quad s.t \norm{\bphi_i}_2^2=1, \, i=1,\hdots,m. 
 %\ and,
% \label{quant_full_b}
%\text{pair $\bphi_i$, $\bphi_j$ is such that}  \ \ \bphi_i=[\underbrace{0,\dots }_{{(r-1)\times n.}}, \hat{\bphi}^{(r)}, \ 0 \dots],\ \bphi_j=[\underbrace{0,\dots }_{{(k-1)\times n.}}, \hat{\bphi}^{(k)}, \ 0 \dots] \ \ \text{with}\ \ r\neq k.
%\end{gather}
%\label{quant_full}
\end{equation}
Then, the optimal value is the sum of the first $m$ largest eigenvalues of $m$ diagonal sub-blocks of $\bQ_l$. The optimal solutions $\LRc{\bphi_i}_{i=1}^m$ are the corresponding eigenvectors of these diagonal sub-blocks written in the form given by \eqref{new_rep}.
\end{proposition}
\begin{proof}
    The first-order optimality condition can be obtained by first formulating the 
Lagrangian as:
\begin{equation}
\begin{aligned}
     L\LRp{\bphi_1;\dots, \bphi_m;\lambda_1;\dots \lambda_m}&=d\sJ(\Omega_0;\ (l,\  \varphi,\ \sigma))+\sum_{i=1}^{m}\lambda_i (1-\norm{\bphi_i}_2^2)\\
     &=\varphi^T \bQ_l \varphi+\sum_{i=1}^{m}\lambda_i (1-\norm{\bphi_i}_2^2),
\end{aligned} 
\label{lag}
\end{equation}
where $\varphi=\bphi_1+\dots \bphi_m$. Now notice that using the constraint $\bphi_i\neq \bphi_j$ in \eqref{quant_full_a} and $\bphi_i, \bphi_j \in \LRc{\Phi_k}_{k=1}^n$,  the Lagrangian can be rewritten in terms of the new variables $\{\hat{\bphi}_1,\dots \hat{\bphi}_m\}$ as:
\[ L\LRp{\hat{\bphi}_1;\dots; \hat{\bphi}_m;\lambda_1;\dots \lambda_m}=\sum_{i=1}^m \LRs{\hat{\bphi}_{i}}^T \bQ_l^{\iota(i)} \hat{\bphi}_{i} + \sum_{i=1}^{m}\lambda_i \LRp{ 1-\norm{\hat{\bphi}_{i}}_2^2},
\]
where $\iota\LRp{\cdot}: \LRc{1,\hdots,m} \ni i \mapsto \iota(i) \in \LRc{1,\hdots, n}$ is the mapping from the subscript of $\hat{\bphi}_i$ to the diagonal block number $\iota(i)$ in $\bQ_l$, and $\lambda_i$, $i=1,\hdots,m$ are the Lagrangian multiplier. Using the Lagrangian,  the first-order optimality condition of \eqref{quant_full_a} reads

\[ 
\bQ_l^{\iota(i)} \hat{\bphi}_{i} =\lambda_i  \hat{\bphi}_{i},\quad \forall i,
\]
that is, the pair $\LRc{\lambda_i ,\hat{\bphi}_{i}}$ is an eigenpair of the diagonal block $\iota(i)$ of $\bQ_l$. 

% As a result, the solution to the optimization problem \eqref{quant_full_a} is given as

% \begin{equation}
%     \Lambda_l= \sum_{i=1}^{m}\lambda_i, \quad  \Phi_l=\bphi_1+\dots \bphi_{m},
%     \label{compute_full}
% \end{equation}
% where, $\lambda_i$ denote the largest eigenvalue of  $Q_l^{(i)}$ and the corresponding eigenvector $\hat{\bphi}^{(i)}$  is used to construct $\bphi_i$ based on \ref{quant_full_b}. 

As a result, the optimal objective function in \eqref{quant_full_a} is the sum of the first $m$ largest eigenvalues of $m$ diagonal sub-blocks of $\bQ_l$.  The optimal solutions $\LRc{\bphi_i}_{i=1}^m$ are the corresponding eigenvectors of these diagonal sub-blocks written in the form \eqref{new_rep}. If we denote the $m$ largest eigenvalues (from highest to lowest) as $\{\lambda_i \}_{i=1}^m$, the solution is written as:
\begin{equation}
\Lambda_l=\sum_{i=1}^m \lambda_i, \quad \Phi_l=\bphi_i+\dots \bphi_m.
\label{compute_full}
\end{equation}

\end{proof}
Based on the above-computed quantities, our algorithm for adapting the depth of a fully connected feed-forward architecture is given in \cref{Algo_full}. 

\begin{remark}{}
\label{additional_rem}
Note that our framework can also be applied to a convolutional neural network (CNN) architecture and the details are provided in \cref{CNN_application}. We also provide a  numerical demonstration in \cref{additional_numerical_res}.
      %  \item In addition, in \cref{optimal_trans} we show that our topological derivative approach can also be viewed through the lens of an optimal transport problem. In particular, we show that our layer insertion strategy can be derived as a
%solution to maximizing a topological derivative in $p-$Wasserstein space, where $p\geq 1$.
\end{remark}
%\textcolor{blue}{Here, there might be a way to innovate for width adaptivity, where $a_m$ is also a variable treated as an unknown variable as we have both width and depth adaptivity together. Let's look at the following optimization problem for instance:}

%\[ r_l^1=\max_{\phi_1}d_T\sJ(\Omega_0,\ l,\ \epsilon \bphi_1,\ \sigma),\quad s.t \norm{\bphi_1}_2=1.\]

%Choose $a_m^l$ such that $\forall i\in [2,\ a_m^l]$:

%\[ r_l^1-\underbrace{\max_{\phi_{i}} d_T\sJ(\Omega_0,\ l,\ \epsilon \bphi_{i},\ \sigma)}_{r_l^i}\leq \epsilon,\quad s.t \norm{\bphi_i}_2=1\]

%\[ \text{pair $\bphi_i$, $\bphi_j$ is such that}  \ \ \bphi_i=[\underbrace{0,\dots }_{{(r-1)\times n.}}, \hat{\bphi_r}, \ 0 \dots],\ \bphi_j=[\underbrace{0,\dots }_{{(k-1)\times n.}}, \hat{\bphi_k}, \ 0 \dots] \ \ \text{with}\ \ r\neq k\]

%Then store:

%\[  \Lambda_l=r_l^1+\dots r_l^{a_m^l}\]

%\[  \Phi_l=\bphi_1^*+\dots \bphi_{a_m^l}^*.\]

%\textcolor{blue}{Write what the above constraint means. }

\begin{algorithm} 
	\caption{Fully connected architecture adaptation algorithm for a regression task}
	\hspace*{\algorithmicindent} \textbf{Input}: Training data $\bX$, labels $\bC$, validation data $\bX_1$, validation labels $\bC_1$, number of neurons in each hidden layer $n$, loss function $\Phi$, number of neurons to activate in each hidden layer ${m}$, number of iterations $N_n$, parameter $\epsilon$, $\epsilon^t$, parameter $T_b$, hyperparameters and predefined scheduler for optimizer (\cref{hyper_parameter_n}).\\
	\hspace*{\algorithmicindent} \textbf{Initialize}:   Initialize  network $\sQ_1$ with $T_b$ hidden layers.\\
	\begin{algorithmic}[1] 
  \State Train network $\sQ_1$ and store the validation loss $(\epsilon_v)^{1}$.
  		\State set $i =1$,  $(\epsilon_v)^{0}>>(\epsilon_v)^{1}$, $\Lambda_l^m\geq \epsilon^t$
		\While{$i \le N_n$ \textbf{and} $\LRs{(\epsilon_v)^{i}\leq (\epsilon_v)^{i-1}}$ \textbf{and} $\Lambda_l^m\geq \epsilon^t$} 
 \State Compute the topological derivative for each layer $l$  using \eqref{compute_full} and store as $\{ \Lambda_l\}$, also store $\Lambda_l^m=\max_l\{ \Lambda_l\}$
   \State Store the corresponding eigenvectors for each layer as $\Phi_l$ given by \eqref{compute_full}.
        \State Obtain the new network $\sQ_{i+1}$ by adding a new layer at position $l^*=\argmax\limits_l \{\Lambda_l\}$ with initialization  $\epsilon\Phi_{l^*}$.
        \State Perform a backtracking line search to update $\epsilon$ as outlined in \cref{back_track}.
        \State Update the parameters for optimizer if required (refer  \cref{hyper_parameter_n} for scheduler details).
%        \State Perform a backtracking algorithm for finding the best $\epsilon\in [0,\ 1]$ that gives the maximum decrease in training 
      %    \hspace*{\algorithmicindent}  loss.
        \State Train network $\sQ_{i+1}$  and store the best validation loss $(\epsilon_v)^{i+1}$ and the best network $\sQ_{i+1}$.
		\State $i = i+1$
		\EndWhile
	\end{algorithmic} \label{Algo_full}
\hspace*{\algorithmicindent} \textbf{Output}: Network $\sQ_{i-1}$
\end{algorithm}

\section{Fully automated network growing \cref{Algo_full_auto}}
\label{automated}
Note that \cref{Algo_full} activates the ``the best" $m$ neurons while inserting a new layer as described in  \cref{select_neuron}. Note that $m$ is a user-defined hyperparameter in \cref{Algo_full}. In this section, we extend our algorithm to automatically select the parameter $m$, i.e. automatically determine the number of neurons (width) of the added hidden layer. To that end let us revisit
 \cref{prop_fi} where  $m$ is now unknown. The idea is to look at how the optimal loss in \eqref{quant_full_a} changes as one increases $m$. In particular, for each layer $l$ we choose the largest value of $m$ such that:
\begin{equation}
\frac{\lambda_{1}-\lambda_{m}}{\lambda_1}\leq \epsilon_s,\quad s.t \ \ \lambda_1,\ \lambda_m>0,
\label{additional_constraint}
\end{equation}
where $\lambda_i$ is  defined in \eqref{compute_full} and $\epsilon_s$ is a user-defined hyperparameter. Since now $m$ could be different for each layer $l$, let us denote it as $m(l)$. Criterion \eqref{additional_constraint} simply compares the sensitivity of each neuron relative to the most sensitive neuron in the added layer. 
If $\epsilon_s=0$, then it is clear that only the most sensitive neuron gets activated. On the other hand, when $\epsilon_s=1$ all the neurons (with $\lambda_m>0$) in the hidden layer gets activated. In practice, we choose $\epsilon_s=0.5$ such that all neurons which are at least $50\ \%$ as sensitive as the most sensitive neuron gets activated in the added layer. Further, note that the derivative $\Lambda_l$ in \eqref{compute_full} is computed with respect to the operation of activating exactly $m$ neurons in an added layer. However, since in this case $m=m(l)$ can be different for each layer $l$, one needs to normalize the derivative in \eqref{compute_full} as $\Lambda_l=\frac{\Lambda_l}{m(l)}$  while making decision on where to add a new layer. The full algorithm is provided in  \cref{Algo_full_auto}.

Further, in the fully automated network growing (Proposed II), we do not use a predefined scheduler (such as in \cref{Algo_full}) to decide when a new layer needs to be added during the training process. Instead, a validation data set is employed to decide when a new layer needs to be added.  The key difference of fully-automated growing from \cref{Algo_full} is the training loop without using a predefined scheduler (lines 9-17 in  \cref{Algo_full_auto}). Note that one stops training the current network $\sQ_{i+1}$ if the validation loss does not decrease for $N_k$ consecutive training epochs. Then, a new layer is added. 
The extension of \cref{Algo_full_auto} to the case of a CNN architecture is trivial (in this case the parameter $m$  corresponds to the number of channels in the hidden layer).
    \begin{algorithm} 
	\caption{Fully automated network growing algorithm for fully connected architecture}
	\hspace*{\algorithmicindent} \textbf{Input}: Training data $\bX$, labels $\bC$, validation data $\bX_1$, validation labels $\bC_1$, number of neurons in each hidden layer $n$, loss function $\Phi$, number of iterations $N_n$,  parameter $\epsilon$, $\epsilon^t$, parameter $T_b,\ N_k$, hyperparameters for optimizer (\cref{hyper_parameter_n}).\\
	\hspace*{\algorithmicindent} \textbf{Initialize}:   Initialize  network $\sQ_1$ with $T_b$ hidden layers.\\
	\begin{algorithmic}[1] 
  \State Train network $\sQ_1$  and store the best validation loss $(\epsilon_v)^{1}$.
  		\State set $i =1$,  $(\epsilon_v)^{0}>>(\epsilon_v)^{1}$, $\Lambda_l^m\geq \epsilon^t$
		\While{$i \le N_n$ \textbf{and} $\LRs{(\epsilon_v)^{i}\leq (\epsilon_v)^{i-1}}$ \textbf{and} $\Lambda_l^m\geq \epsilon^t$} 
 \State Compute $m(l)$ for each layer $l$ using   \eqref{additional_constraint} as outlined in  \cref{automated}.
 \State Compute the derivative for each layer $l$  using \eqref{compute_full} and store as $\{ \frac{\Lambda_l}{m(l)}\}$, also store $\Lambda_l^m=\max_l\{ \frac{\Lambda_l}{m(l)}\}$
   \State Store the corresponding eigenvectors for each layer as $\Phi_l$ given by \eqref{compute_full}.
        \State Obtain the new network $\sQ_{i+1}$ by adding a new layer at position $l^*=\argmax\limits_l \{ \frac{\Lambda_l}{m(l)}\}$ with initialization  $\epsilon\Phi_{l^*}$.
        \State Perform a backtracking line search to update $\epsilon$ as outlined in \cref{back_track}.
%        \State Perform a backtracking algorithm for finding the best $\epsilon\in [0,\ 1]$ that gives the maximum decrease in training 
      %    \hspace*{\algorithmicindent}  loss.
\State Set epoch $j=1$, $(\epsilon_v)^{i+1}_j$ as the validation loss for $\sQ_{i+1}$, and set $k=1$.
      \While{$k \leq N_k$}\Comment{Training loop without using a scheduler}
        \State Update parameters of network $\sQ_{i+1}$   and store the current best validation loss $(\epsilon_v)^{i+1}_{j+1}$. 
        \If{$(\epsilon_v)^{i+1}_{j+1}\geq (\epsilon_v)^{i+1}_{j}$}
        \State k=k+1
                \Else
        \State set $k=1$
        \EndIf
		\State $j = j+1$
		\EndWhile
  \State Store the best validation loss for $\sQ_{i+1}$ as $(\epsilon_v)^{i+1}$=$(\epsilon_v)^{i+1}_{j+1}$ and the corresponding  best network $\sQ_{i+1}$.
  \State $i = i+1$
  \EndWhile
	\end{algorithmic} \label{Algo_full_auto}
\hspace*{\algorithmicindent} \textbf{Output}: Network $\sQ_{i-1}$
\end{algorithm}

\section{Topological derivative approach through the lens of an  optimal transport problem}
\label{optimal_trans}
%\krish{The below section completely revised to derive the topological derivative in $p-$Wasserstein space and the optimization problem (Theorem C.2).}
In this section, we show that our layer insertion strategy can be derived  as a solution to maximizing a topological derivative in $p$-Wasserstein space, where $p\geq 1$. Optimization in $\infty-$Wasserstein space was earlier considered by Liu et al. \cite{wu2019splitting} to explain their neuron splitting strategy for growing neural networks along the width. The standard optimal transport (Kantorovich formulation \cite{santambrogio2015optimal}) aims at computing a minimal cost transportation plan $\gamma$  between a source probability measure $\mu$, and target probability measure $\nu$.
%Note that the plan $\gamma$ is a joint probability measure whose marginals are $\mu$ and $\nu$ on the first and second factors, respectively.
A common formulation of the cost function is based on the Wasserstein distance  $W_p (\mu;\ \nu)$.
\begin{definition}[$p$-Wasserstein metric]
The  $p$-Wasserstein distance between two probability measures $\mu$ and $\nu$ is given as:
\begin{equation}
    W_p \LRp{\mu;\ \nu}={\inf_{\gamma \in \Gamma (\mu;\ \nu)} \LRp{\underset{(\bphi,\ \bphi') \sim \gamma}{\expect}  \norm{\bphi-\bphi'}_2^p}^{1/p}},
    \label{wasr}
\end{equation}
    where  $\Gamma (\mu;\ \nu)$ is the set of couplings of $\mu$ and $\nu$. A coupling $\gamma$ is a joint probability measure whose marginals are $\mu$ and $\nu$ on the first and second factors, respectively. 
    %$esssup$ denotes the essential supremum \footnote{essential supremum denotes the smallest number $c$ such that the set $\{ (\bphi,\ \bphi'): \norm{\bphi-\bphi'}_2 >c\}$ has zero probability under $\gamma$. } of $\norm{\bphi-\bphi'}_2$ with respect to measure $\gamma$.

\end{definition}
In this work, {\em we set up an optimization problem  to determine the unknown target measure $\nu$ given the source measure $\mu$.} Let us consider the scenario depicted in \cref{fig:TRANSPORT} in which we add a message-passing layer  with zero weights/biases, and then ask the question:  what is the initialization values $\overline{T}({\bf{0}})$ (where $\overline{T}$ is the transport map) for the weights/biases of the newly added layer such that the loss function is most sensitive too? In this section, we answer this question from an optimal transport point of view.  For the clarity of the exposition (and without loss of generality), let us assume that the message-passing layer (new layer) is added in front of the output layer of $\Omega_0$ as depicted in 
\cref{fig:TRANSPORT}.
\def\layersep{1.4cm}
\def\nodeinlayersep{0.7cm}
\begin{figure}[H]
\centering
\begin{tikzpicture}[
    node distance=\layersep,
    edge/.style={-stealth,shorten >=1pt, draw=black!50,thin},
    neuron/.style={circle,fill=black!25,minimum size=10pt,inner sep=0pt},
    operator/.style={rectangle,fill=green!,minimum height= \nodeinlayersep, minimum width= 0.8 * \layersep, inner sep=0pt, rounded corners},
    input neuron/.style={neuron, fill=green!50,minimum size=12pt},
    output neuron/.style={neuron, fill=green!50,minimum size=12pt},
    hidden neuron/.style={neuron, fill=blue!50},
    Forward map/.style={operator, fill=red!50},
    annot/.style={text width=4em, text centered},
    every node/.style={scale=1.0},
    node1/.style={scale=2.0}
]
    % Draw the input layer nodes
    \foreach \name / \y in {1,...,4}
        {\ifnum \y=3
            \node (I-\name) at (0,-\nodeinlayersep * \y - \nodeinlayersep * 0.5) {$\vdots$};
            % \node[input neuron] (I-\name) at (0,-\y-0.4) {$y_{\y}$};
        \else
            \ifnum \y=4
                \node[input neuron] (I-\name) at (0,-\nodeinlayersep * \y - \nodeinlayersep * 0.5) {$x_{n}$};
            \else
                \node[input neuron] (I-\name) at (0,-\nodeinlayersep *\y - \nodeinlayersep * 0.5 ) {$x_{\y}$};
            \fi
        \fi}
        
    % Draw the output layer node
    \foreach \name / \y in {1,...,5}
        {\ifnum \y=4
            \node (O-\name) at (3*\layersep,-\nodeinlayersep *\y - \nodeinlayersep * 0.0) {$\vdots$};
      
            % \node[input neuron] (I-\name) at (0,-\y-0.4) {$y_{\y}$};
        \else
            \ifnum \y=5
                \node[input neuron] (O-\name) at (3*\layersep,-\nodeinlayersep *\y - \nodeinlayersep * 0.0) {$y_{p}$};
            \else
                \node[input neuron] (O-\name) at (3*\layersep,-\nodeinlayersep *\y - \nodeinlayersep * 0.0) {$y_{\y}$};
            \fi
        \fi}

    % set number of hidden layers
    \newcommand \Nhidden{2}
    % Draw the hidden layer nodes
    \foreach \N in {1,...,\Nhidden} {
        \foreach \y in {1,...,5} { %%% MODIFIED (1,...,12 -> 1,...,5, and the next five lines)
            \ifnum \y=4
                \node at (\N*\layersep,-\y*\nodeinlayersep) {$\vdots$};
            \else
                \node[hidden neuron] (H\N-\y) at (\N*\layersep,-\y*\nodeinlayersep ) {$\sigma$};
        \fi
      }
    }

    %%% <-- MODIFIED (from H\Nhidden-6 to H\Nhidden-3) 
    % Connect every node in the input layer with every node in the
    % hidden layer.
    \foreach \source in {1,2,4}
        \foreach \dest in {1,...,3,5} %%% <-- MODIFIED (1,...,12 -> 1...,3,5)
            \draw[edge] (I-\source) -- (H1-\dest);
    
    % connect all hidden stuff
    \foreach [remember=\N as \lastN (initially 1)] \N in {2,...,\Nhidden}
      \foreach \source in {1,...,3,5} %%% <-- MODIFIED (1,...,12 -> 1...,3,5)
          \foreach \dest in {1,...,3,5} %%% <-- MODIFIED (1,...,12 -> 1...,3,5)
              \draw[edge] (H\lastN-\source) -- (H\N-\dest);
              
    % Connect every node in the hidden layer with the output layer
    \foreach \source in {1,...,3,5} %%% <-- MODIFIED (1,...,12 -> 1...,3,5)
        \foreach \dest in {1,2,4,5}
            \draw[dashed] (H\Nhidden-\source) -- (O-\dest);

    \draw [decorate, decoration = {calligraphic brace, mirror}, thick]  (1*\layersep,-.7*\nodeinlayersep) -- (1*\layersep,-.7*\nodeinlayersep) node[pos=1.5,below=3.5cm,black]{$l=1$};
  
    \draw [decorate, decoration = {calligraphic brace, mirror}, thick]  (2.5*\layersep,-.7*\nodeinlayersep) -- (2.5*\layersep,-.7*\nodeinlayersep) node[pos=1.5,above=0.5cm,black]{message-pasing layer};

    \draw [decorate, decoration = {calligraphic brace, mirror}, thick]  (2.5*\layersep,-.7*\nodeinlayersep) -- (2.5*\layersep,-.7*\nodeinlayersep) node[pos=1.5,above=0.1cm,black]{ with parameters ${\bf{0}}$};
    
    \draw [decorate, decoration = {calligraphic brace, mirror}, thick]  (2*\layersep,-.7*\nodeinlayersep) -- (2*\layersep,-.7*\nodeinlayersep) node[pos=1.5,below=3.5cm,black]{$l=2$};
       \draw [decorate, decoration = {calligraphic brace, mirror}, thick]  (1.5*\layersep,-.7*\nodeinlayersep) -- (1.5*\layersep,-.7*\nodeinlayersep) node[pos=1.5,below=4cm,black]{Original network $\Omega_0$};

      % \draw [dashed] (5.5,-4.8) -- (5.5,0.5);
          \draw[->, thick] (3.5,-0.3) -- (3.5,-0.65);
  \draw[->, thick] (5.8,-0.1) -- (9.5,-0.1) node[pos=0.5,above=0cm,black]{Transport of parameters};
\end{tikzpicture}
\hspace{-2 cm}
\begin{tikzpicture}[
    node distance=\layersep,
    edge/.style={-stealth,shorten >=1pt, draw=black!50,thin},
    neuron/.style={circle,fill=black!25,minimum size=10pt,inner sep=0pt},
    operator/.style={rectangle,fill=green!,minimum height= \nodeinlayersep, minimum width= 0.8 * \layersep, inner sep=0pt, rounded corners},
    input neuron/.style={neuron, fill=green!50,minimum size=12pt},
    output neuron/.style={neuron, fill=green!50,minimum size=12pt},
    hidden neuron/.style={neuron, fill=blue!50},
    Forward map/.style={operator, fill=red!50},
    annot/.style={text width=4em, text centered},
    every node/.style={scale=1.0},
    node1/.style={scale=2.0}
]
    % Draw the input layer nodes
    \foreach \name / \y in {1,...,4}
        {\ifnum \y=3
            \node (I-\name) at (0,-\nodeinlayersep * \y - \nodeinlayersep * 0.5) {$\vdots$};
            % \node[input neuron] (I-\name) at (0,-\y-0.4) {$y_{\y}$};
        \else
            \ifnum \y=4
                \node[input neuron] (I-\name) at (0,-\nodeinlayersep * \y - \nodeinlayersep * 0.5) {$x_{n}$};
            \else
                \node[input neuron] (I-\name) at (0,-\nodeinlayersep *\y - \nodeinlayersep * 0.5 ) {$x_{\y}$};
            \fi
        \fi}
        
    % Draw the output layer node
    \foreach \name / \y in {1,...,5}
        {\ifnum \y=4
            \node (O-\name) at (3*\layersep,-\nodeinlayersep *\y - \nodeinlayersep * 0.0) {$\vdots$};
      
            % \node[input neuron] (I-\name) at (0,-\y-0.4) {$y_{\y}$};
        \else
            \ifnum \y=5
                \node[input neuron] (O-\name) at (3*\layersep,-\nodeinlayersep *\y - \nodeinlayersep * 0.0) {$y_{p}$};
            \else
                \node[input neuron] (O-\name) at (3*\layersep,-\nodeinlayersep *\y - \nodeinlayersep * 0.0) {$y_{\y}$};
            \fi
        \fi}

    % set number of hidden layers
    \newcommand \Nhidden{2}
    % Draw the hidden layer nodes
    \foreach \N in {1,...,\Nhidden} {
        \foreach \y in {1,...,5} { %%% MODIFIED (1,...,12 -> 1,...,5, and the next five lines)
            \ifnum \y=4
                \node at (\N*\layersep,-\y*\nodeinlayersep) {$\vdots$};
            \else
                \node[hidden neuron] (H\N-\y) at (\N*\layersep,-\y*\nodeinlayersep ) {$\sigma$};
        \fi
      }
    }

    %%% <-- MODIFIED (from H\Nhidden-6 to H\Nhidden-3) 
    % Connect every node in the input layer with every node in the
    % hidden layer.
    \foreach \source in {1,2,4}
        \foreach \dest in {1,...,3,5} %%% <-- MODIFIED (1,...,12 -> 1...,3,5)
            \draw[edge] (I-\source) -- (H1-\dest);
    
    % connect all hidden stuff
    \foreach [remember=\N as \lastN (initially 1)] \N in {2,...,\Nhidden}
      \foreach \source in {1,...,3,5} %%% <-- MODIFIED (1,...,12 -> 1...,3,5)
          \foreach \dest in {1,...,3,5} %%% <-- MODIFIED (1,...,12 -> 1...,3,5)
              \draw[edge] (H\lastN-\source) -- (H\N-\dest);
              
    % Connect every node in the hidden layer with the output layer
    \foreach \source in {1,...,3,5} %%% <-- MODIFIED (1,...,12 -> 1...,3,5)
        \foreach \dest in {1,2,4,5}
            \draw[edge] (H\Nhidden-\source) -- (O-\dest);

    \draw [decorate, decoration = {calligraphic brace, mirror}, thick]  (1*\layersep,-.7*\nodeinlayersep) -- (1*\layersep,-.7*\nodeinlayersep) node[pos=1.5,below=3.5cm,black]{$l=1$};
  
    \draw [decorate, decoration = {calligraphic brace, mirror}, thick]  (2.5*\layersep,-.7*\nodeinlayersep) -- (2.5*\layersep,-.7*\nodeinlayersep) node[pos=1.5,above=0.5cm,black]{New output layer with };

    \draw [decorate, decoration = {calligraphic brace, mirror}, thick]  (2.5*\layersep,-.7*\nodeinlayersep) -- (2.5*\layersep,-.7*\nodeinlayersep) node[pos=1.5,above=0.1cm,black]{ parameters $\overline{T}\LRp{{\bf{0}}}$};
    
    \draw [decorate, decoration = {calligraphic brace, mirror}, thick]  (2*\layersep,-.7*\nodeinlayersep) -- (2*\layersep,-.7*\nodeinlayersep) node[pos=1.5,below=3.5cm,black]{$l=2$};
       \draw [decorate, decoration = {calligraphic brace, mirror}, thick]  (1.5*\layersep,-.7*\nodeinlayersep) -- (1.5*\layersep,-.7*\nodeinlayersep) node[pos=1.5,below=4cm,black]{New network $\Omega_\epsilon$};

      % \draw [dashed] (5.5,-4.8) -- (5.5,0.5);
      \draw[->, thick] (3.5,-0.3) -- (3.5,-0.65);

\end{tikzpicture}
\caption{Optimal transport interpretation of our proposed approach:  We wish to optimally transport (in some sense) the parameters from network $\Omega_0$ (left figure) to the new network $\Omega_\epsilon$ (right figure).  }
\label{fig:TRANSPORT}
\end{figure}
For the message-passing layer with zero weights/biases, the functional representation of the loss \eqref{for_later} can be written as: 
\begin{equation}
\begin{aligned}
    \sJ\LRp{\mu_{\bf{0}}}=\frac{1}{S}\sum_{s=1}^S \Phi \LRp{ \underset{\bphi \sim \mu_{\bf{0}} }{\expect}    \LRs{\Omega_\epsilon \LRp{\bx_{s,0};\ \bphi}}}, \quad \quad %\mu=\frac{1}{r}\sum_{i=1}^{r} \delta_{\bphi_i}, 
    \mu_{\bf{0}}=\delta_{0} \times \delta_{0}\times \dots \mathrm(n_p\ times)\dots \times \delta_{0},
    \end{aligned}
    \label{delta_measure}
\end{equation}
where the second argument in  $\Omega_\epsilon \LRp{\bx_{s,0};\ \bphi}$ denotes the parameters $\bphi$ of the added layer, $\delta_{0}$ denotes the Dirac measure on $\real$, $n_p$ denotes the total number of parameters in the added layer, $\mu_{\bf{0}}$ is the product measure. Similarly, for some non-zero initialization values $\epsilon\boldsymbol{\phi}$ (which we seek to determine) for the weights/biases of the newly added layer, the functional representation of the loss \eqref{for_later} can be written as: 
\begin{equation}
\begin{aligned}
    \sJ(\nu_{\epsilon\bphi})=\frac{1}{S}\sum_{s=1}^S \Phi \LRp{ \underset{\bphi \sim \nu_{\epsilon\bphi} }{\expect}    \LRs{\Omega_\epsilon \LRp{\bx_{s,0};\ \bphi}}}, \quad \quad %\mu=\frac{1}{r}\sum_{i=1}^{r} \delta_{\bphi_i}, 
    \nu_{\epsilon\bphi}=\delta_{\epsilon\bphi_1} \times \delta_{\epsilon\bphi_2}\times \dots \times \delta_{\epsilon\bphi_{n_p}},
    \end{aligned}
    \label{delta_measure_v}
\end{equation}

\begin{theorem}[Topological derivative in $p-$Wasserstein space]
\label{lem_opti}
Let $\mu_{\bf{0}}$ be a given probability measure  \eqref{delta_measure} and assume that conditions \cref{one} and, \cref{two} in \cref{prop_admissible} are satisfied for the neural network. Let $\nu_{\epsilon\bphi}$, be another (unknown) perturbed measure \eqref{delta_measure_v} such that $ W_p \LRp{\mu_{\bf{0}};\ \nu_{\epsilon\bphi}}\leq \epsilon$. Further, assume that the conditions \ref{cond_one}, \ref{cond_two} of  \cref{exist_th} are satisfied. 
Then, the following results holds:
\begin{enumerate}
\item The probability measure $\nu_{\epsilon\bphi}$  has the form $\nu_{\epsilon\bphi}=\overline{T}_{\#}\mu$ where the transport map $\overline{T}$ satisfies:
\begin{equation}
    \overline{T}({\bf{0}})=\epsilon\bphi, \quad \mathrm{where}\ \norm{\bphi}_2\leq 1,
    \label{transp}
\end{equation}
    \item With $ W_p \LRp{\mu_{\bf{0}};\ \nu_{\epsilon\bphi}}\leq \epsilon$, the Topological derivative in $p-$Wasserstein space at $\mu_{\bf{0}}$ along the direction $\bphi$ is given by: 
    \begin{equation}
    d\sJ(\mu_{\bf{0}};\ \bphi)=-\lim_{\epsilon \downarrow 0}\frac{\sJ(\nu_{\epsilon\bphi})-\sJ(\mu_{\bf{0}})}{\epsilon^2}= \frac{\bphi^TQ(\bf{0})\bphi}{2},\quad \norm{\bphi}_2\leq 1,
    \label{topo_der_w}
\end{equation}
where  \[ Q(\bphi)=\sum_{s=1}^S \nabla^2_{\bphi} H_2\LRp{\bx_{s,2};\ \bp_{s,2}; \ \bphi},\]
and $H_2$ is the Hamiltonian  given by \eqref{hamiltonian} corresponding to the output layer. 
 %where,  $\mathbf{D}^2J(\mu)[\nu]$ denotes the second variation of the functional $J(\mu)$ computed based on the   displacement interpolation \cite{santambrogio2015optimal, carmona2018probabilistic,wu2019splitting} given by:
%\[  \mathbf{D}^2J(\mu)[\nu]=\LRs{\frac{d^2}{d\eta^2}\sJ(\mu_\eta)}_{\eta=0},\ \ \mu_\eta=(\pi_\eta)_{\#}\mu,\]
%where $\pi_\eta=(1-\eta)\bphi+\eta\overline{T}(\bphi)$ with $\eta=[0,\ 1]$.
\item Consider the problem of maximizing the Topological derivative in $p-$Wasserstein space:
\[ \max_{\bphi} \LRp{d\sJ(\mu_{\bf{0}};\ \bphi)},\quad \text{subject to:} \norm{\bphi}_2\leq 1.\]
Then the optimal transport map in \eqref{transp} satisfies the following:
\[ \overline{T}({\bf{0}})= \epsilon \ v_{max}({\bf{0}}),\]
where $v_{max}({\bf{0}})$ is the eigenvector of $Q({\bf{0}})$ corresponding to  the maximum eigenvalue. 
\end{enumerate}

\end{theorem}

\begin{proof}
\begin{inlineenum}
    \item The first part of the proof follows from the fact that for Dirac measures $\mu_{\bf{0}},\ \nu_{\epsilon\bphi}$, $p-$Wasserstein distance  \cref{wasr} simplifies as:
\[W_p \LRp{\mu_{\bf{0}};\ \nu_{\epsilon\bphi}}=\norm{\overline{T}({\bf{0}})-{\bf{0}}}_2.\]
Therefore,
\[W_p \LRp{\mu_{\bf{0}};\ \nu_{\epsilon\bphi}}\leq  \epsilon  \implies \norm{\overline{T}({\bf{0}})-{\bf{0}}}_2\leq \epsilon\implies \overline{T}({\bf{0}})=\epsilon\bphi, \quad \norm{\bphi}_2\leq 1. \]
\item To derive the Topological derivative in $p-$Wasserstein space, we look at  the Taylor series expansion of the functional $\sJ$ about the measure $\mu_{\bf{0}}$ as follows:
\begin{equation}
    \sJ(\nu_{\epsilon\bphi})-\sJ(\mu_{\bf{0}})= \mathbf{D}J(\mu_{\bf{0}})[\nu_{\epsilon\bphi}]+ \frac{1}{2}\mathbf{D}^2J(\mu_{\bf{0}})[\nu_{\epsilon\bphi}]+\frac{1}{6}\LRs{\frac{d^3}{d\eta^3}\sJ(\mu_\eta)}_{\eta=\zeta},
    \label{variation_taylor}
\end{equation}
where $\zeta$ is a number between $0$ and $1$, $\mathbf{D}^kJ(\mu_{\bf{0}})[\nu_{\epsilon\bphi}]$ denotes the $k^{th}$ variation of the functional $J(\mu_{\bf{0}})$ computed based on the   displacement interpolation \cite{santambrogio2015optimal, carmona2018probabilistic,wu2019splitting} given by:
\[  \mathbf{D}^kJ(\mu_{\bf{0}})[\nu_{\epsilon\bphi}]=\LRs{\frac{d^k}{d\eta^k}\sJ(\mu_\eta)}_{\eta=0},\ \ \mu_\eta=(\pi_\eta)_{\#}\mu_{\bf{0}},\]
where $\pi_\eta=(1-\eta)\bphi+\eta\overline{T}(\bphi)$ with $\eta=[0,\ 1]$ (note that $\mu_0=\mu_{{\bf{0}}}$ and $\mu_1=\nu_{\epsilon\bphi}$). Analyzing the first-order term in \eqref{variation_taylor} yields:
\begin{equation}
    \begin{aligned}
  \mathbf{D}J(\mu_{\bf{0}})[\nu_{\epsilon\bphi}]&=\frac{d}{d\eta} \LRp{\frac{1}{S}\sum_{s=1}^S \Phi \LRp{\underset{\bphi_{\eta} \sim \mu_{\eta} }{\expect}  
  \LRs{\Omega_\epsilon \LRp{\bx_{s,0};\ \bphi_\eta}}}}\Bigg  |_{\eta=0} \\
    &= \frac{d}{d\eta} \LRp{\frac{1}{S}\sum_{s=1}^S \Phi \LRp{\underset{\bphi \sim \mu_{\bf{0}} }{\expect}    
    \LRs{\Omega_\epsilon \LRp{\bx_{s,0};\ \eta \overline{T}(\bphi)+(1-\eta)\bphi}}}}\Bigg  |_{\eta=0}\\
    &= \frac{1}{S}\sum_{s=1}^S \LRs{\nabla \Phi \LRp{\underset{\bphi \sim \mu_{\bf{0}} }{\expect}    
    \LRs{\Omega_\epsilon (\bx_{s,0};\ \bphi)}} \cdot \underset{\bphi \sim \mu_{\bf{0}} }{\expect}    
    \LRs{\LRp{\nabla_{\bphi} \ \Omega_\epsilon (\bx_{s,0};\ \bphi)} \LRp{\overline{T}(\bphi)-\bphi} } },
\end{aligned}
\label{last_line}
\end{equation}
where $\nabla_{\bphi}$ denotes the gradient w.r.t the second argument of $\Omega_{\epsilon}$. 
Further  note that the second term in the last line of \eqref{last_line}  can be simplified as:
\[\underset{\bphi \sim \mu_{\bf{0}} }{\expect}\LRs{\LRp{\nabla_{\bphi} \ \Omega_\epsilon (\bx_{s,0};\ \bphi)} \LRp{\overline{T}(\bphi)-\bphi} } =\nabla_{\bphi} \ \Omega_\epsilon (\bx_{s,0};\ {\bf{0}}) \overline{T}({\bf{0}}) -  {(\nabla_{\bphi} \ \Omega_\epsilon (\bx_{s,0};\ {\bf{0}})) {\bf{0}}}={\bf{0}},\]
%\[=\nabla_{\bphi} \ \Omega_\epsilon (\bx_{s,0};\ {\bf{0}}) \overline{T}({\bf{0}}) -  {(\nabla_{\bphi} \ \Omega_\epsilon (\bx_{s,0};\ {\bf{0}})) {\bf{0}}}={\bf{0}},\]
where we have used the property of the Dirac-measure as defined in \eqref{delta_measure}  and also used 
$\nabla_{\bphi} \ \Omega_\epsilon (\bx_{s,0};\ {\bf{0}})={\bf{0}}$ due to
to assumption (\cref{two}) in  \cref{lem_opti}. Proceeding similarly for computing the second-order term $\mathbf{D}^2J(\mu_{\bf{0}})[\nu_{\epsilon\bphi}]$ in \eqref{variation_taylor}, we have:
\begin{equation}
    \begin{aligned}
 &  \mathbf{D}^2J(\mu_{\bf{0}})[\nu_{\epsilon\bphi}]= \frac{d}{d\eta} \LRs{\frac{1}{S}\sum_{i=1}^S \LRs{\nabla \Phi \LRp{\underset{\bphi \sim \mu_{\bf{0}} }{\expect}  
 \LRs{\Omega_\epsilon \LRp{\bx_{s,0};\ \pi_\eta}}} \cdot \underset{\bphi \sim \mu_{\bf{0}} }{\expect}  
 \LRs{\LRp{\nabla_{\bphi} \ \Omega_\epsilon (\bx_{s,0};\ {\pi_\eta)}} g(\bphi) } }}\Bigg  |_{\eta=0}=\\
  & \frac{1}{S}\sum_{s=1}^S \LRs{\LRp{\nabla^2\Phi   
  \LRs{     \underset{\bphi \sim \mu_{\bf{0}} }{\expect}   \LRs{\Omega_\epsilon (\bx_{s,0};\ \bphi)}}\hspace{-0.1 cm}\underset{\bphi \sim \mu_{\bf{0}} }{\expect}   \LRs{\LRp{\nabla_{\bphi} \ \Omega_\epsilon (\bx_{s,0};\ {\bphi})} (g(\bphi))}} \cdot \hspace{-0.2 cm} \underset{\bphi \sim \mu_{\bf{0}} }{\expect}   
  \LRs{\nabla_{\bphi} \ \Omega_\epsilon (\bx_{s,0};\ {\bphi}) g(\bphi)} }\\
  &+\sum_{s=1}^S \LRs{  \underset{\bphi \sim \mu_{\bf{0}} }{\expect}   
  \LRs{g(\bphi)^T \ \nabla^2_{\bphi} \ (-\bp_{s,T} \cdot \Omega_\epsilon (\bx_{s,0};\ {\bphi}))\ g(\bphi) }},
  \end{aligned}
  \label{second_order}
\end{equation}
where $g(\bphi)=\overline{T}(\bphi)-\bphi$ and we have used \eqref{set_of} in the second term of \eqref{second_order} while applying the chain rule, i.e,
\[\bp_{s,T}=-\frac{1}{S}\nabla\Phi\LRp{\bx_{s,T}}=
-\frac{1}{S} \nabla\Phi \LRp{\underset{\bphi \sim \mu_{\bf{0}} }{\expect} 
\LRs{\Omega_\epsilon \LRp{\bx_{s,0};\ \bphi}}}.\]
Now note that the first term in \eqref{second_order} vanishes due to the property of the Dirac-measure as defined in \eqref{delta_measure}  and also used 
$\nabla_{\bphi} \ \Omega_\epsilon (\bx_{s,0};\ {\bf{0}})={\bf{0}}$ due to assumptions (\cref{two}) in  \cref{lem_opti}.
Further, note that from \eqref{hamiltonian} we have: 
  \[ H_2\LRp{\bx_{s,2};\ \bp_{s,2};\ \bphi}=\bp_{s,T}  \cdot \Omega_\epsilon (\bx_{s,0};\ {\bphi})=\bp_{s,T}  \cdot \bff_3\LRp{\bx_{s, 2};\ \bphi}.\]
where we have used $\bp_{s,2}=\bp_{s,3}=\bp_{s,T}$ since the added layer with zero weights and biases acts as a message-passing layer due to \cref{prop_admissible}.  Using the above result and the definition of $Q(\bphi)$, \eqref{second_order} can be simplified as:
  \[   \mathbf{D}^2J(\mu_{\bf{0}})[\nu_{\epsilon\bphi}]= -\sum_{s=1}^S \underset{\bphi \sim \mu_{\bf{0}} }{\expect}  
  \LRs{\LRp{\overline{T}(\bphi)-\bphi}^T \ \nabla^2_{\bphi} H_2\LRp{\bx_{s,2};\bp_{s,2};\bphi}\LRp{\overline{T}(\bphi)-\bphi} }\]
  \[=-\underset{\bphi \sim \mu_{\bf{0}} }{\expect}\LRs{\LRp{\overline{T}(\bphi)-\bphi}^T\  Q(\bphi) \ \LRp{\overline{T}(\bphi)-\bphi}}=-{\overline{T}({\bf{0}})}^T\  Q({\bf{0}}) \ {\overline{T}({\bf{0}})},\]
where we used the property of Dirac measure in the last line. Now using \eqref{transp}, the expression for the second variation above can be further simplified as:
\begin{equation}
    \mathbf{D}^2J(\mu_{\bf{0}})[\nu_{\epsilon\bphi}]=-\epsilon^2\bphi^T  Q({\bf{0}})\bphi,\quad \mathrm{where}\ \norm{\bphi}_2\leq 1.
    \label{second_var_simplif}
    \end{equation}
Proceeding similarly for computing the third-order term in \eqref{variation_taylor}, it can be shown after some algebraic manipulations that \cite{wu2019splitting}:
\begin{equation}
{\LRs{\frac{d^3}{d\eta^3}\sJ(\mu_\eta)}_{\eta=\zeta}}= \mathcal{O} \LRp{\underset{\bphi \sim \mu_{\bf{0}} }{\expect}\LRs{\norm{\bphi-\overline{T}(\bphi)}_2^3}}=\mathcal{O} \LRp{\norm{{\bf{0}}-\overline{T}({\bf{0}})}_2^3} = \mathcal{O} \LRp{\epsilon^3},
\label{c_10}
\end{equation}
where we have used  the assumptions in  \cref{exist_th} while simplifying and used the form of transport map \eqref{transp}.
Substituting results \eqref{last_line}, \eqref{second_var_simplif}, \eqref{c_10} in \eqref{variation_taylor} we have:
\[  \sJ(\nu_{\epsilon\bphi})-\sJ(\mu_{\bf{0}})= -\frac{\epsilon^2\bphi^T  Q({\bf{0}})\bphi}{2}  +\frac{1}{6} \mathcal{O} \LRp{\epsilon^3}, \quad \mathrm{where}\ \norm{\bphi}_2\leq 1.\]
Therefore, the Topological derivative in $p-$Wasserstein space at $\mu_{\bf{0}}$ along the direction $\bphi$ is given by: 
   \[ d\sJ(\mu_{\bf{0}};\ \bphi)=-\lim_{\epsilon \downarrow 0}\frac{\sJ(\nu_{\epsilon\bphi})-\sJ(\mu_{\bf{0}})}{\epsilon^2}= \frac{\bphi^TQ(\bf{0})\bphi}{2},\quad \mathrm{where}\ \norm{\bphi}_2\leq 1.\]
\item Finally maximizing the Topological derivative \eqref{topo_der_w}  yields:
\begin{equation}
     \max_{\bphi} \LRp{d\sJ(\mu_{\bf{0}};\ \bphi)}= \max_{\bphi}\LRp{\frac{\bphi^TQ(\bf{0})\bphi}{2}},\quad \mathrm{subject \ to}\ \norm{\bphi}_2\leq 1.
     \label{optim_was}
     \end{equation}
The optimal solution $\bphi^*$ to \eqref{optim_was} is clearly given by:
\[ \bphi^*=v_{max}({\bf{0}}),\]
where $v_{max}({\bf{0}})$ is the eigenvector of $Q({\bf{0}})$ corresponding to  the maximum eigenvalue. Therefore, from \eqref{transp} the optimal transport map satisfies:
\[  \overline{T}({\bf{0}})=\epsilon\bphi^*=\epsilon v_{max}({\bf{0}}),\]
thereby concluding the proof.
\end{inlineenum}

\end{proof}

\section{Numerical experiments}
\label{experim}
In this section, we numerically demonstrate the proposed approach using different types of architecture such as a) Radial basis function neural network; b) Fully connected neural network; and c) vision transformer (ViT). {In all our numerical experiments, we choose  a linear combination of $Swish$ and $tanh$ as our activation function; that is, we take $ \sF=\{ Swish,\ tanh\}$ in \cref{univer}}. Our supplementary file contains additional numerical examples where we considered both simulated and real-world data sets and also demonstrated the approach for a convolutional neural network architecture (CNN). General experimental settings for all the problems and descriptions of methods adopted for
comparison are detailed in \cref{hyper_parameter}. Note that in our numerical results, Proposed (I) refers to \cref{Algo_full} and proposed (II) refers to ``fully automated growing" algorithm mentioned in \cref{automated}.

\subsection{Proof of concept example: Radial basis function (RBF) neural network}

As a proof of concept of our proposed approach, we first consider the problem of learning a 1-dimensional function using a radial basis function (RBF) neural network. In particular, we consider the multi-layered RBF neural network \cite{craddock1996multi,bodyanskiy2020multilayer} with residual connections where $\bg_{t+1}(.;\ .)$ in \eqref{res_o} is given by:
\begin{equation}
   \bg_{i+1}(x_{s,i};\ \btheta_{i+1})=\btheta_{{i+1}}^{(3)}\times \exp\LRp{-\frac{1}{2}\LRp{\btheta_{{i+1}}^{(1)}x_{s,i}+\btheta_{{i+1}}^{(2)}-c}^2}-\btheta_{{i+1}}^{(3)}\times \exp\LRp{-\frac{1}{2}\LRp{c}^2},
    \label{resnet_RBF}
\end{equation}
where  $\btheta_i=\LRs{\theta_{i}^{(1)},\ \theta_{i}^{(2)},\ \theta_{i}^{(3)}}^T\in \real^3$ for $i=0,\dots T-1$, and $c$ is a non-zero parameter of the activation function.  Note that, the second term in \eqref{resnet_RBF} is introduced to satisfy condition \ref{two} of   \cref{prop_admissible} and therefore in this case we have a modified radial basis function. Further, note that \eqref{resnet_RBF} corresponds to using a single neuron in each hidden layer.

\subsubsection{Topological derivative for the modified RBF network \eqref{resnet_RBF}}

The topological derivative is computed based on solving the eigenvalue problem \eqref{eigen_matrixx} where the matrix $\bQ_l$ (equation \eqref{eigen_matrixx}) is given as:
\begin{equation}
\bQ_l= \frac{1}{2}\sum_{s=1}^S   \LRp{\begin{bmatrix}
  0  & 0&  c_1 \ p_{s,l}\ x_{s,l}  \\ 
  0 & 0 &     c_1 \ p_{s,l}\\ 
   c_1 \ p_{s,l}\ x_{s,l} &   c_1 \ p_{s,l} &      0
 \end{bmatrix}},
\end{equation}
where $c_1=c\ \exp\LRp{-\frac{1}{2}\LRp{c}^2}$ and $c$ is a parameter in \eqref{resnet_RBF}.
\subsubsection{Data generation and numerical results}

For generating the training data set, we set $T=15$ in \eqref{res_o} and sample a true set of parameters $\{\btheta_i\}_{i=1}^T$ from 
normal distribution $\sN\LRp{{\bf{0}},3 {\bf{I}}}$, where ${\bf{I}}$ is the $3\times 3$ identity matrix.
Further, we set $c=0.1$ in \eqref{resnet_RBF}.  The training data set $D=\{x_i,\ c_i\}_{i=1}^{5000}$ is then generated by drawing $x_i$ uniformly from $[-2,\ 2]$ and computing the corresponding labels as  $c_i=\Omega^*(x_i)$, where $\Omega^*$ denote the true map to be learnt. In addition, we consider an additional $500$ data points for generating the validation data set and $1000$ data points for the testing data set. The generated true function is shown in \cref{learnt_curve_for_different_cases} (rightmost figure).
\begin{figure}[h!]      

\hspace{-0.7 cm}  
  \begin{tabular}{c}

      \begin{tabular}{c}

          \centering
          \includegraphics[scale=0.38]{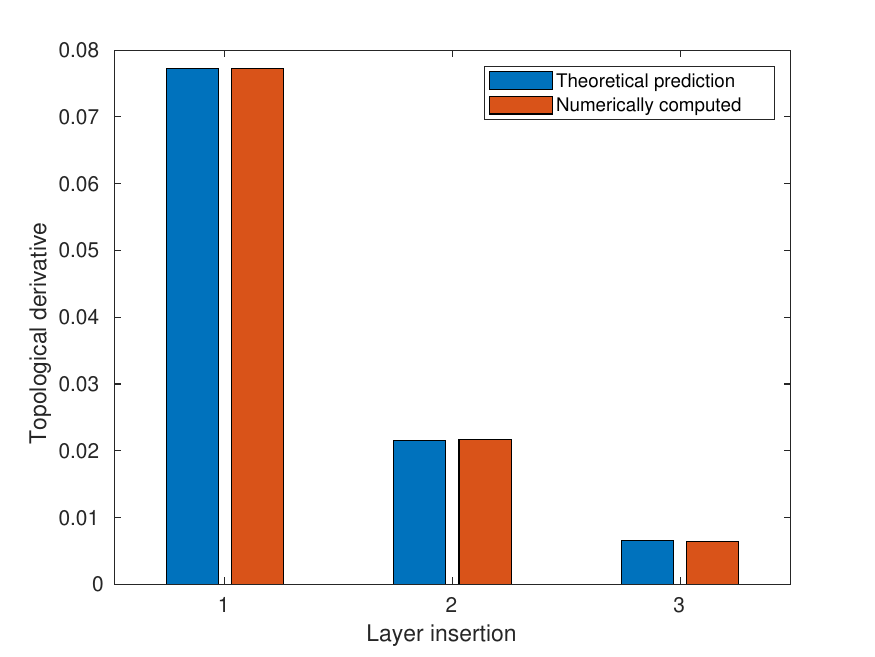}
          % \caption{Effect of noise on DI and Tikhonov solutions}
          % \figlab{some_other_good_name}

      \end{tabular}

\hspace{-1 cm}

      \begin{tabular}{c}

          \centering
          \includegraphics[scale=0.38]{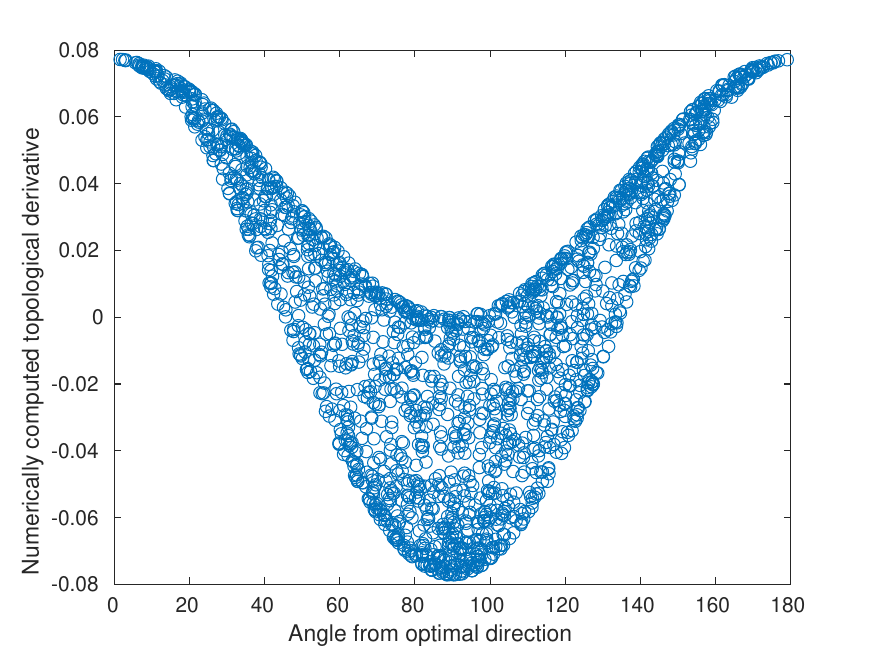}
          % \caption{Effect of noise on DI and Tikhonov solutions}
          % \figlab{some_good_name}

      \end{tabular}

\hspace{-1 cm}

 \begin{tabular}{c}
           \includegraphics[scale=0.38]{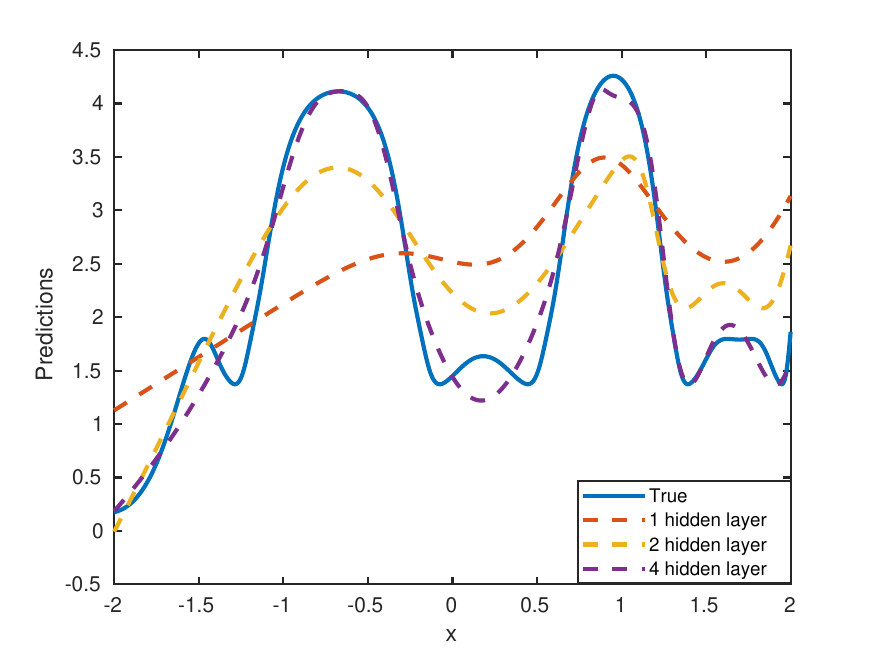}          % \caption{Effect of noise on DI and Tikhonov solutions}
          % \figlab{some_good_name}

      \end{tabular}
  
  \end{tabular}
  % \end{subfigure} \\
    \caption{Validating \cref{exist_th}. {\bf Left subfigure}: comparison of the theoretically computed topological derivative (equation  \eqref{topo_de}) with the numerically computed derivative for layer `l' with the largest eigenvalue and initialization  $\Phi_l$ given by \eqref{quant}; {\bf Middle subfigure}: effect of initialization $\bphi$ on the numerically computed derivative $ d\sJ(\Omega_0;\ (l,\  \boldsymbol{\phi},\ \sigma))$ in \eqref{topo_der} for $l=1$ and at the end of $1^{st}$ iteration; {\bf Right subfigure}: learned function using the proposed approach at different iterations of the algorithm. 
    }
  \label{learnt_curve_for_different_cases}
\end{figure} 

\subsubsection{Discussion of results}
As outlined in  \cref{deriv_algo}, we begin the adaptive training process by starting with a one-hidden layer network and progressively adding new layers. Each layer is trained to a (approximately) local minimum before adding a new layer. To validate our proposed theory, we compare the topological derivative predicted by our  \cref{exist_th}  with that computed numerically at each iteration of the algorithm and the results are shown in   \cref{learnt_curve_for_different_cases} (left). 
Note that the derivatives are computed for the layer $l$ with the largest eigenvalue at each iteration and initialization $\bphi=\Phi_l$ given by \eqref{quant}. We use a step size of $\epsilon=1\times 10^{-4}$ to compute  the numerical derivative $ d\sJ(\Omega_0;\ (l,\  \boldsymbol{\phi},\ \sigma))$ in \eqref{topo_der}. \cref{learnt_curve_for_different_cases} (left) shows that the numerical derivative is in close agreement with the theoretical derivative.  \cref{learnt_curve_for_different_cases} (middle) investigates how the choice of initialization $\bphi$ influences the derivative $ d\sJ(\Omega_0;\ (l,\ \boldsymbol{\phi},\ \sigma))$ in \eqref{topo_der}. For this, we set $l=1$ (at the end of training the initial one hidden layer network)  and generated different unit vectors (each sampled vector is represented as a blue circle in \cref{learnt_curve_for_different_cases} (middle)). The x-axis denotes the angle that each sampled vector makes with the optimal direction $\Phi_l$ given by \eqref{quant} and the y-axis denotes the numerically computed derivative for each sampled $\bphi$.  It is clear from \cref{learnt_curve_for_different_cases} (middle) that the maximum 
topological derivative is indeed observed for the optimal eigenvector $\Phi_l$ predicted by \eqref{quant}. \cref{learnt_curve_for_different_cases} (rightmost figure) shows the learned function (by our proposed approach) at each iteration of the algorithm.

\begin{figure}[h!]      
  \centering
  \begin{tabular}{c}
\begin{tabular}{c}
          \includegraphics[scale=0.4]{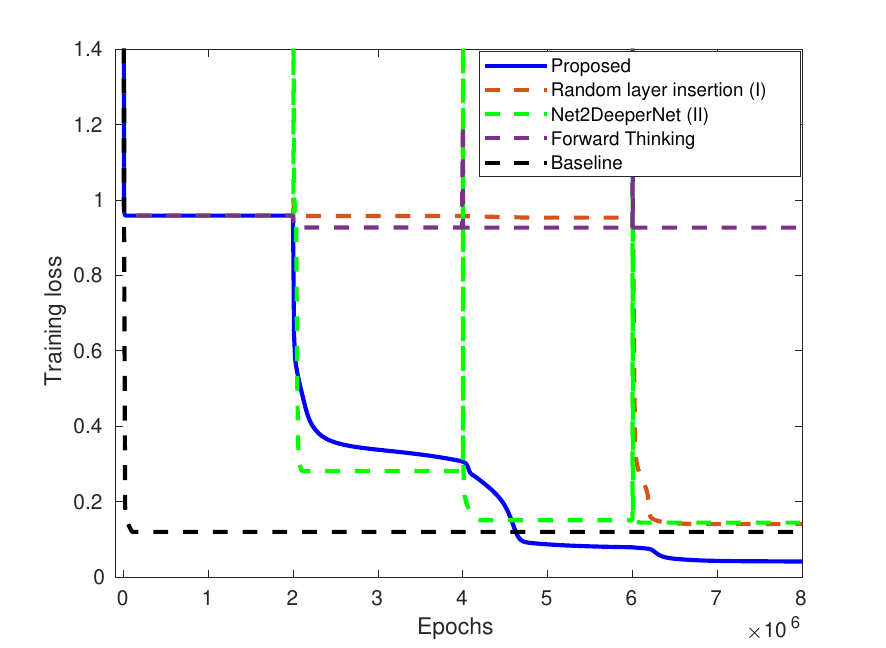}          % \caption{Effect of noise on DI and Tikhonov solutions}
          % \figlab{some_other_good_name}
 \end{tabular}

      \begin{tabular}{|c | c | c | c|}
        \hline
    Method &  Test loss  ($\mu \pm \sigma$)  \\ \hline
          {\bf{  Proposed }} & $0.051 \pm 0.027$\\ \hline
   Random layer insertion (I)   & $0.099 \pm 0.0514$ \\ \hline
   Net2DeeperNet (II) \cite{chen2015net2net}  & $0.069 \pm 0.0335$  \\ \hline
     Baseline  network  & $0.0897 \pm 0.038$ \\ \hline
Forward Thinking \cite{hettinger2017forward} &  $0.23 \pm 0.150$   \\ \hline
	\end{tabular} 
 \end{tabular}
  % \end{subfigure} \\
    \caption{{\bf Left subfigure}:  typical training loss curves for different approaches. {\bf Right subfigure}: summary of results.
    }
  \label{training}
\end{figure} 
Typical training curves for different adaptation strategies are shown in  \cref{training} (left). The sharp dip in the loss in  \cref{training} (left) corresponds to a decrease in loss on adding a new layer.  It is clear from  \cref{training} (left)
that our proposed framework outperforms all other adaptation strategies while also exhibiting superior performance in comparison to a baseline network trained from scratch. It is also interesting to note that while Net2DeeperNet (II) seems to exhibit superior performance in the beginning stages in \cref{training} (left), it is clear that there is no guarantee that  Net2DeeperNet (II)  will escape saddle points as evident from the later stages of adaptation. In contrast, our layer initialization strategy ensures a decrease in loss for sufficiently small $\epsilon$  thereby escaping saddle points.  The decision-making process at each iteration on where to add a new layer is shown in  \cref{rel_top_po_rbf} where the blocks represent the relative magnitude of the derivative for different layers. Note that a new layer is inserted at the location with maximum derivative.
\begin{figure}[h!]      
\centering
  
  \begin{tabular}{c}
\hspace{-1 cm}
      \begin{tabular}{c}

          \centering
            \includegraphics[scale=0.31,trim={0 0 0 5cm},clip]{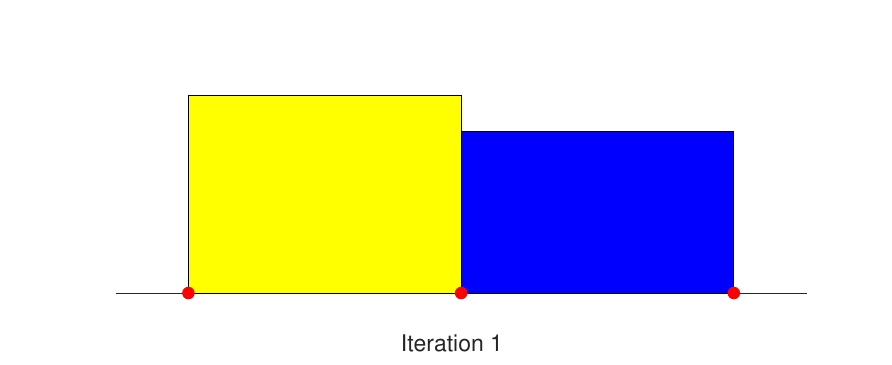}
        
          % \caption{Effect of noise on DI and Tikhonov solutions}
          % \figlab{some_other_good_name}

      \end{tabular}

\hspace{-1.5 cm}

      \begin{tabular}{c}

          \centering
          \includegraphics[scale=0.31,trim={0cm 0 0 5cm},clip]{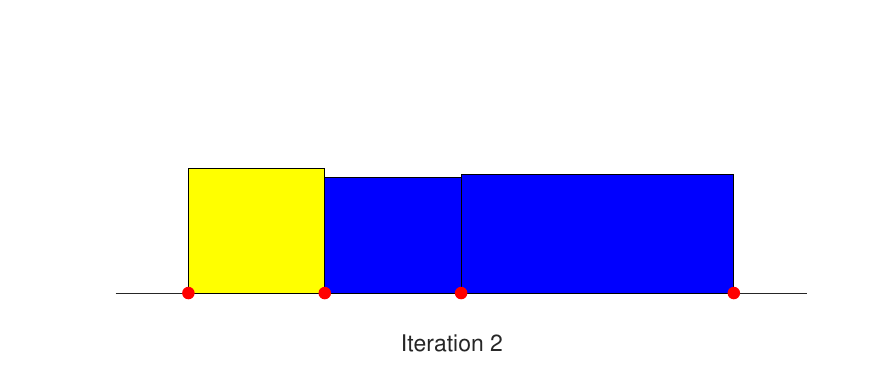}
      \end{tabular}
\hspace{-1.5 cm}

 \begin{tabular}{c}

          \centering
            \includegraphics[scale=0.31,trim={0 0 0 5cm},clip]{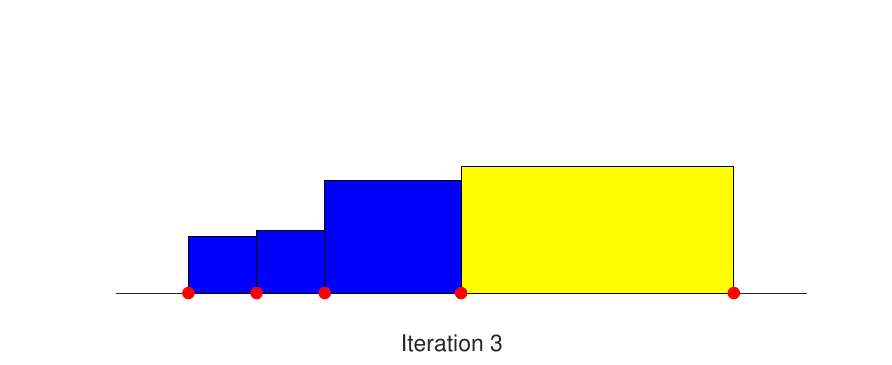}
        
          % \caption{Effect of noise on DI and Tikhonov solutions}
          % \figlab{some_other_good_name}

      \end{tabular}
\hspace{-1.5 cm}
         \begin{tabular}{c}

          \centering
          \includegraphics[scale=0.31,trim={0 0 0 5cm},clip]{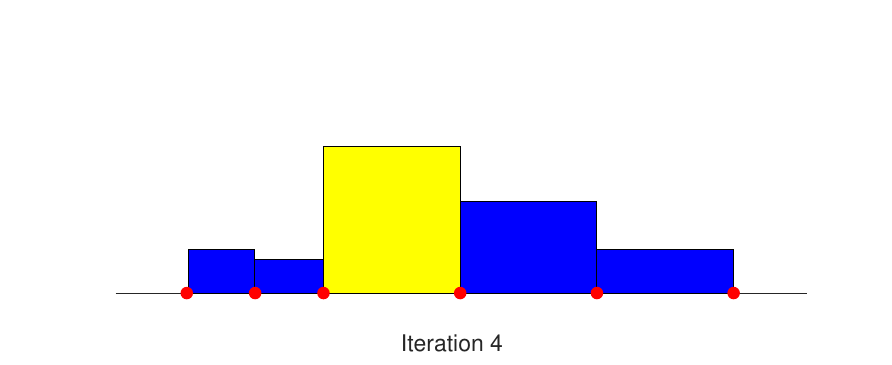}
      \end{tabular}

  \end{tabular}

  % \end{subfigure} \\
    \caption{Relative magnitude of the topological derivative for different hidden layers during the adaptation procedure. A new layer is inserted at the position of the maximal topological derivative (yellow block), and the red dots denote the location of layers.
    }
    \label{rel_top_po_rbf}
\end{figure} 
Further, the effectiveness of the approach is quantified by performing an uncertainty analysis where the algorithm is run considering different random initializations for the initial single hidden layer network and the results are tabulated in  \cref{training} (right).  \cref{training} (right) shows that on average our proposed method outperformed all other adaptation strategies while also exhibiting lower uncertainty. Note that in contrast to other approaches, the only source of uncertainty in our approach is the random initialization of initial network $\sQ_1$ in  \cref{Algo_full}. Further, the results in  \cref{training} (right) can be further improved (lower testing loss) if one considers an RBF network with multiple neurons in each hidden layer. {\em{It is also important to note that in practice it is not essential to train the network to a local minima at each iteration of the algorithm as demonstrated in \cref{training} (left) (note that the topological derivative in Theorem \eqref{exist_th} is valid even when the current network $\Omega_0$ is not in a local minima).  This is a significant advantage over some other methods that require the training loss to plateau before growing the network \cite{wu2019splitting,kilcher2018escaping}.}}
One may consider more efficient training approaches where the network is trained for a fixed user-defined number of epochs (scheduler) at each iteration as suggested in \cite{evci2022gradmax}. Our \cref{Algo_full} considers the use of a predefined scheduler to decide when a new layer needs to be added during the training phase, whereas  \cref{Algo_full_auto}
employs the use of a validation metric to automatically detect when a new layer needs to be added. 

In the following sections, we consider experiments with more complex datasets and assess the performance of both  \cref{Algo_full} and \cref{Algo_full_auto}.

\subsection{Experiments with fully connected neural network}

In this section, we consider experiments with fully connected neural networks (see  \cref{Algo_full}).

\subsubsection{Learning the observable to parameter map for 2D heat equation}
\label{poisson_sec}

In this section we consider solving the 2D heat conductivity inversion problem \cite{nguyen2024tnet}. 
The heat equation we consider is the following:
\begin{equation}
\begin{aligned}
     -\nabla \cdot \LRp{e^u \nabla y} & = 20  \quad \text{in } \Omega = \LRp{0,1}^2,\\
    y & = 0 \quad \text{ on } \Gamma^{\text{ext}}, \\
    \textbf{n} \cdot \LRp{e^u \nabla y} & = 0 \quad \text{ on } \Gamma^{\text{root}},
\end{aligned}
\label{heat_equation_o}
\end{equation}
where the conductivity $u$ is the parameter of interest (PoI), $y$ is the temperature field, and $\textbf{n}$ is the unit outward normal vector on Neumann boundary part $\Gamma^{\text{root}}$. The objective is to infer the parameter field $u(\bx)$ given the observation field $y(\bx)$. We construct the input vector for the neural network as $\by=[y(\bx_1),\ \dots y(\bx_{n_0})]$, where $\bx_i$ are fixed locations on $\Omega$. The domain $\Omega$ along with the fixed locations is provided in \cref{mesh_details}. For the present experiment, we set the input dimension $n_0=10$.

\vspace{0.1 cm}

\hspace{-0.6 cm}{\bf{\underline{Data generation and numerical results}}}

\vspace{0.1 cm}

Note that in order to generate the observation vector $\by$, one needs to assume a parameter field $u(\bx)$ and solve \eqref{heat_equation_o}. The  parameter field $u(\bx)$ is generated as follows:
\begin{equation}
   u(\mb{x}) = \sum_{i =1 }^{n_T} \sqrt{\lambda_i}\  \mb{\phi}_i(\bx)\  c_i,\quad \bx \in [0,\ 1]^2,
   \label{twopo}
\end{equation}
where $\LRp{\lambda_i, \ \mb{\phi}_i}$ is the eigenpair of an exponential two-point correlation function from \cite{constantine2016accelerating} and $\bc = \LRp{c_1,\hdots, c_{n_T}}\sim \sN({\bf{0}}, {\bf{I}})$ is a standard Gaussian random vector. Note that we consider $\bc$ as the network output with dimension $n_T = 12$. In addition, $5 \%$ additive Gaussian noise is added to the observations $\by$ to represent the actual field condition. Therefore, the training data points can be denoted as $ \{ \by_i,\ \bc_i\}_{i=1}^S$. Given a new observation data $\by$, the network outputs the vector $\bc$ which can then be used to reconstruct $u(\bx)$ using \eqref{twopo}.  To quantify the performance of our algorithm, we compute the average
relative errors on the test dataset  as follows:

\[ \text{Err}=\frac{1}{M}\sum_{i=1}^M \frac{\norm{\bu^{pred}_i-\bu^{true}_i}^2}{\norm{\bu^{true}_i}^2},\]
where $M$ denotes the number of test data samples, $\bu^{pred}_i$ denotes the neural network prediction for the parameter field $u(\bx)$ for the $i^{th}$ sample and $\bu^{true}_i$ denotes the synthetic ground truth
parameters. Here, $\bu$ is a vector containing the solutions on a $16\times 16$ grid shown in \cref{mesh_details} (right).

\begin{figure}[h!]      
  % \begin{subfigure}[b]{\textwidth}
\centering
  \begin{tabular}{c}

      \begin{tabular}{c}

          \centering
          \includegraphics[scale=0.23]{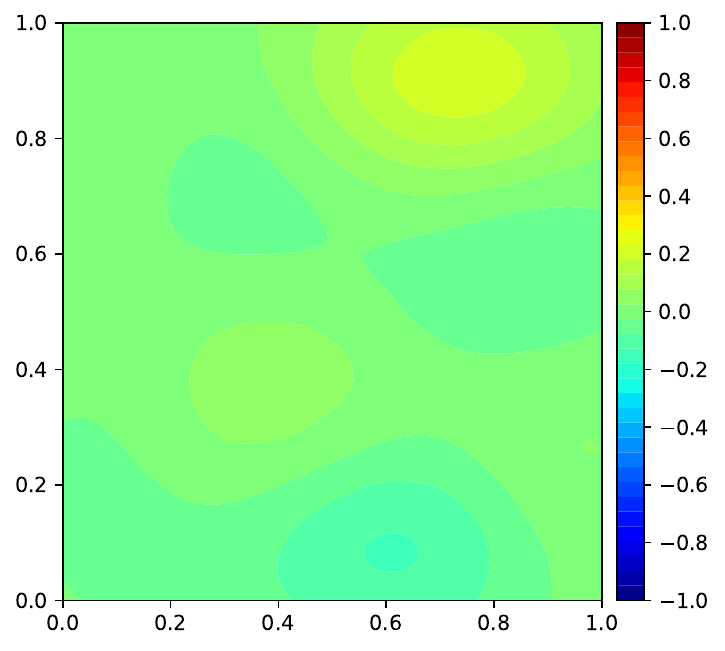}
          % \caption{Effect of noise on DI and Tikhonov solutions}
          % \figlab{some_other_good_name}

      \end{tabular}

  % \hspace{0.5 cm}

    %  \begin{tabular}{c}

      %    \centering
      %    \includegraphics[scale=0.3]{Figures/Poisson/proposed_1.pdf}
          % \caption{Effect of noise on DI and Tikhonov solutions}
          % \figlab{some_good_name}

     % \end{tabular}

 \hspace{-0.35 cm}

      \begin{tabular}{c}

          \centering
          \includegraphics[scale=0.23]{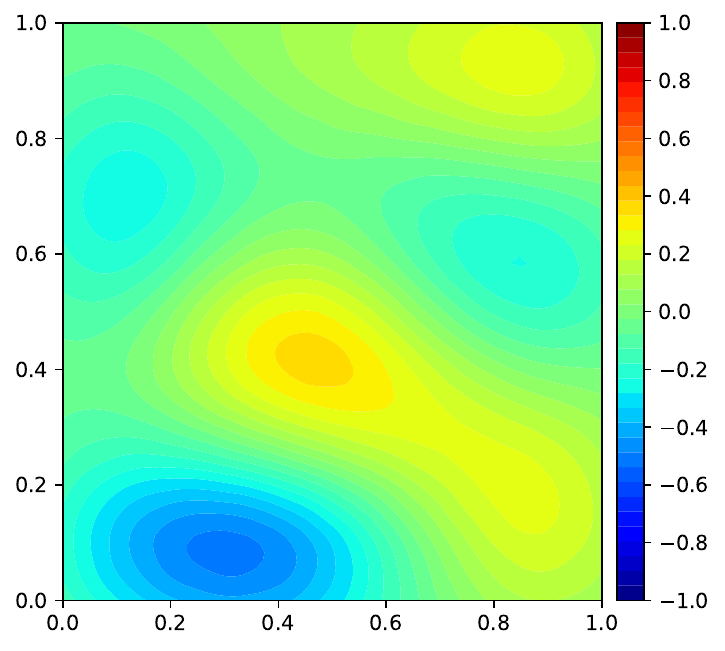}
          % \caption{Effect of noise on DI and Tikhonov solutions}
          % \figlab{some_good_name}

      \end{tabular}

 \hspace{-0.35 cm}
 
\begin{tabular}{c}

          \centering
          \includegraphics[scale=0.23]{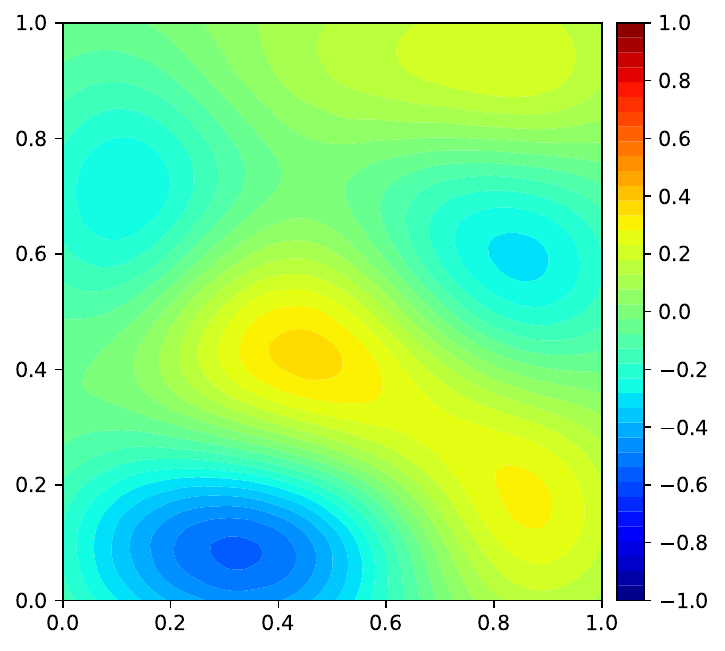}
          % \caption{Effect of noise on DI and Tikhonov solutions}
          % \figlab{some_other_good_name}

      \end{tabular}

 \hspace{-0.35 cm}
 
       \begin{tabular}{c}

          \centering
          \includegraphics[scale=0.23]{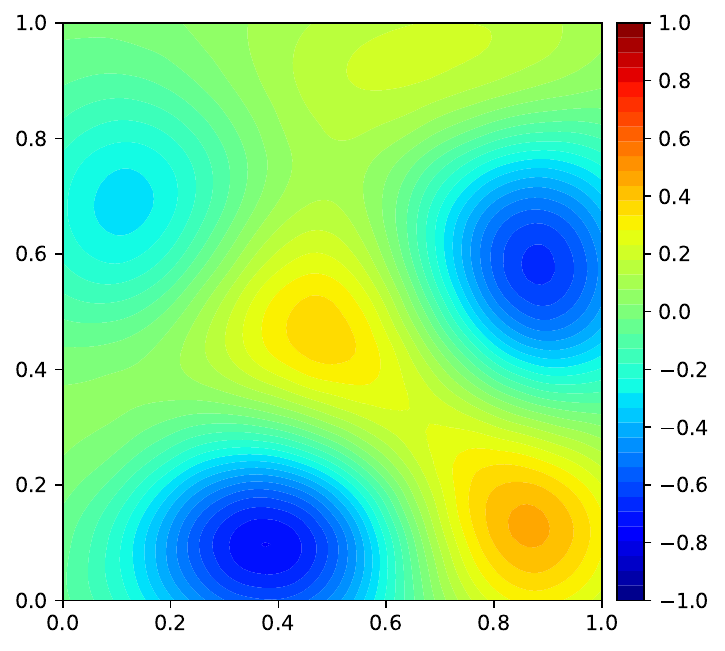}
          % \caption{Effect of noise on DI and Tikhonov solutions}
          % \figlab{some_good_name}

      \end{tabular}

 \hspace{-0.35 cm}
 
       \begin{tabular}{c}

          \centering
          \includegraphics[scale=0.23]{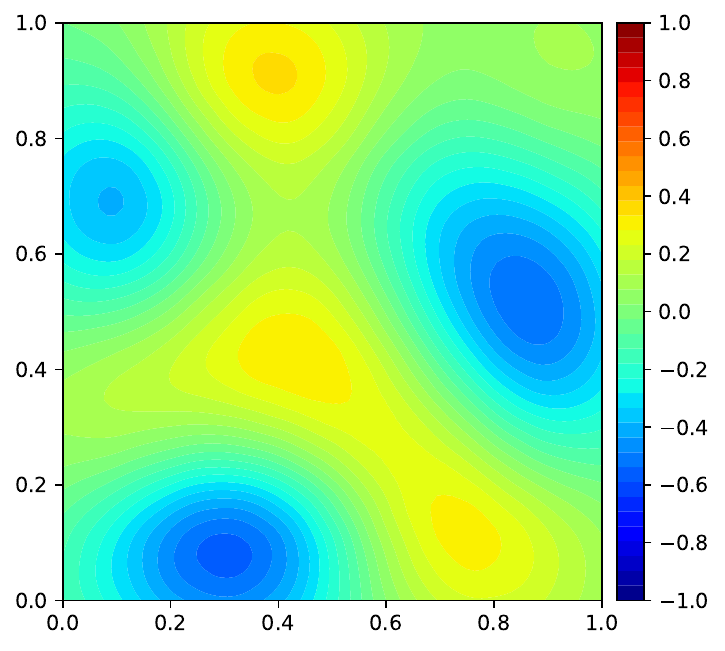}
          % \caption{Effect of noise on DI and Tikhonov solutions}
          % \figlab{some_good_name}

      \end{tabular}
      
  \end{tabular}\\

  \caption{Evolution of parameter field $u(\bx)$ for a particular test observation upon adding new layers for $S=1000$ (Left to Right): inverse solution after the $1^{st}$ iteration;  inverse solution after the $3^{rd}$ iteration; inverse solution after the $5^{th}$ iteration;  inverse solution after the $7^{th}$ iteration; and the groundtruth parameter distribution.}
\label{param_evolution}
\end{figure}     

Further, we consider experiments with training data set of size $S=1000$ and $S=1500$. We consider an additional $200$ data points for generating the validation data set and $500$ data points for the testing data set. Other details on the hyperparameter settings are provided in \cref{Input_values_different_problems}. Note that for each adaptation approach, the best network (in terms of least average relative error on the validation dataset) out of $100$ random initialization is retained for comparison purposes (such as for generating  \cref{param_evolution} and \cref{comparis_po}).

 \cref{param_evolution} shows the parameter field (for a particular observational data input) predicted by our proposed approach at different iterations of our algorithm. It is interesting to see that while the initial networks with a smaller number of parameters learn the low-frequency components of the solution, later iterations of the algorithm capture the more complex features in the solution (as more parameters are added). In other words, our algorithm identifies missing high-level (high-frequency) features and inserts new layers to capture them.  Further, the final parameter field is visibly close to the ground truth solution as evident from  \cref{param_evolution}.
Further, in an attempt to estimate the accuracy of each adaptation strategy (for a graphical representation of the error by each approach), we define the pointwise average relative error as:
\begin{equation}
    \text{Err}_j=\frac{1}{M}\sum_{i=1}^M \frac{\LRp{\bu^{pred}_{i,j}-\bu^{true}_{i,j}}^2}{\norm{\bu^{true}_i}^2/|\bu_i|},
    \label{error_metric}
\end{equation}
where subscript $j$ denotes the $j^{th}$ component of $\bu_i$ and $|\bu_i|$ denotes the number of elements in the vector. \cref{comparis_po} shows the error $\text{Err}_j$ plotted for  the main adaptation strategies. It is quite clear from   \cref{comparis_po} that our approach outperforms all other adaptation strategies by a good margin.

\begin{table}[h!]
\caption{Statistics ($\mu \pm \sigma$) of the relative error (rel. error) for different methods (2D heat equation)}
\centering
\resizebox{1\textwidth}{!}{
\begin{tabular}{|c | c | c | c|c|c|}
        \hline
    Method &  rel. error  & rel. error & Best rel. error&Training \\ 
     &  ($S=1000$) & ($S=1500$) & $S=1000 \ \vline  \ S=1500$& time for $S=1000$ \\ \hline
     {\bf{  Proposed (II)}}   & ${\bf{0.400}}  \pm 0.032$ & ${\bf{0.391}} \pm 0.033$& ${\bf{0.327}}$   \vline \ ${\bf{0.321}}$ & 11.2 min \\ \hline
          {\bf{  Proposed (I)}} & $0.434 \pm 0.022$ 
&$0.429 \pm 0.023$& 0.351  \vline \  0.364 & 15.1 min\\ \hline
   Random layer insertion (I)   & $0.457 \pm 0.031$ & $0.447 \pm 0.029$& 0.392   \vline \  0.360 & 13.7 min\\ \hline
   Net2DeeperNet (II) \cite{chen2015net2net}  & $0.456 \pm 0.027$ & $0.451 \pm 0.022$& 0.400   \vline \ 0.362   & 12.4 min\\ \hline
     Baseline  network  & $0.50 \pm 0.028$ & $0.489 \pm 0.029$&  0.446   \vline \  0.420 & 16.4 min\\ \hline
Forward Thinking \cite{hettinger2017forward} &  $0.66 \pm 0.033$ &   $0.65 \pm 0.034$&  0.570   \vline \  0.555 & 12 min\\ \hline
NAS (our activation) \cite{li2020system,li2020random}&  $-- \pm --$ &   $-- \pm --$&  0.349   \vline \  0.344 & 376 min\\ \hline
NAS (ReLU activation) \cite{li2020system,li2020random}&  $-- \pm --$ &   $-- \pm --$&  0.355   \vline \  0.347 & 380 min\\ \hline
NAS (tanh activation) \cite{li2020system,li2020random}&  $-- \pm --$ &   $-- \pm --$&  0.350   \vline \  0.342 & 370 min\\ \hline
	\end{tabular} }
 \label{stat_po}
\end{table}

\begin{comment}
\begin{table}[h!]
\caption{Best relative error achieved by different methods (2D heat equation)}
\centering
\begin{tabular}{|c | c | c | c|c|c|c|c|}
        \hline
&   {\bf{  P(I)}} &  {\bf{  P(II)}} & R(I) & R(II) & (B) & (H) & (S)\\ \hline
Best relative error (S=1000)&  $0.366$ &   $0.368$&   $0.392$ & $0.400$ & $0.446$& $0.57$& $0.383$\\ \hline
Best relative error (S=1500)&  $--$ &   $--$ &  $--$   & $--$  & $--$ & $--$ & $--$ \\ \hline
	\end{tabular} 
 \label{stat_po_max}
\end{table}
\end{comment}
\begin{figure}[h!]      
  % \begin{subfigure}[b]{\textwidth}
\centering
  \begin{tabular}{c}

      \begin{tabular}{c}

          \centering
          \includegraphics[scale=0.23]{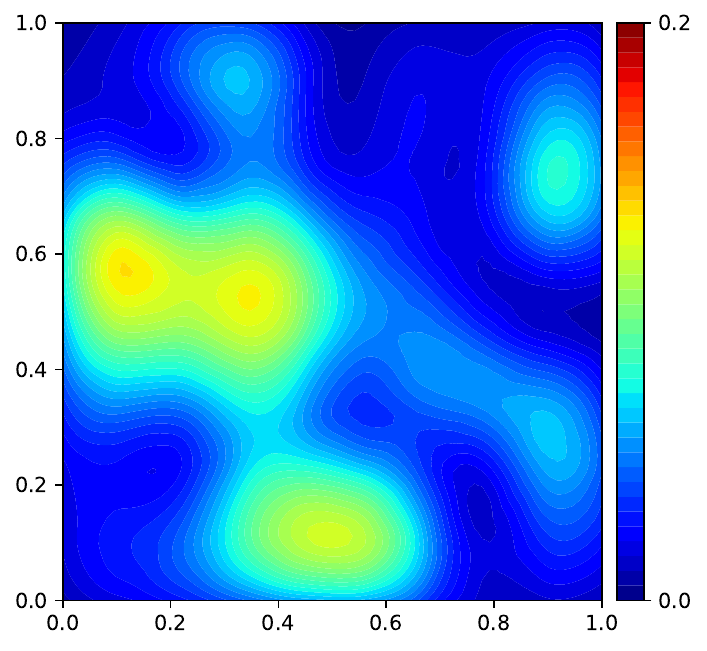}
          % \caption{Effect of noise on DI and Tikhonov solutions}
          % \figlab{some_other_good_name}

      \end{tabular}

   \hspace{-0.35 cm}

      \begin{tabular}{c}

          \centering
          \includegraphics[scale=0.23]{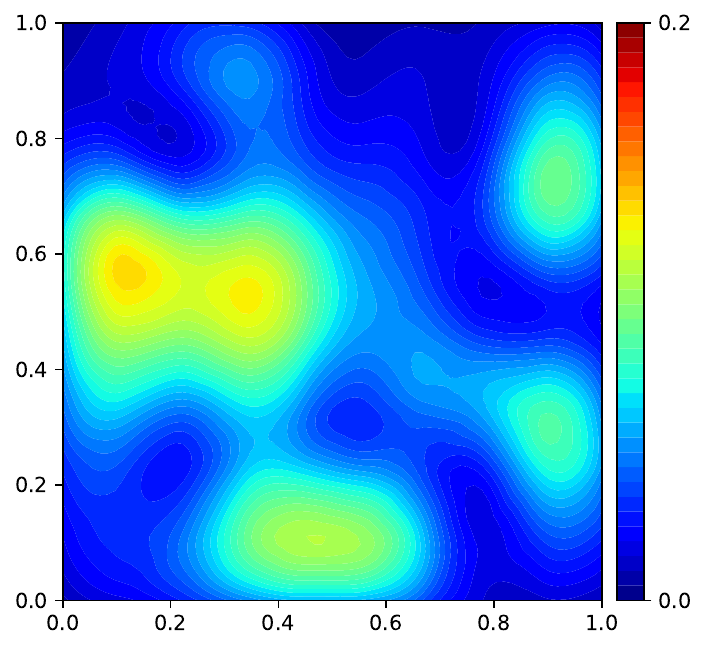}
          % \caption{Effect of noise on DI and Tikhonov solutions}
          % \figlab{some_good_name}

      \end{tabular}

\hspace{-0.35 cm}

      \begin{tabular}{c}

          \centering
          \includegraphics[scale=0.23]{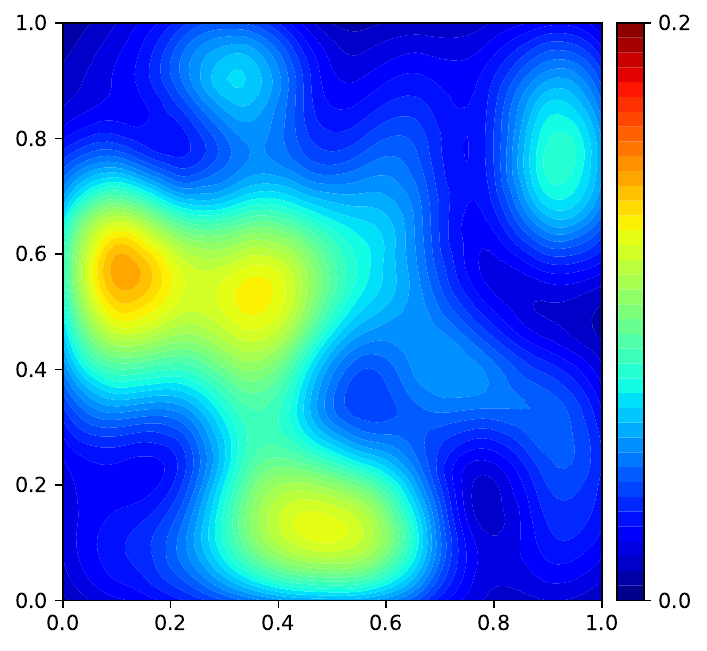}
          % \caption{Effect of noise on DI and Tikhonov solutions}
          % \figlab{some_good_name}

      \end{tabular}

 \hspace{-0.35 cm}
 
 \begin{tabular}{c}

          \centering
          \includegraphics[scale=0.23]{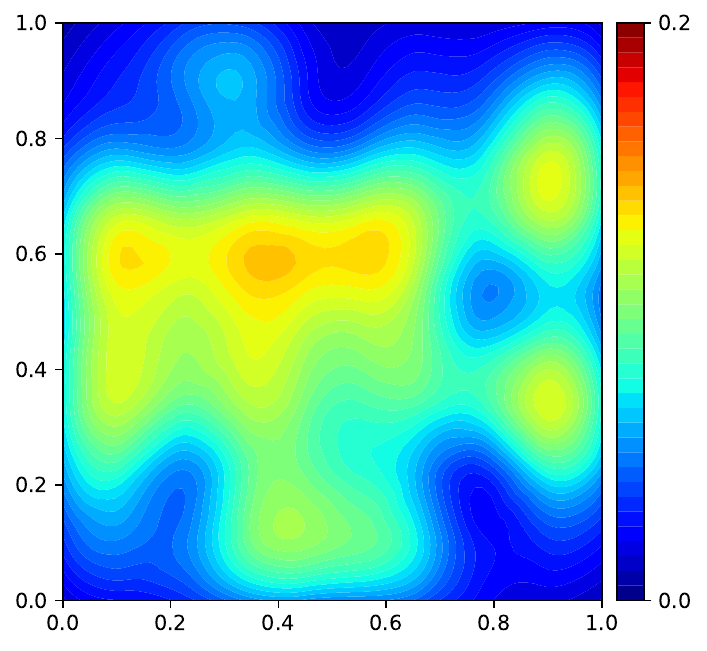}
          % \caption{Effect of noise on DI and Tikhonov solutions}
          % \figlab{some_other_good_name}

      \end{tabular}

 \hspace{-0.35 cm}
 
           \begin{tabular}{c}

          \centering
          \includegraphics[scale=0.23]{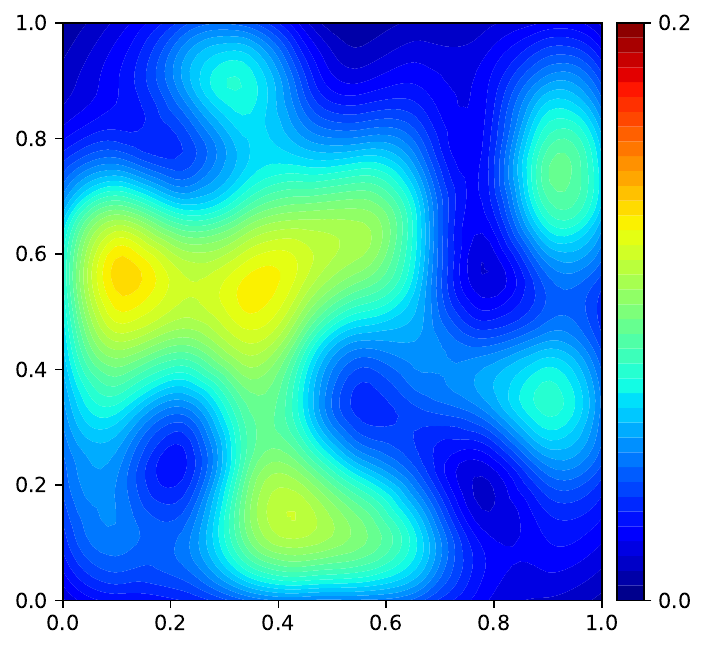}
          % \caption{Effect of noise on DI and Tikhonov solutions}
          % \figlab{some_good_name}

      \end{tabular}

  \end{tabular}

  \caption{Average error in the predicted parameter field over the spatial domain (for the full test data set,  $S=1000$). Left to right: Proposed method (II); Random layer insertion (I);  Net2DeeperNet (II) \cite{chen2015net2net}; Forward Thinking \cite{hettinger2017forward}; Baseline. }
  \label{comparis_po}
\end{figure}     
The uncertainty associated with each method is also investigated and the results are provided in \cref{stat_po} where it is clear that our proposed method outperformed all other strategies (details of other methods are provided in \cref{hyper_parameter}). Note that the relative error reported is with respect to the test dataset. Further, we have also conducted experiments with a larger data set ($S=1500$) and the results are shown in  \cref{stat_po}.  When considering the best relative error achieved, we observe from  \cref{stat_po} that the effect of our proposed strategy (I) over other methods decreases as the number of data points increases. However, the fully-automated network growing presented in  \cref{automated} (Proposed (II)) seems to outperform all the other methods both in the low and high training data regime. More theoretical analysis is necessary to characterize our algorithm's performance in low and high data regimes (see the discussion in  \cref{conclude} for more details). {We hypothesize that the data-dependent initialization of the added layers in our approach plays a key role in guiding the optimization toward good local minima, leading to better generalization. Note that the baseline network presented in \cref{stat_po}, although having the same number of layers as Proposed (I), uses randomly initialized parameters leading to poor generalization performance.}

{Further, a comparison of training times for different algorithms is provided in \cref{stat_po}. As shown, our fully automated algorithm, Proposed (II), achieves the best relative error in the shortest computational time. Proposed (I) requires more time because each layer is trained for a larger number of epochs (due to the use of a scheduler) compared to Proposed (II). The baseline network takes even longer because it trains a larger number of parameters from the start. While \cref{stat_po} shows that neural architecture search (NAS) can produce good results, it is the most computationally expensive approach among all the methods.}

{Since our algorithm and all methods in \cref{stat_po} uses a new activation function $\sigma(x)=\alpha_1 \ \mathrm{Swish(x)}+\alpha_2\ \mathrm{tanh(x)}$, where the coefficients $\alpha_1,\ \alpha_2$ are defined in \cref{activ_2}, it is necessary to examine what happens when popular activations such as  $\mathrm{tanh(x)}$ or $\mathrm{ReLU(x)}$ are used instead in the hidden layers. We conducted a neural architecture search (NAS) using various activation functions, and the results are shown in \cref{stat_po}. As illustrated in \cref{stat_po}, our activation function performs on par with, and in some cases slightly better than, ReLU and tanh.}

\begin{comment}
\begin{figure}[h!]      
\centering
  
  \begin{tabular}{c}
  
      \begin{tabular}{c}

          \centering
            \includegraphics[scale=0.4]{Figures/Poisson/adaptation_summary_poisson.pdf}
        
          % \caption{Effect of noise on DI and Tikhonov solutions}
          % \figlab{some_other_good_name}

      \end{tabular}

      \begin{tabular}{c}

          \centering
           \includegraphics[scale=0.4]{Figures/Poisson/bar_chart_poisson.pdf}
      \end{tabular}

  \end{tabular}
  % \end{subfigure} \\
    \caption{ Left to Right: Summary of the adaptation results: Bar chart shows the test loss achieved at each stage of the procedure; Error bar 
 with one standard deviation of uncertainty for different adaptation strategies: 100 different random initialization are considered to investigate the uncertainty in each method.}
  \label{adaptation_results_sum_poisson}
\end{figure} 
\end{comment}

%\textcolor{blue}{Error bar plot }

%error bar

%proposed: 0.445 $\pm$ 0.033

%baseline: 0.50 $\pm$ 0.028

%unit vec=0.456 $\pm$ 0.031

%random vec=0.456 $\pm$ 0.027

%hettinger= 0.66 $\pm$ 0.032

\subsubsection{Learning the observable to parameter map for 2D Navier-Stokes equation}
\label{nav_main}
In this section, we consider another inverse problem concerning the 2D Navier-Stokes equation for
viscous and incompressible fluid \cite{li2020fourier} written as:
\begin{equation}
\begin{aligned}
    \partial_t u(\bx,\ t)+{\bv(\bx,\ t)\cdot \nabla} u(\bx,\ t)&=\nu \Delta u(\bx,\ t)+f(\bx),\quad \bx \in \LRp{0,\ 1}^2,\ t\in (0,\ T] \\
    \nabla \cdot \bv(\bx,\ t)&=0,\quad \quad \quad \quad \quad \quad \quad \quad \ \bx \in \LRp{0,\ 1}^2,\ t\in (0,T]\\
    u(\bx,\ 0)&=u_0(\bx),\quad \quad \quad \quad \quad \quad \ \bx \in \LRp{0,\ 1}^2,
\end{aligned}
\label{navier_equation_o}
\end{equation}
where $\bv(\bx,\ t)$ is the velocity field, $u(\bx,\ t)$ is the vorticity, $u_0(\bx)$ is the initial vorticity, $f(\bx)=0.1\LRp{\sin{(2\pi(x_1+x_2))}+\cos{(2\pi (x_1+x_2))}}$ is the forcing function, and $\nu=10^{-3}$ is the viscosity coefficient. The spatial domain is discretized with $32 \times 32$ uniform mesh, while the
time horizon $t \in (0,10)$ is subdivided into $1000$ time steps with $\Delta t=10^{-2}$. Our objective is to reconstruct the initial vorticity $u_0(\bx)$ from the measurements of vorticity at $20$ observed points at the final time $T=10$ (denoted as $\by$), i.e, in this case, we have input dimension $n_0=20$.  Additional details on dataset generation is provided in \cref{nav_addi}.

We consider experiments with training data set of size $S=250$ and $S=500$. 
We consider an additional $50$ data points for validation data set and $200$ data points for testing data set. Other details on the hyperparameter settings are provided in \cref{Input_values_different_problems}. 
\begin{figure}[h!]      
  % \begin{subfigure}[b]{\textwidth}
\centering
  \begin{tabular}{c}

      \begin{tabular}{c}

          \centering
          \includegraphics[trim={0 0 2cm 0},clip,scale=0.42]{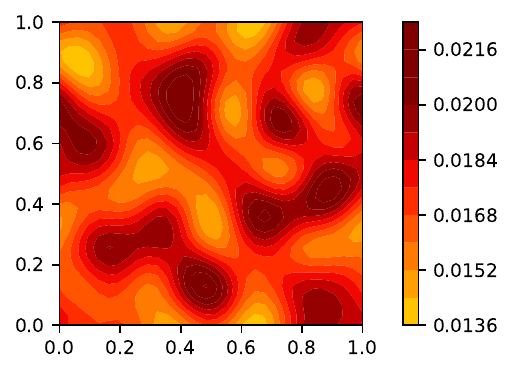}
          % \caption{Effect of noise on DI and Tikhonov solutions}
          % \figlab{some_other_good_name}

      \end{tabular}

  \hspace{-0.5 cm}

      \begin{tabular}{c}

          \centering
           \includegraphics[trim={0 0 2cm 0},clip,scale=0.42]{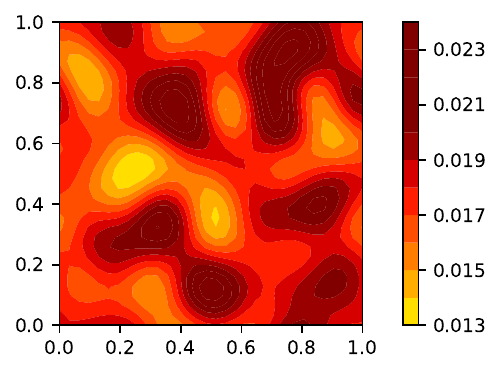}
          % \caption{Effect of noise on DI and Tikhonov solutions}
          % \figlab{some_good_name}

      \end{tabular}

\hspace{-0.5 cm}

      \begin{tabular}{c}

          \centering
         \includegraphics[trim={0 0 2cm 0},clip,scale=0.42]{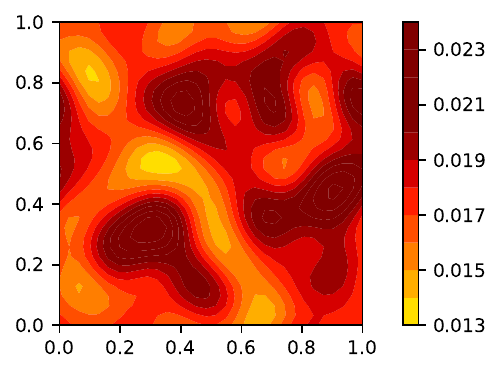}
          % \caption{Effect of noise on DI and Tikhonov solutions}
          % \figlab{some_good_name}

      \end{tabular}

\hspace{-0.5 cm}

 \begin{tabular}{c}

          \centering
        \includegraphics[trim={0 0 2cm 0},clip,scale=0.42]{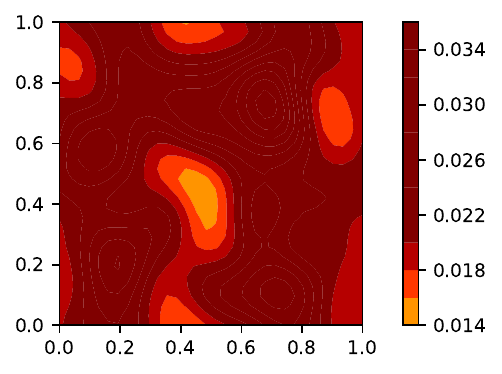}
          % \caption{Effect of noise on DI and Tikhonov solutions}
          % \figlab{some_other_good_name}

      \end{tabular}

\hspace{-0.5 cm}

  \begin{tabular}{c}

          \centering
          \includegraphics[scale=0.42]{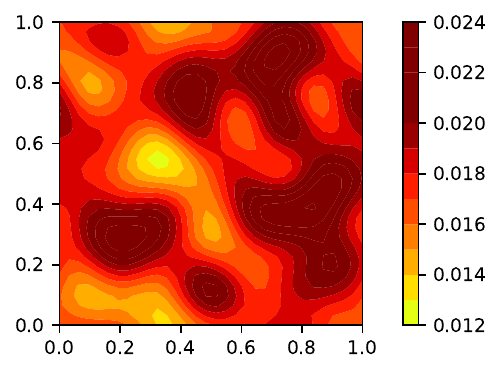}
          % \caption{Effect of noise on DI and Tikhonov solutions}
          % \figlab{some_good_name}

      \end{tabular}
      
  \end{tabular}
  \caption{Average error in the predicted parameter field over the spatial domain (for the full test data set, $S=250$). Left to right: Proposed method (II) (best rel. error: 0.295); Random layer insertion (I) (best rel. error: 0.299); Net2DeeperNet (II) \cite{chen2015net2net} (best rel. error: 0.309); Forward Thinking \cite{hettinger2017forward} (best rel. error: 0.362); Baseline (best rel. error: 0.305). }
  \label{rel_error_nav}
\end{figure} 
 \cref{rel_error_nav} shows the error $\text{Err}_j$ (equation \eqref{error_metric}) plotted for each approach (for S=250) where it is clear that our proposed approach provided the best results. Additional results on the performance (statistics of relative error) for each method is also provided in  \cref{stat_na}, where we observed 
that in the low data regime, when one considers the average relative error, our proposed approach outperformed all other strategies, performing on par with a neural architecture search algorithm. However, as the dataset size increased, we observed that random layer insertion (I) strategy performed equally well to our approach (see \cref{stat_na}). We hypothesize that our initialization strategy based on local sensitivity analysis matters more (in terms of generalization) in the low-data regime. Similar observations have been made in \cref{poisson_sec}.

\subsection{Topological derivative informed transfer learning approach}
\label{trans_lear}
Transfer learning is a machine learning technique in which knowledge gained through one task or dataset is used to improve the model's performance on a new dataset by fine-tuning the model on the new dataset. In this section, we demonstrate how the results in \cref{exist_th} can be used for adapting a pre-trained model in the context of transfer learning.
\subsubsection{Improving performance of pre-trained state-of-the-art machine learning models}
\label{vision_t}
 Using models pre-trained on large datasets and transferring to tasks with fewer data points is a commonly adopted strategy in machine learning literature \cite{dosovitskiy2020image}. In this section, we consider a tiny vision transformer model (ViT) pre-trained on ImageNet dataset (more details on the specific architecture employed is provided in \cref{vit_tran}). Our objective is to fine-tune this model to achieve the best performance on the CIFAR-10 dataset. The traditional approach
is to modify the output layer, i.e the MLP (multilayer perception) head of the ViT to have an output dimension of $10$ and retrain the whole network to achieve the best performance on CIFAR-10 dataset. Let's call the best-performing model  obtained this way as ``ViT baseline" (as denoted in \cref{ViT})
Our objective in this section is to use the topological derivative approach as a post-processing stage to improve the performance of ``ViT baseline" by further adapting the MLP Head using  \cref{Algo_full}.

\begin{table}[h!]
\caption{Accuracy achieved by different adaptation strategies}
\centering
\begin{tabular}{|c | c | c | c|c|c|c|}
        \hline
    ViT  &  Proposed (II) & Proposed (I)  & Random layer & Net2DeepNet \cite{chen2015net2net}& Forward  \\ 
    baseline  &  &  &insertion (I) & & Thinking \cite{hettinger2017forward} \\ \hline
         $90.9\%$ & $91.37 \%$ &  ${\bf{91.52}}\%$ &  $91.11 \pm 0.027$ & $91.11\pm 0.0137$ & $90.9 \pm 0.0$\\ \hline
	\end{tabular} 
 \label{ViT}
\end{table}
The best accuracy achieved by different adaptation strategies is shown in  \cref{ViT} and it is clear that our approach produces the best results. Note that for this task there is no source of randomness in our proposed approach and hence the statistics is not reported. However, other adaptation strategies has an inherent randomness due to random initialization of an added layer.
%In addition, Figure \ref{} also shows how the performance improves with the addition of each layer on the MLP head, where it is again clear how other strategies  fails to generalize better in comparison to our approach.  
Further, Forward Thinking \cite{hettinger2017forward} was unable to improve the performance over ViT baseline due to the freezing of previous layers at each adaptation step.  Other details of hyperparameters used for the problem are provided in \cref{hyper_parameter_n}. {Numerical experiments on fine-tuning internal encoder blocks, rather than the MLP head, are beyond the scope of the present work. In such a setting, our approach could be used to address the following questions: (i) which encoder block should be adapted; (ii) for the selected encoder block, where to add a new layer in the MLP; and (iii) how to initialize the added layer. State-of-the-art fine-tuning methods, such as LoRA \cite{hu2022lora}, while effective, do not address these questions in a principled manner. This type of multi-level adaptation framework is beyond the scope of the present work and will be explored in future research.}

\subsubsection{Application in parameter-efficient fine tuning}
\label{sci_tran}
%Transfer learning is a machine learning technique in which knowledge gained through one task or dataset is used to improve model's performance on a new dataset by fine-tuning the model on the new dataset. 
%However, new learning may interfere catastrophically  with old learning (catastrophic forgetting) when networks are trained sequentially \cite{kirkpatrick2017overcoming}. 
Parameter-efficient fine tuning is a transfer learning approach to address catastrophic forgetting \cite{kirkpatrick2017overcoming} by freezing many parameters/layers in the pre-trained network thereby preventing overfitting on the new data \cite{poth2023adapters}. %Parameter-efficient  transfer learning finds important application in scientific machine learning when it is often desirable to build machine learning surrogates combining sparse experimental data with huge simulation data \cite{olson2024transformer}. 
The procedure adopted in parameter-efficient transfer learning may be summarized as follows: a) Train a neural network from scratch on the large dataset  $\sD_s$; b) Fine tune a small part of the network (by freezing parameters in rest of the network) to fit the sparse data-set $\sD_e$. 
However, in this procedure, it is often unclear on which hidden layers of the network needs retraining to fit the new data $\sD_e$. Subel et al. \cite{subel2022explaining} pointed out that common wisdom guided transfer learning wherein one trains the last few layers is not necessarily the best option.  Often in literature one finds that the layers that need retraining are treated as a hyperparameter and a random search is conducted to arrive at the best decision \cite{olson2024transformer}.  

\begin{figure}[h!]      
  % \begin{subfigure}[b]{\textwidth}
 \centering
  \begin{tabular}{c}

      \begin{tabular}{c}

          \centering
          \includegraphics[scale=0.23]{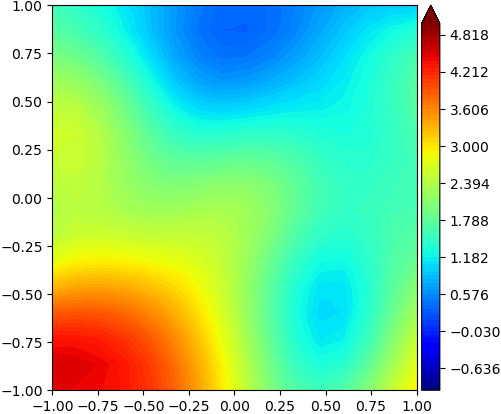}
          % \caption{Effect of noise on DI and Tikhonov solutions}
          % \figlab{some_other_good_name}

      \end{tabular}

      \begin{tabular}{c}

          \centering
          \includegraphics[scale=0.23]{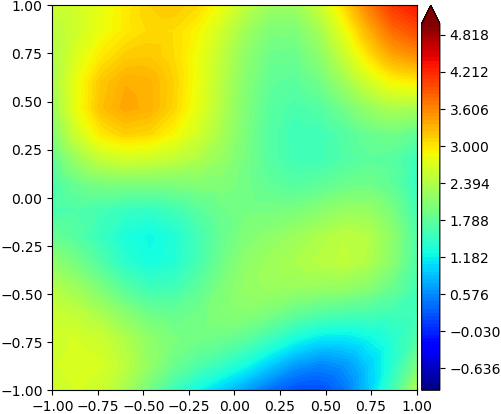}
          % \caption{Effect of noise on DI and Tikhonov solutions}
          % \figlab{some_good_name}

      \end{tabular}

      \begin{tabular}{c}

          \centering
          \includegraphics[scale=0.23]{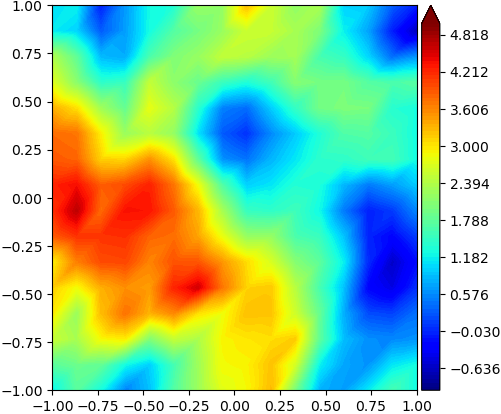}
          % \caption{Effect of noise on DI and Tikhonov solutions}
          % \figlab{some_good_name}

      \end{tabular}

            \begin{tabular}{c}

          \centering
          \includegraphics[scale=0.23]{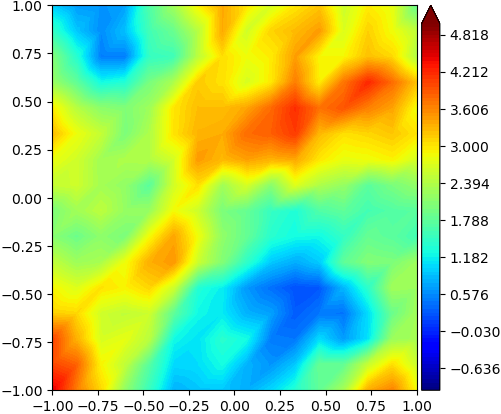}
          % \caption{Effect of noise on DI and Tikhonov solutions}
          % \figlab{some_good_name}

      \end{tabular}

  \end{tabular}
  \caption{Samples of conductivity field drawn from different distributions. Left to right: First two are samples from autocorrelation  prior; Last two are samples from a BiLaplacian prior.  }
\label{param_efficient_tr}
\end{figure} 
In this section, we demonstrate how the topological derivative in \cref{exist_th} could be used to inform which layers need retraining in parameter-efficient transfer learning. In particular, we seek to determine the most sensitive part of the network with respect to 
the new data set $\sD_e$. This is accomplished by defining the loss functional in \eqref{topo_der} based on the new data set $\sD_e$ and computing the topological derivative which informs where to add a new layer as described in  \cref{deriv_algo}. 
For demonstration, we revisit the heat equation \eqref{heat_equation_o} where the objective now is to learn a surrogate for the parameter to observable map (PtO). In particular, the network takes in the conductivity field $\bu\in \real^{256}$ as input and predicts the corresponding temperature $\by\in \real^{50}$ at fixed locations on the domain.
To train the initial network, the data-set $\sD_s$ consists of $10,000$ data points where the conductivity field $\bu$ is drawn from a BiLaplacian prior \cite{wittmer2023unifying} with mean ${\bf{\mu}}=2$ and covariance $\sC$ given in \eqref{priors} with parameters $(\gamma,\ \delta)=(0.1,\ 0.5)$. It is expected that with the passage of time the distribution of conductivity field $\bu$ changes due to changes in field conditions. To simulate this shift in distribution, the new data-set $\sD_e$ consists of $50$ samples where $\bu$ is drawn from a Gaussian autocorrelation smoothness prior \cite{kaipio2006statistical} with mean $\bf{\mu}=2$ and covariance  ${\bf{\Gamma}}$ given in \eqref{priors} with with $\sigma^2=2,\ \rho=0.5$.
\begin{equation}
   \sC=\LRp{\delta I+\gamma \nabla \cdot ( \nabla)}^{-2},\quad {\bf{\Gamma}}_{ij}=\sigma^2\exp \LRp{-\frac{\norm{\bx_i-\bx_j}_2^2}{2\rho^2}}.
   \label{priors}
\end{equation}
\begin{table}[h!]
\caption{Mean squared error achieved by different transfer learning approaches}
\centering
\begin{tabular}{|c | c | c | c|c|}
        \hline
    Complete &  Traditional transfer learning& Proposed  approach &  Exhaustive search \\ 
   retraining   &  &  & Best case \vline \ Worst case \\ \hline
         $3.572$ & $5.040$ &  ${\bf{2.576}}$ &  2.604 \vline \ \ \  3.240  \\ \hline
	\end{tabular} 
 \label{transfer}
\end{table}
Samples of generated conductivity fields from the two distributions are shown in  \cref{param_efficient_tr}.  Performance of different transfer learning approaches on a test data set of size $1000$ is shown in \cref{transfer} where it is clear that our proposed approach provides the best results. In  \cref{transfer}, ``Complete retraining" refers to the case when the entire pre-trained network is trained to fit the new data; ``Traditional transfer learning" refers to retraining only the last layer in the network; ``Proposed approach" refers to the topological derivative approach where a new layer is added and retrained along with the first and last layers; and ``Exhaustive search" refers to the case of retraining the last layer, first layer and a randomly chosen hidden layer. In addition, we also observed that while our proposed approach and traditional transfer learning takes approximately $3$ min to retrain the network, an exhaustive search procedure conducted serially took around $20$ min to find the best retrained network.
Additional justification (with numerical results) on why topological derivative serves as a good indicator for
determining where to add a new layer is provided in \cref{param_eff}.

\section{Conclusion}
\label{conclude}

In this work, we derived an expression for the  ``topological derivative" for a neural network and demonstrated how the concept can be used for progressively adapting a neural network along the depth. In \cref{optimal_trans}, we also showed that our layer insertion strategy can be viewed through the lens of an optimal transport problem. 
Numerically, we observed that our proposed method outperforms other adaptation strategies for most datasets considered in the study. {In particular, our method  (especially Proposed (I)) exhibits superior performance compared to other approaches in the low training data regime. Numerical results further suggest that employing a scheduler (Proposed I) is in general effective leading to superior performance in comparison to other strategies although occasionally our Proposed (II) (fully automated growing in \cref{additional_rem}) exhibited superior performance. Note that when the training dataset is very large, it is sufficient to ensure a good fit to the training data to achieve strong generalization. In this setting, the problem of architecture adaptation—namely, how to distribute parameters within a network—becomes less relevant. Instead, it is often sufficient to consider a very deep network and train it to fit the data well, as demonstrated by the success of large language models.}  %Motivated by this observation, to extract the best performing architecture  we recommend allotting half the computational resources for Proposed (I) strategy, and half for Proposed (II) strategy. 

Note that the topological derivative approach is a greedy approach where one chooses the best decision (locally) on where to add a layer along with the initialization at each step to reach at a globally optimal architecture. Here we define globally optimal architecture as the one that gives the lowest generalization error out of all possible architecture sampled from a predefined search space (see  \cref{hyper_parameter} for the search space we consider) and considering all random initializations from the normal distribution $\sN\LRp{{\bf{0}},\sigma_n^2 {\bf{I}}}$. In spite of our approach being greedy, we observe that our approach achieves comparable performance to that of a neural architecture search algorithm \cite{li2020random} for the datasets considered in this study. Deriving bounds on  the distance between the best architecture from our algorithm and the globally optimal architecture is beyond the scope of present work and will be investigated in the future (previous works have developed distance metric in the space of neural network
architectures \cite{kandasamy2018neural}). We anticipate that this analysis would provide insight on why our method (especially Proposed (I)) work particularly well in the low training data regime in comparison to other methods. 

%\krish{The conclusions should also summarize your theoretical findings (including connection with optimal transport) and numerical findings, and how your approach is better than others in which regimes. When your methods are not good, and point to ongoing research to address these issues (or point out why such issues are intrinsic to your methods)

%\krish{other hyperparameters need tuning}

\section*{Acknowledgment}
The authors would like to thank  Hai Van Nguyen for fruitful
discussions and helping to generate some of the data sets for computational experiments. The authors also acknowledge the Texas Advanced Computing Center (TACC)
at The University of Texas at Austin for providing HPC, visualization, database, or grid
resources that have contributed to some of the results reported within this paper. \href{http://www.overleaf.com}{URL:}  \url{http://www.tacc.utexas.edu.}

\bibliographystyle{siamplain}

%\bibliography{references}

\section*{SUPPLEMENTARY MATERIAL: TOPOLOGICAL DERIVATIVE APPROACH FOR DEEP NEURAL
NETWORK ARCHITECTURE ADAPTATION}

 \phantom{.}

\appendix

\begin{section}{Architecture adaptation for convolutional neural network (CNN)}
\label{CNN_application}
Note that though  \cref{opt_top_full} and \cref{full_topo} present the approach in the context of a fully connected network, the framework can also be applied to a CNN architecture where the input is a tensor of shape: (input height) $\times$ (input width) $\times$ (input channels). Note that in this case the function $\bff_{t+1}(\bx_{s,t};\ \Btheta_{t+1})$ in \eqref{training_problem} simply represents the 2D convolutional layer where $\bx_{s,t}$ represents the vectorized input and $\btheta_{t+1}$ represents the vectorized filter parameters. Note that $\bff_t(.;\ .)$ satisfies the assumptions in \cref{exist_th} when employing differentiable activation functions.

\begin{algorithm}
	\caption{CNN architecture adaptation algorithm}
	\hspace*{\algorithmicindent} \textbf{Input}: Training data $\bX$, labels $\bC$, validation data $\bX_1$, validation labels $\bC_1$, number of channels in each hidden layer $n$, loss function $\Phi$, filter size $f\times f \times n$, stride $s$, number of channels to activate in each hidden layer ${m}$, number of iterations $N_n$, parameter $\epsilon,\ \epsilon^t$, parameter $T_b$, hyperparameters and predefined scheduler for optimizer (\cref{hyper_parameter_n}).\\
	\hspace*{\algorithmicindent} \textbf{Initialize}:   Initialize  network $\sQ_1$ with $T_b$ hidden layers.\\
	\begin{algorithmic}[1] 
  \State Train network $\sQ_1$  and store the validation loss $(\epsilon_v)^{1}$.
  		\State set $i =1$,  $(\epsilon_v)^{0}>>(\epsilon_v)^{1}$, $\Lambda_l^m\geq \epsilon^t$
		\While{$i \le N_n$ \textbf{and} $\LRs{(\epsilon_v)^{i}\leq (\epsilon_v)^{i-1}}$\textbf{and} $\Lambda_l^m\geq \epsilon^t$} 
		%\State Train network $\sQ_i$ for $E_e$ epochs using standard optimizers.
		%\State Compute the topological derivative for each layer $l$ as $\Lambda_l=\sum_{j=1}^{a_m}\lambda_j^l$, where $\{\lambda_1^l\dots \lambda_{a_m}^l \}$ represents the top 
 % \hspace*{\algorithmicindent}   $a_m$ eigen values of $\sQ_l$ given by \eqref{eigen_matrix}. 
 \State Compute the topological derivative for each layer $l$  using \eqref{compute_full} and store as $\{ \Lambda_l\}$, also store $\Lambda_l^m=\max_l\{ \Lambda_l\}$.
   \State Store the corresponding eigenvectors for each layer as $\Phi_l$ representing $n$ filters of size $f\times f \times n$.
        \State Obtain the new network $\sQ_{i+1}$ by adding a new layer at position $l^*=\argmax\limits_l \{\Lambda_l\}$ with filter parameters  $\epsilon\Phi_{l^*}$.
        \State Perform a backtracking line search to update $\epsilon$ as outlined in \cref{back_track}.
        \State Update the parameters for optimizer if required (refer  \cref{hyper_parameter_n} for scheduler details).
%        \State Perform a backtracking algorithm for finding the best $\epsilon\in [0,\ 1]$ that gives the maximum decrease in training 
      %    \hspace*{\algorithmicindent}  loss.
        \State Train network $\sQ_{i+1}$  and store the best validation loss $(\epsilon_v)^{i+1}$ and the best network $\sQ_{i+1}$.
		\State $i = i+1$
		\EndWhile
	\end{algorithmic} \label{Algo_full_CNN}
\hspace*{\algorithmicindent} \textbf{Output}: Network $\sQ_{i-1}$
\end{algorithm}

In order to satisfy condition \ref{one} in  \cref{prop_admissible} (ResNet propagation), zero-padding is employed to preserve the original input size. Note that in this case $\Phi_l$ in \eqref{compute_full} denotes the filter parameters for the added layer (vectorized). Our algorithm for a CNN architecture is given in  \cref{Algo_full_CNN} and is self-explanatory.
Note that our proposed approach can be applied to any architecture as long as conditions in  \cref{prop_admissible} and \cref{exist_th} are satisfied.

\end{section}

\begin{section}{General setting for numerical experiments}{}
   \label{hyper_parameter}
All codes were written in PyTorch. Throughout the study, we have employed the Adam optimizer \cite{kingma2014adam} for minimizing the loss function.      Our proposed approach is compared with a number of different approaches as given below: 
\begin{equation*}
 \begin{aligned}
\text{Proposed  (I)}&: \text{Semi-automated architecture adaptation based on the }\\
& \text{\ \ \ topological  derivative approach as described in  }\\
& \text{\ \ \ \cref{Algo_full}, \cref{Algo_full_CNN}}.\\
\text{Proposed  (II)}&: \text{Automated architecture adaptation based on the }\\
& \text{\ \ \ topological derivative approach as described in }\\
& \text{\ \ \ \cref{automated} and \cref{Algo_full_auto}}.\\
\text{Random layer insertion (I)}&: \text{Adaptation strategy based on inserting a new layer at }\\
& \text{\ \ \ random position initialized with $\epsilon \Phi$ where $\Phi$ is a random }\\
& \text{\ \ \ unit vector}.\\
\text{ Net2DeeperNet  (II) }&: \text{Increasing depth of network based on  based on function }\\
& \text{\ \ \  preserving transformations. A layer is inserted at random  }\\
& \text{\ \ \ position with a small Gaussian noise added to the }\\
& \text{\ \ \ parameters to break symmetry \cite{chen2015net2net}}.\\
\text{ Baseline network (B)}&: \text{Training a randomly initialized network with the same }\\
& \text{\ \ \ final architecture as obtained by our proposed approach}.\\
\text{Forward Thinking (H)}&: \text{Algorithm for layerwise adaptation proposed by}\\
& \text{\ \ \  Hettinger et al. \cite{hettinger2017forward}}.\\
\text{Neural Architecture Search (NAS)}&: \text{Random search with early stopping proposed }\\
& \text{\ \ \  in \cite{li2020random,li2020system}}.\\
\end{aligned}  
\end{equation*}
We maintain the same activation functions and hyperparameters for all the adaptation strategies in order to make a fair comparison (except Proposed (II) which does not have a predefined scheduler (see  \cref{automated})). Note that the strategies Proposed (I), Random layer insertion (I), and Net2DeeperNet (II) differ only in the way a newly added layer is initialized and where a new layer is inserted. The Baseline network (B) is trained for the same total number of epochs as the proposed method (I). Note that, for all methods, the reported numerical results correspond to the model that achieves the best validation loss.  Further, the uncertainty in each approach is quantified (presented in the numerical results) by running the algorithm for different random initializations (100 in this case). For all problems we consider $\sF=\{ Swish,\ tanh\}$ (see  \cref{act_const}) for constructing the activation function $\sigma(x)$. Further, for neural architecture search \cite{li2020random,li2020system}, the search space is defined by considering architectures with a maximum of $N_n$ layers with each layer having a maximum of $n$ neurons (refer to \cref{hyper_parameter_n} for the values of $N_n$ and $n$ for each problem). We randomly sample a total of $50,000$ architectures in the beginning and follow the procedure in \cite{li2020random,li2020system} to arrive at the best architecture.
\end{section}

\begin{section}{Backtracking algorithm}
\label{back_tr}
The backtracking line search for choosing parameter $\epsilon$ is provided in \cref{back_track}. Note that lines 1-4 in \cref{back_track} try to find the $\epsilon$ that leads to a maximum decrease in loss $\sJ$. 
%However,  for an initially chosen $\epsilon$,  if $\sJ(\epsilon \Phi_l)> \sJ(\bf{0})$, then one need to decrease $\epsilon$ until the loss decreases. \krish{remove! }This is achieved in Lines 6-9 which shows the backtracking line search to find $\epsilon$ that fulfills the Armijo-Goldstein condition \cite{armijo1966minimization}. 
The values of $\tau_1$ and $\epsilon$ used is provided in   \cref{hyper_parameter_n}.
    \begin{algorithm} 
	\caption{Backtracking line search}
	\hspace*{\algorithmicindent} \textbf{Input}: Loss function $\sJ$, initialization of the added layer $\Phi_l$,  parameter $\tau_1$, $\epsilon$.\\
	%\hspace*{\algorithmicindent} \textbf{Initialize}:   Initialize  $\epsilon=\epsilon_0$.\\
	\begin{algorithmic}[1] 
   \While{$\sJ((\epsilon+\tau_1) \Phi_l)\leq \sJ(\epsilon \Phi_l)$} 
   \State $\epsilon=\epsilon+\tau_1$
   \EndWhile
	\end{algorithmic} \label{back_track}
\hspace*{\algorithmicindent} \textbf{Output}: Return $\epsilon$ as the solution.
\end{algorithm}

\end{section}

\begin{section}{Additional numerical results}
\label{additional_numerical_res}

\subsection{Learning the observable to parameter map for 2D heat equation}
\label{poisson_supply}

\cref{mesh_details} describes the domain $\Omega$ along with the fixed locations $\bx_i$ in \cref{poisson_sec}.

\begin{figure}[h!]      
    \centering
  \begin{tabular}{c}
  
      \begin{tabular}{c}

          \hspace{0.1 cm}
              \includegraphics[scale=0.18]{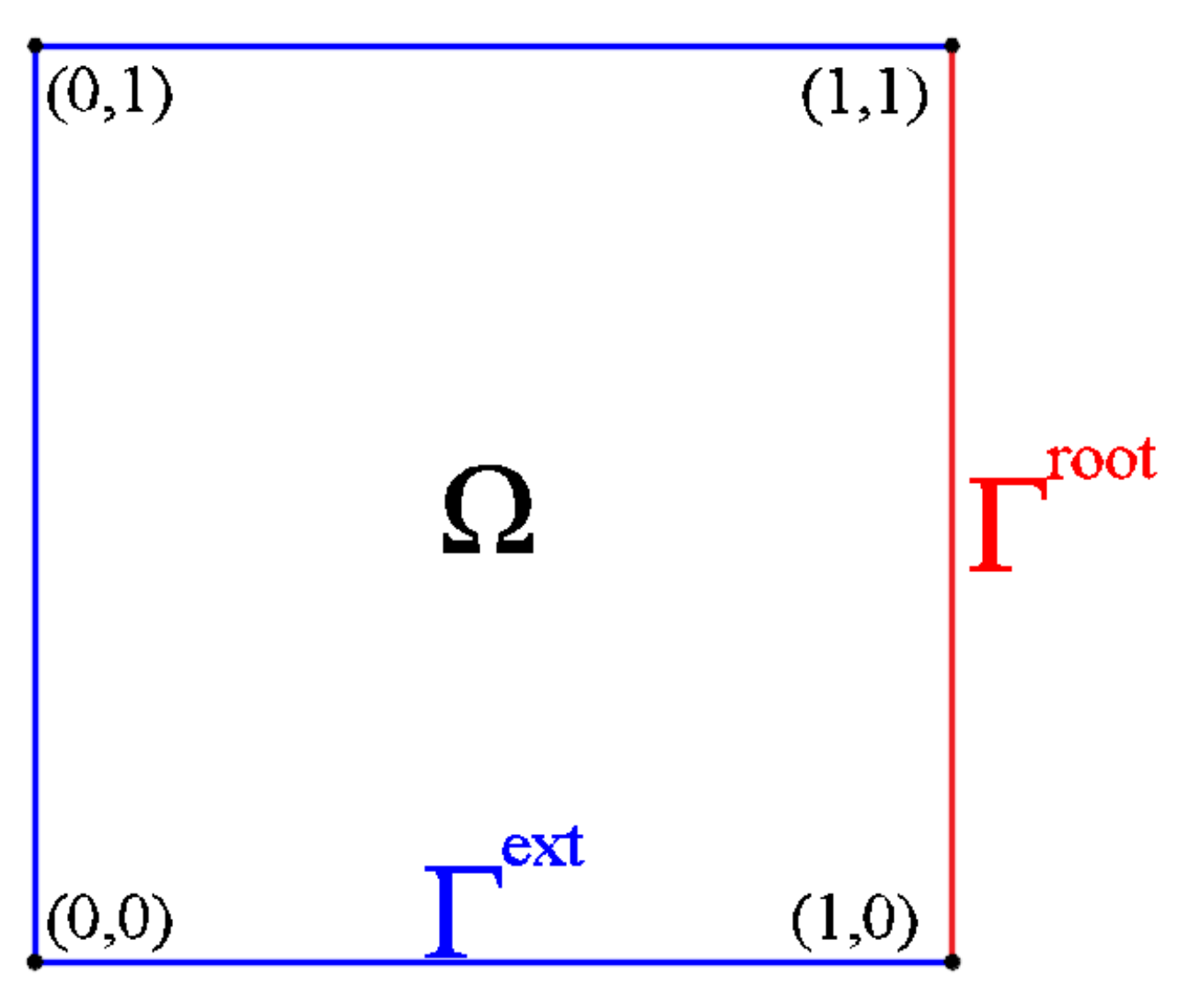}
        
          % \caption{Effect of noise on DI and Tikhonov solutions}
          % \figlab{some_other_good_name}

      \end{tabular}
 \begin{tabular}{c}

          \hspace{0.8 cm}
           \includegraphics[scale=0.7]{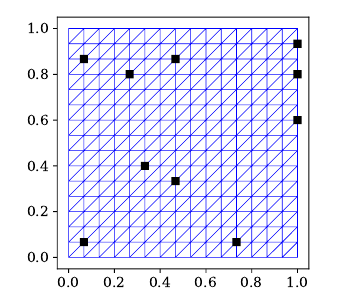}
         
      \end{tabular}

  \end{tabular}
  % \end{subfigure} \\
       \caption{ 2D heat conductivity inversion problem (Left to Right): The domain and the boundaries; A $16\times 16$  finite element mesh and $10$ observational locations.}
  \label{mesh_details}
\end{figure}

\subsection{Learning the observable to parameter map for 2D Navier-Stokes equation}
\label{nav_addi}

\paragraph{Data generation and numerical results}

For dataset generation in \cref{nav_main}, we draw
samples of $u(\bx,\ 0)$ based on the truncated Karhunen-Lo\`eve expansion as:
\begin{equation}
    u(x,\ 0)=\sum_{i=1}^{n_T}\sqrt{\lambda_i}\omega_i(x)c_i,
    \label{kl_nav}
\end{equation}
 where $\bc = \LRp{c_1,\hdots, c_{n_T}}\sim \sN({\bf{0}},I)$ is a standard Gaussian random vector, $\LRp{\lambda_i, \ \mb{\omega}_i}$  are eigenpairs obtained by the eigendecomposition of the covariance operator $7^{\frac{3}{2}}\LRp{-\Delta+49 {\bf{I}}}^{-2.5}$ with periodic boundary conditions. For demonstration, we choose $n_T=15$. For a given $u_0(\bx)$, we solve the Navier-Stokes equation by the stream-function formulation with a pseudospectral method \cite{li2020fourier} to compute $u(\bx,\ T)$ on the grid, which is then used to generate the observation vector $\by$. Similar to  \cref{poisson_sec} for a new observation data $\by$, the network outputs the vector $\bc$ which can then be used to reconstruct $u_0(\bx)$ using \eqref{kl_nav}.
\begin{figure}[h!]      
  % \begin{subfigure}[b]{\textwidth}
 \centering
 
  \begin{tabular}{c}

      \begin{tabular}{c}

          \centering
          \includegraphics[scale=0.5]{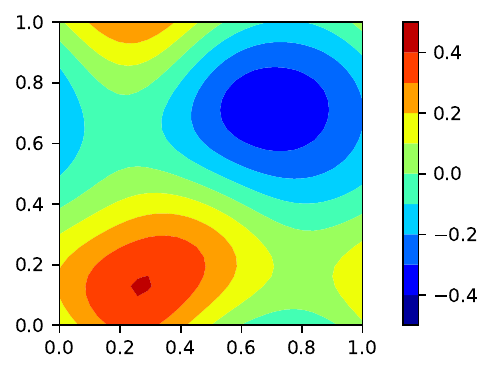}
          % \caption{Effect of noise on DI and Tikhonov solutions}
          % \figlab{some_other_good_name}

      \end{tabular}

   \hspace{0.2 cm}

      \begin{tabular}{c}

          \centering
          \includegraphics[scale=0.5]{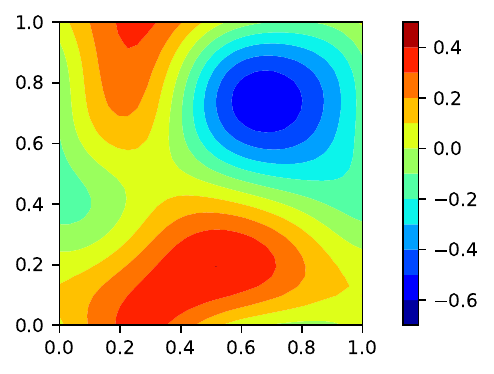}
          % \caption{Effect of noise on DI and Tikhonov solutions}
          % \figlab{some_good_name}

      \end{tabular}

\hspace{0.2 cm}

      \begin{tabular}{c}

          \centering
          \includegraphics[scale=0.5]{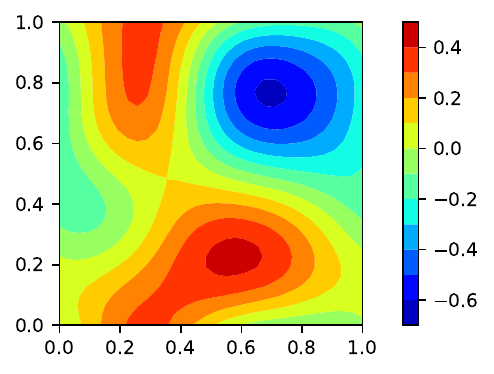}
          % \caption{Effect of noise on DI and Tikhonov solutions}
          % \figlab{some_good_name}

      \end{tabular}

  \end{tabular}\\

\vspace{-0.2 cm}

  \begin{tabular}{c}

      \begin{tabular}{c}

          \centering
          \includegraphics[scale=0.5]{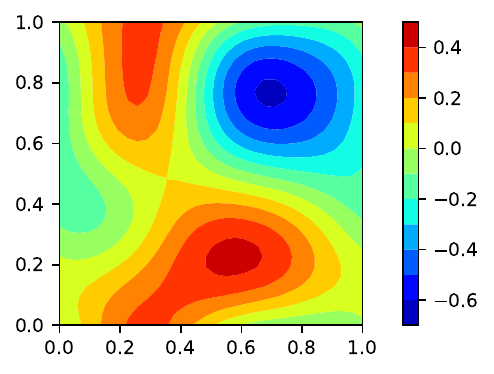}
          % \caption{Effect of noise on DI and Tikhonov solutions}
          % \figlab{some_other_good_name}

      \end{tabular}

   \hspace{0.2 cm}

      \begin{tabular}{c}

          \centering
          \includegraphics[scale=0.5]{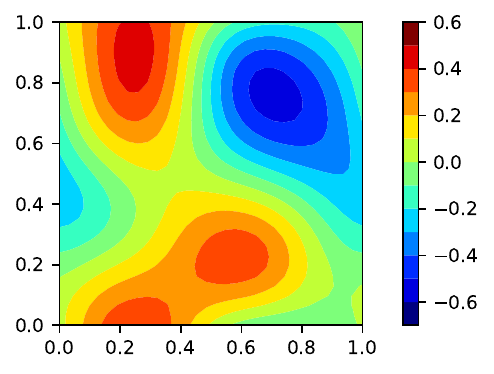}
          % \caption{Effect of noise on DI and Tikhonov solutions}
          % \figlab{some_good_name}

      \end{tabular}

\hspace{0.2 cm}

      \begin{tabular}{c}

          \centering
          \includegraphics[scale=0.5]{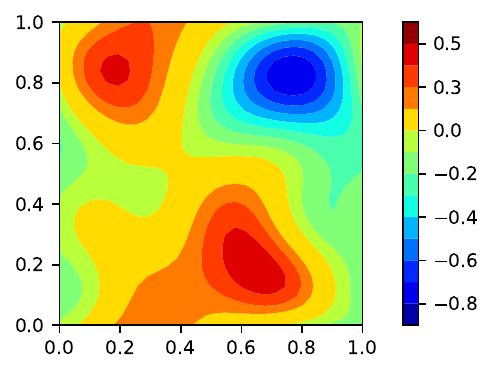}
          % \caption{Effect of noise on DI and Tikhonov solutions}
          % \figlab{some_good_name}

      \end{tabular}

  \end{tabular}
  \caption{Evolution of parameter field  $u(\bx,\ T)$ for a particular test observation upon adding new layers for $S=250$ (Left to Right): inverse solution after the $1^{st}$ iteration;  inverse solution after the $2^{nd}$ iteration;  inverse solution after the $3^{rd}$ iteration; inverse solution after the $4^{th}$ iteration; inverse solution after the $6^{th}$ iteration; and the groundtruth parameter distribution.}
  \label{param_evol_nav}
\end{figure}   
 \cref{param_evol_nav} shows the parameter field (for a particular observational data input) predicted by our proposed approach at
different iterations of our algorithm where we see that the predicted vorticity field improves as one adds more parameters (weights/biases). 

Further, the performance (statistics of relative error) for each method is also provided in  \cref{stat_na}. Results in \cref{stat_na} are explained in \cref{nav_main}.
%From  \cref{stat_na} it is clear that in the low data regime, when one considers the average relative error, our proposed approach outperformed all other strategies, performing on par with a neural architecture search algorithm. However, as the dataset size increases, we see that random layer insertion (I) strategy performed equally well to our approach. We hypothesize that our initialization strategy based on local sensitivity analysis matters more (in terms of generalization) in the low-data regime. Similar observations have been made in \cref{poisson_sec}.
\begin{table}[h!]
\caption{Statistics ($\mu \pm \sigma$) of the relative error (rel. error) different methods (Navier-Stokes equation)}
\centering
\begin{tabular}{|c | c | c | c|c|}
        \hline
    Method &  rel. error & rel. error & Best  error\\ 
     &  ($S=250$) & ($S=500$) & $S=250 \ \ \ \vline  \ S=500$\\ \hline
     {\bf{  Proposed (II)}}   &  $0.328 \pm 0.0227$ & $0.283\pm 0.0039$& ${\bf{0.295}}$  \vline \  0.274\\ \hline
          {\bf{  Proposed (I)}} & ${\bf{0.320}} \pm 0.0198$ 
&${\bf{0.277}} \pm 0.0034$& $0.301$ \vline \  $0.271$\\ \hline
   Random layer insertion (I)   & $0.326 \pm 0.0219$ & ${\bf{0.277}} \pm 0.0036$ &$0.299$  \vline \ ${\bf{0.267}}$\\ \hline
   Net2DeeperNet (II) \cite{chen2015net2net}  & $0.347 \pm 0.050$ & $0.285 \pm 0.006$& $0.309$   \vline \ $0.275$  \\ \hline
     Baseline  network  &$0.350 \pm 0.038$ & $0.280 \pm 0.0037$&  $0.305$   \vline \  $0.273$\\ \hline
Forward Thinking \cite{hettinger2017forward} &  $0.429 \pm 0.0286$ &   $0.313 \pm 0.017$& $0.362$   \vline \  $0.284$ \\ \hline
NAS \cite{li2020system,li2020random} &  $-- \pm --$ &   $-- \pm --$&  ${\bf{0.295}}$  \vline \  0.273\\ \hline
	\end{tabular} 
 \label{stat_na}
\end{table}

\subsection{Performance on a real-world regression data set: The California housing dataset}

In this section, we consider experiments with a real-world data set, i.e the California housing dataset (ML data set in scikit-learn \cite{scikit-learn}) where the task is to predict the housing prices given a set of $8$ input features, i.e we have $n_0=8$ and $n_T=1$ for this problem. We consider experiments with training data sets of sizes $S=500$ and $S=1000$,  and considered a validation data set of size $100$  and a testing data set of size $6192$. Other details on the
hyperparameter settings are provided in \cref{Input_values_different_problems}. \cref{stat_cal} provides the summary of the result (we use mean squared error (mse) to quantify the performance) where it is clear that our proposed method (I) outperformed other strategies. However, the fully-automated network growing \cref{Algo_full_auto} (i.e proposed (II)) does not seem to be very effective for this dataset in comparison to other adaptation strategies which could be due to overfitting in the early iterations of the algorithm.
%Figure \ref{rel_calif} and \ref{sum_cal} shows the summary of our adaptation strategy and the mean squared error achieved on the test dataset at the end of each iteration (the results are for the best network out of $100$ random initialization).

\begin{comment}
\begin{figure}[h!]      
\centering
  
  \begin{tabular}{c}

      \begin{tabular}{c}

          \centering
           \includegraphics[scale=0.4]{Figures/Navier_Stokes/adaptation_summary_navier.pdf}
        
          % \caption{Effect of noise on DI and Tikhonov solutions}
          % \figlab{some_other_good_name}

      \end{tabular}

      \begin{tabular}{c}

          \centering
         \includegraphics[scale=0.4]{Figures/Navier_Stokes/adaptation_summary_navier.pdf}
      \end{tabular}
  \end{tabular}
  % \end{subfigure} \\
       \caption{ Summary of the adaptation results ($S=1000$). (Left to right): Bar chart showing the mean squared error on the training dataset achieved at the end of each iteration; Bar chart showing mean squared error on test dataset achieved at the end of each iteration. \textcolor{blue}{ Change this figure }}
     \label{sum_cal}
\end{figure} 
\end{comment}

\begin{table}[h!]
\caption{Statistics ($\mu \pm \sigma$) of the mean squared error (MSE) for different methods (California housing dataset)}
\centering
\begin{tabular}{|c | c | c | c|c|}
        \hline
    Method & MSE  &MSE  & Best  error\\ 
     &  ($S=500$) & ($S=1000$) & $S=500 \ \ \ \vline  \ S=1000$\\ \hline
     {\bf{  Proposed (II)}}   &  $0.455 \pm 0.015$ & $0.407 \pm 0.020$& $0.398$  \vline \  0.350\\ \hline
          {\bf{  Proposed (I)}} & ${\bf{0.448}} \pm 0.020$ 
&${\bf{0.384}} \pm 0.0192$& ${\bf{0.392}}$ \vline \  ${\bf{0.346}}$\\ \hline
   Random layer insertion (I)   & $0.453 \pm 0.016$ & $0.386 \pm 0.0185$&$0.403$  \vline \ $0.350$\\ \hline
   Net2DeeperNet (II) \cite{chen2015net2net}  & $0.455 \pm 0.017$ & $0.391 \pm 0.0197$& $0.407$   \vline \ $0.346$  \\ \hline
     Baseline  network  & $0.481 \pm 0.032$ & $0.391 \pm 0.0207$&  $0.405$   \vline \  $0.351$\\ \hline
Forward Thinking \cite{hettinger2017forward} &  $0.541 \pm 0.028$ &   $0.430 \pm 0.033$& $0.440$   \vline \  $0.380$ \\ \hline
NAS \cite{li2020system,li2020random} &  $-- \pm --$ &   $-- \pm --$&  $0.395$  \vline \  0.347\\ \hline
	\end{tabular} 
 \label{stat_cal}
\end{table}

\subsection{Wind velocity reconstruction problem}
This example considers the wind velocity reconstruction problem where the objective is to predict the magnitude of wind velocity on a uniform 3D grid based on sparse measurement data. The observational dataset consists of $(x,y,z)$ position components and the corresponding velocity $\LRp{u(x,y,z),\ v(x,y,z),\ w(x,y,z)}$  over North America. We train a network that takes in the $\LRp{x,\ y,\ z}$ components as inputs and predicts the magnitude of wind velocity $\sqrt{u^2+v^2+z^2}$ as the output. We consider a training data set of sizes $S=1000$ and $S=5000$,   a validation data set of size $2000$  and a testing data set of size $21525$. The performance of each method is provided in  \cref{stat_wind}. It is clear that both our proposed method (I) and (II) outperformed all other strategies. Further, the improvement in solution (3D air current profile) on adding new layers is shown in  \cref{wind_evolution} where one clearly sees that the algorithm progressively picks up complex features in the solution as evident from more contour lines appearing in later stages of the algorithm. Further,  \cref{rel_error_wind} shows the relative error in predictions w.r.t truth (equation \eqref{error_metric}) over the 3D domain for different methods. It is clear from  \cref{rel_error_wind} that our proposed approach provides the most accurate estimates.

\begin{figure}[h!]      
  % \begin{subfigure}[b]{\textwidth}
\centering
  \begin{tabular}{c}

\hspace{-1.3 cm}
      \begin{tabular}{c}

          \centering
          \includegraphics[scale=0.4]{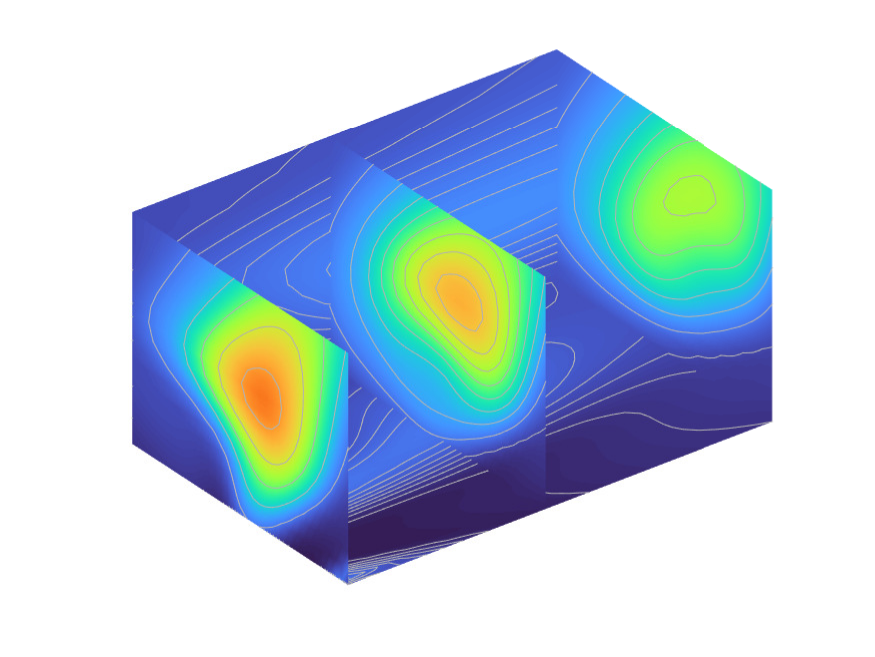}
          % \caption{Effect of noise on DI and Tikhonov solutions}
          % \figlab{some_other_good_name}

      \end{tabular}

 \hspace{-1.1 cm}

      \begin{tabular}{c}

          \centering
          \includegraphics[scale=0.4]{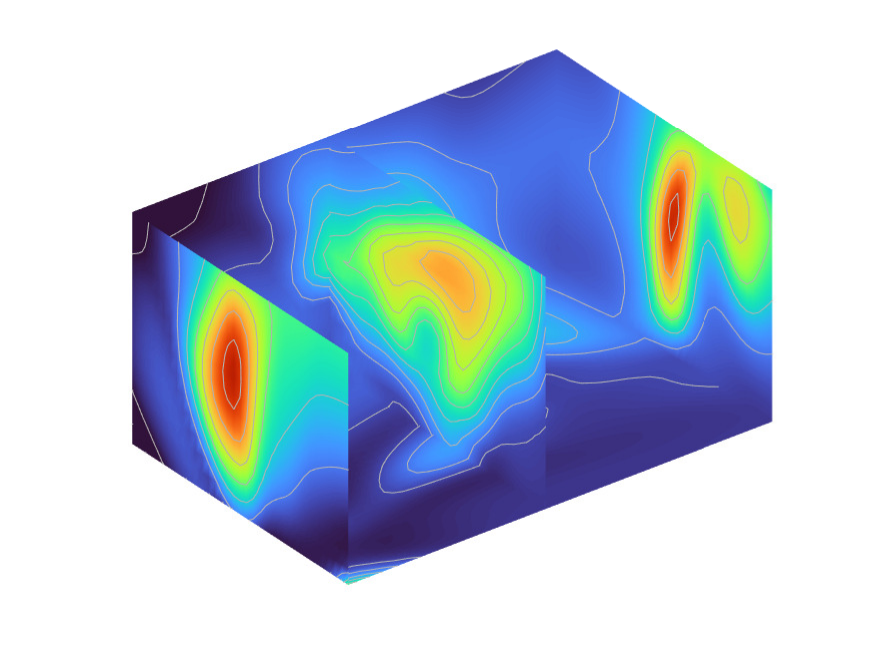}
          % \caption{Effect of noise on DI and Tikhonov solutions}
          % \figlab{some_good_name}

      \end{tabular}

\hspace{-1.1 cm}
      \begin{tabular}{c}

          \centering
          \includegraphics[scale=0.4]{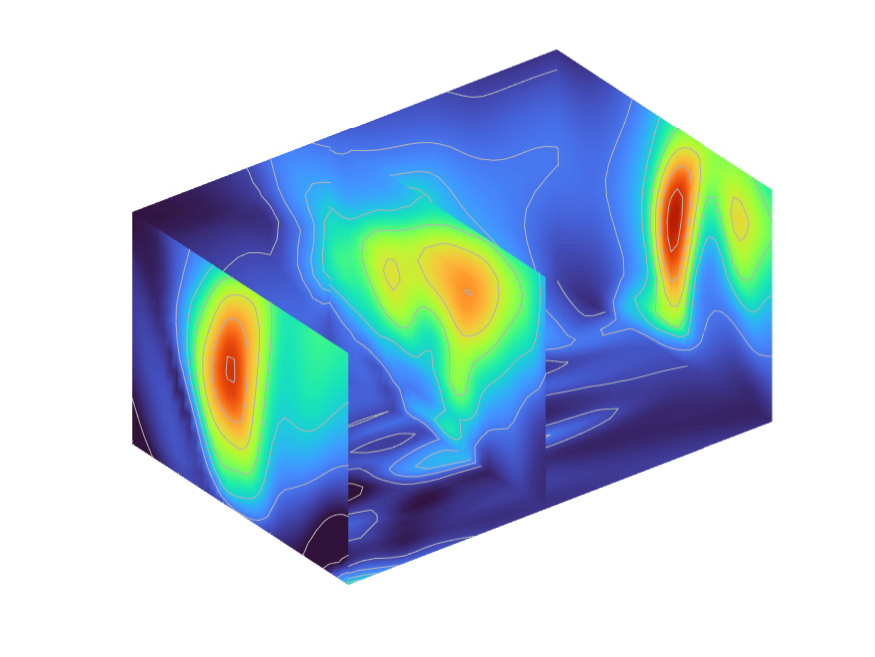}
          % \caption{Effect of noise on DI and Tikhonov solutions}
          % \figlab{some_good_name}

      \end{tabular}

  \end{tabular}\\

  \begin{tabular}{c}

\hspace{-1.3 cm}
      \begin{tabular}{c}

          \centering
          \includegraphics[scale=0.4]{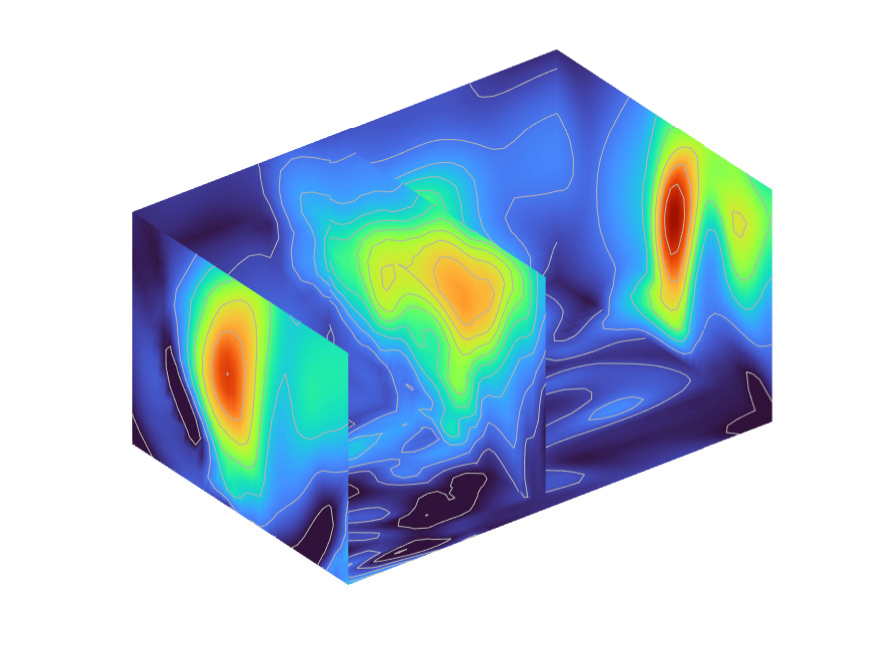}
          % \caption{Effect of noise on DI and Tikhonov solutions}
          % \figlab{some_other_good_name}

      \end{tabular}

\hspace{-1.1 cm}

      \begin{tabular}{c}

          \centering
          \includegraphics[scale=0.4]{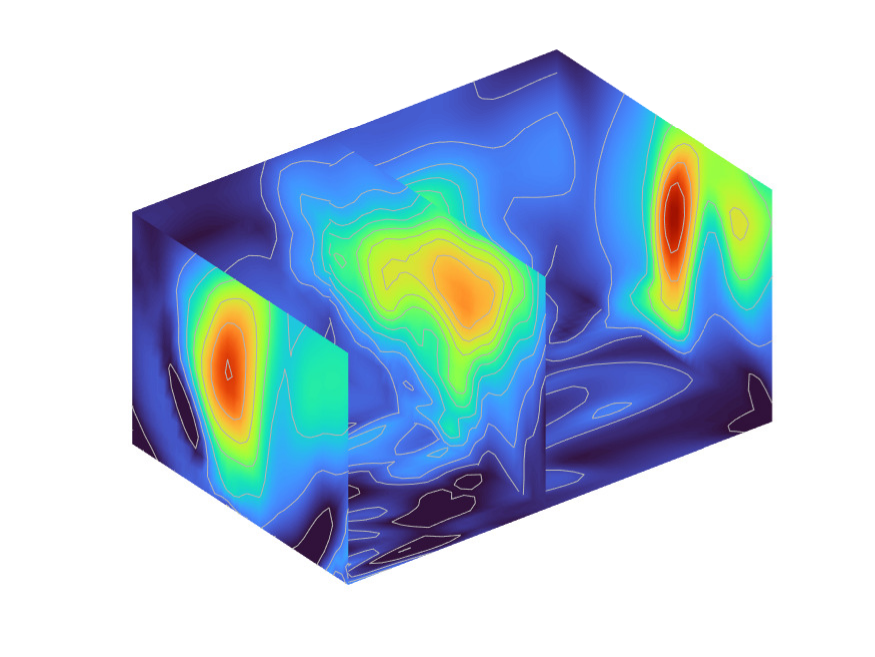}
          % \caption{Effect of noise on DI and Tikhonov solutions}
          % \figlab{some_good_name}

      \end{tabular}

\hspace{-1.1 cm}

      \begin{tabular}{c}

          \centering
          \includegraphics[scale=0.4]{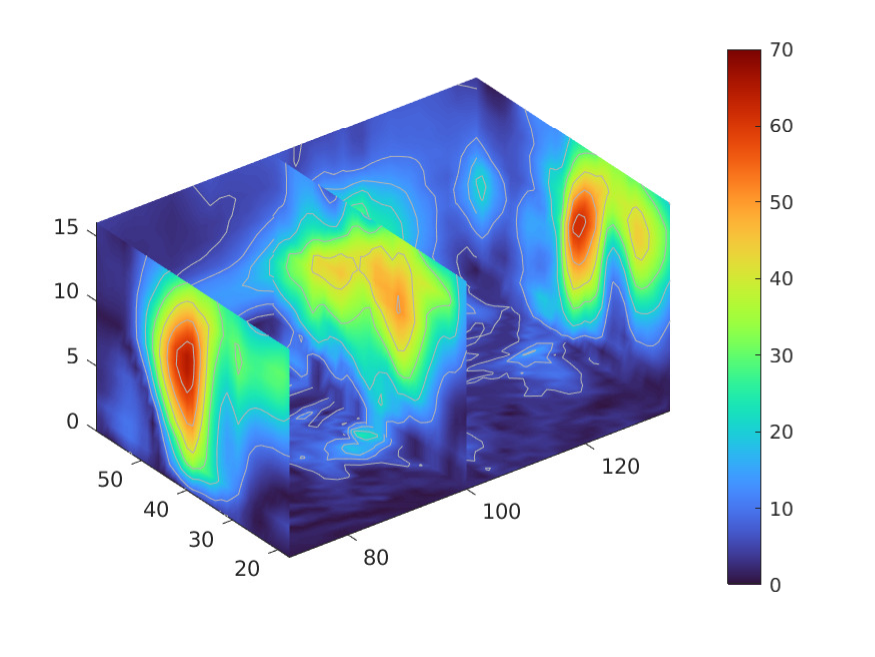}
          % \caption{Effect of noise on DI and Tikhonov solutions}
          % \figlab{some_good_name}

      \end{tabular}

  \end{tabular}
  \caption{Evolution of solution (3D air current profile) upon adding new hidden layers for $S=1000$ (Left to right): Solution after the $4^{th}$ iteration; Solution after the $5^{th}$ iteration; Solution after the $6^{th}$ iteration; Solution after the $8^{th}$ iteration;  Solution after the $10^{th}$ iteration; Ground truth solution.  }
\label{wind_evolution}
\end{figure} 

\begin{table}[h!]
\caption{Statistics ($\mu \pm \sigma$) of the mean squared error (MSE) for different methods (Wind velocity reconstruction)}
\centering
\begin{tabular}{|c | c | c | c|c|}
        \hline
    Method &  MSE  & MSE & Best  error\\ 
     &  ($S=1000$) & ($S=5000$) & $S=1000 \ \ \ \vline \ \ S=5000$\\ \hline
     {\bf{  Proposed (II)}}   &  $16.57 \pm 2.26$ & $5.32 \pm 0.97$& 13.01  \vline \  \ 3.97\\ \hline
          {\bf{  Proposed (I)}} & ${\bf{16.01}} \pm 2.90$ & 
${\bf{4.89}} \pm 1.23$&{\bf{12.40}}\ \vline \ \ {\bf{3.77}}\\ \hline
   Random layer insertion (I)   &  $17.85 \pm 4.17$ & $5.84 \pm 1.41$& 13.29   \vline \ \ 4.68 \\ \hline
   Net2DeeperNet (II) \cite{chen2015net2net}  & $19.10 \pm 4.15$ & $5.70 \pm 1.86$& 13.28  \vline \ \ 4.60  \\ \hline
     Baseline  network  & $18.96\pm 4.62$ & $5.53 \pm 2.02$&  13.24  \vline \ \ 4.53\\ \hline
Forward Thinking \cite{hettinger2017forward} &  $87.90 \pm 17.64$ &   $20.59 \pm 10.23$&48.15  \ \vline \  15.23 \\ \hline
NAS \cite{li2020system,li2020random} &  $-- \pm --$ &   $-- \pm --$&  14.12  \vline \ \ 4.35\\ \hline
	\end{tabular} 
 \label{stat_wind}
\end{table}

\begin{comment}
\begin{table}[h!]
\caption{Statistics of the mean squared error for different methods (Wind velocity reconstruction)}
\centering
\begin{tabular}{|c | c | c | c|c|}
        \hline
    Method &  Avg. mse ($\mu \pm \sigma$) & Average mse ($\mu \pm \sigma$)  & $\%$ improvement of\\ 
     &  ($S=1000$) & ($S=5000$) & proposed over others\\ \hline
          {\bf{  Proposed }} & $16.01 \pm 2.90$ & 
$5.20 \pm 1.36$&$S=1000\ \ S=5000$\\ \hline
   Random layer insertion (I)   & $17.85 \pm 4.17$ & $5.42 \pm 1.24$& 10.3 $\%$   \vline \ \ 4.05 $\%$\\ \hline
   Net2DeeperNet (II) \cite{chen2015net2net}  & $19.10 \pm 4.15$ & $5.70 \pm 1.86$& 16.2 $\%$   \vline \ \ 8.77 $\%$\\ \hline
     Baseline  network  & $18.96\pm 4.62$ & $5.53 \pm 2.02$&  15.5 $\%$   \vline \ \ 5.96 $\%$\\ \hline
Forward Thinking \cite{hettinger2017forward} &  $87.90 \pm 17.64$ &   $20.59 \pm 10.23$&81.8 $\%$   \vline \ \ 74.7 $\%$\\ \hline
	\end{tabular} 
 \label{stat_wind}
\end{table}

\end{comment}

\begin{figure}[h!]      
  % \begin{subfigure}[b]{\textwidth}
\centering
  \begin{tabular}{c}

 \hspace{-0.5 cm}
      \begin{tabular}{c}

          \centering
          \includegraphics[scale=0.37]{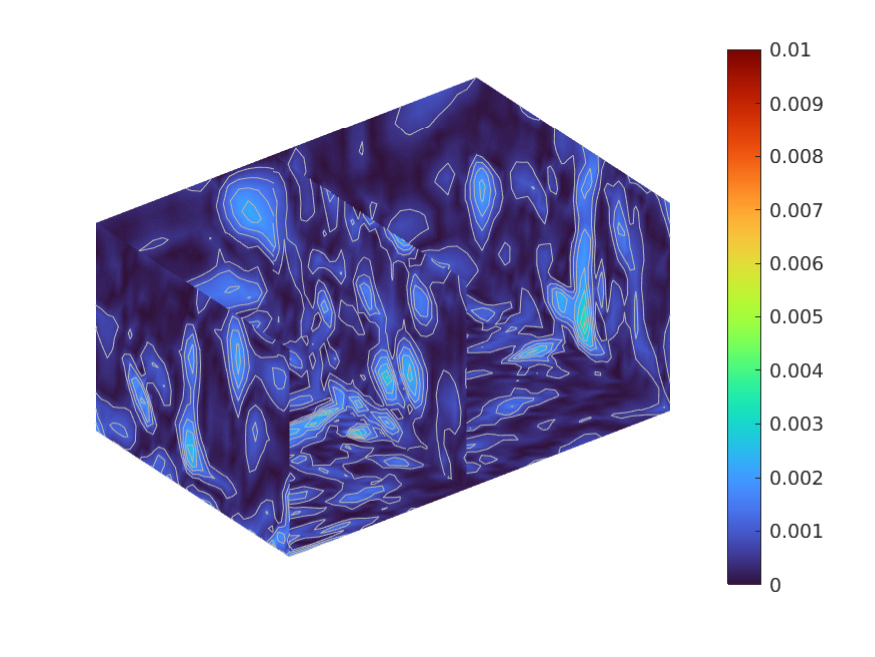}
          % \caption{Effect of noise on DI and Tikhonov solutions}
          % \figlab{some_other_good_name}

      \end{tabular}

   \hspace{-1 cm}

      \begin{tabular}{c}

          \centering
          \includegraphics[scale=0.37]{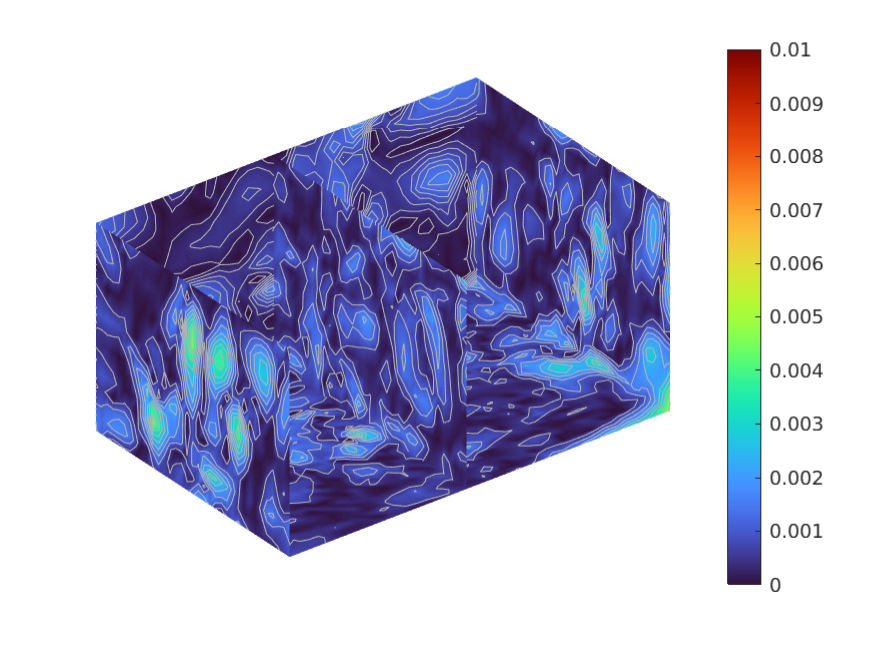}
          % \caption{Effect of noise on DI and Tikhonov solutions}
          % \figlab{some_good_name}

      \end{tabular}

\hspace{-1 cm}

      \begin{tabular}{c}

          \centering
          \includegraphics[scale=0.37]{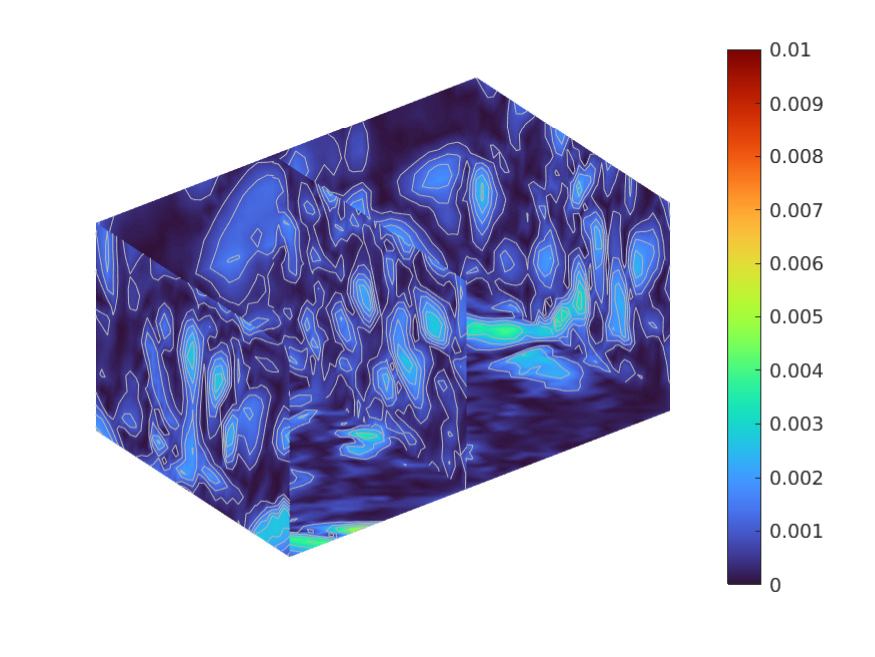}
          % \caption{Effect of noise on DI and Tikhonov solutions}
          % \figlab{some_good_name}

      \end{tabular}

  \end{tabular}\\
    \begin{tabular}{c}

      \begin{tabular}{c}

          \centering
          \includegraphics[scale=0.37]{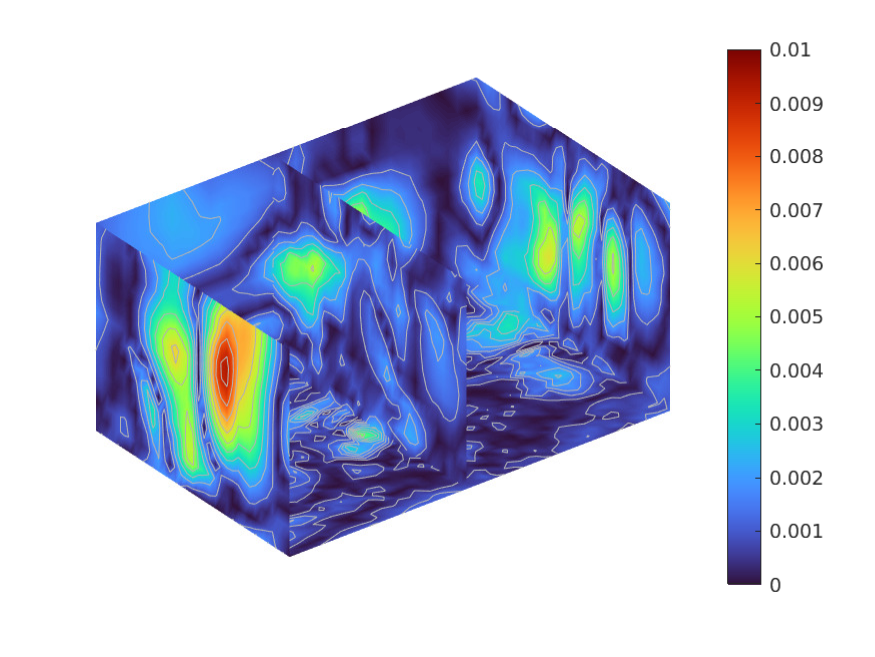}
          % \caption{Effect of noise on DI and Tikhonov solutions}
          % \figlab{some_other_good_name}

      \end{tabular}

   \hspace{0.2 cm}

      \begin{tabular}{c}

          \centering
          \includegraphics[scale=0.37]{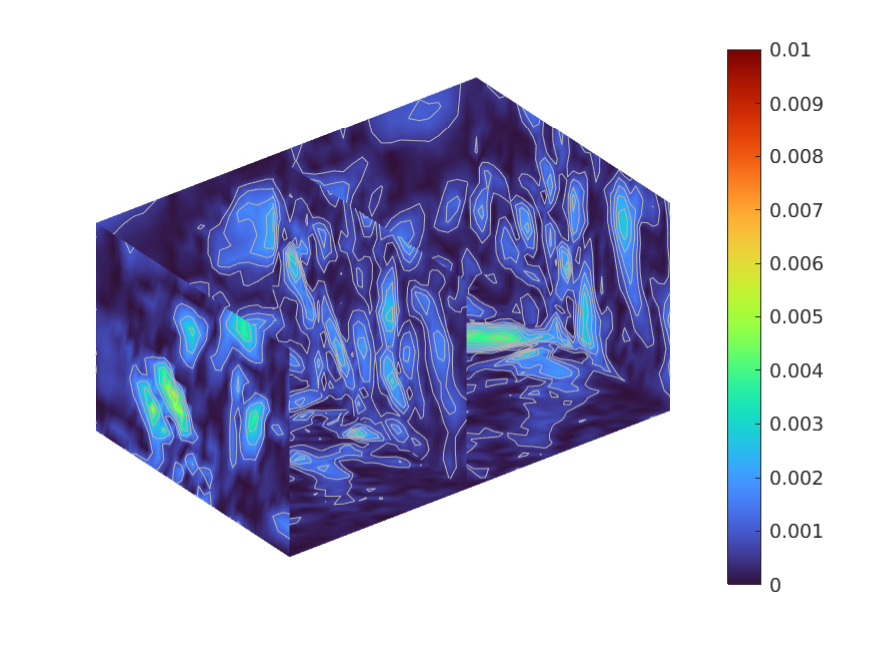}
          % \caption{Effect of noise on DI and Tikhonov solutions}
          % \figlab{some_good_name}

      \end{tabular}
 \end{tabular}
  \caption{ Relative error in the predicted 3D air current profile over the spatial domain ($S=1000$). Left to right: Proposed method (I); Random layer insertion (I); Net2DeeperNet (II) \cite{chen2015net2net}; Forward Thinking \cite{hettinger2017forward}; Baseline. }
  \label{rel_error_wind}

\end{figure} 

\subsection{Image classification with convolutional neural network}

Finally, we test the performance of our adaptation scheme on a CNN architecture for an image classification task. In particular, we consider experiments on 
the MNIST handwritten dataset. Our architecture consists of an upsampling layer that maps the input image to $m$ channels, a sequence of residual layers that we adapt, a downsampling layer that maps $m$ channels to one channel, and a fully connected layer that outputs the probability vector to classify digits.  Other details on the
architecture are provided in  \cref{Input_values_different_problems}.   A summary of results for MNIST dataset (for two different sizes) is provided in \cref{stat_mnist}. Even in this case, we observe that our algorithm significantly outperforms other methods in the low data regime as evident from the maximum accuracy achieved in  \cref{stat_mnist}. However, when looking at the performance achieved by training on the full dataset we do not see considerable advantages over other approaches.
\begin{table}[h!]
\caption{Statistics ($\mu \pm \sigma$)  of the classification accuracy for different methods (MNIST Classification)}
\centering
\begin{tabular}{|c | c | c | c|c|}
        \hline
    Method &  Accuracy in $\%$ & Accuracy in $\%$  & Best  accuracy\\ 
     &  ($S=600$) & ($S=60000$) & $S=600 \ \ \ \vline  \ S=60000$\\ \hline
     {\bf{  Proposed (II)}}   &  $86.52  \pm 0.53$ & $98.90 \pm 0.07$& $87.40$  \vline \  98.95\\ \hline
          {\bf{  Proposed (I)}} & ${\bf{86.98}} \pm 0.73$ 
&$98.92 \pm 0.06$& ${\bf{88.44}}$ \vline \  $98.99$\\ \hline
   Random layer insertion (I)   & $86.78 \pm 0.62$ & $98.90 \pm 0.06$&$88.00$  \vline \ $98.95$\\ \hline
   Net2DeeperNet (II) \cite{chen2015net2net}  & $86.76 \pm 0.71$ & $98.92 \pm 0.07$& $87.64$   \vline \ $98.99$  \\ \hline
     Baseline  network  & $86.78 \pm 0.64$ & ${\bf{98.98}} \pm 0.06$&  $87.81$   \vline \  $99.05$\\ \hline
Forward Thinking \cite{hettinger2017forward} &  $86.60 \pm 0.85$ &   $98.90 \pm 0.08$ & $87.66$   \vline \  $98.95$ \\ \hline
NAS \cite{li2020system,li2020random} &  $-- \pm --$ &   $-- \pm --$&  $88.00$  \vline \  ${\bf{99.07}}$\\ \hline
	\end{tabular} 
\label{stat_mnist}
\end{table}

\subsection{Improving performance of pre-trained state-of-the-art machine learning models}
\label{vit_tran}
For discussion in \cref{vision_t}, we consider a ViT model by Google Brain Team with $5.8$ million parameters (vit\textunderscore tiny\textunderscore patch16\textunderscore 384), a schematic of which is shown in  \cref{vit}. The model is trained on ImageNet-21k dataset and fine-tuned on ImageNet-1k.
\begin{figure}[h!]      
  \centering
          \includegraphics[scale=1.2]{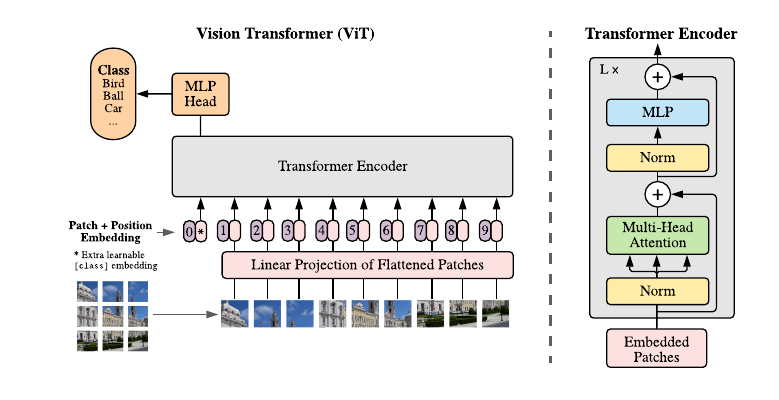}
    \caption{Schematic of a vision transformer (ViT) image classification model (Dosovitskiy et al. \cite{dosovitskiy2020image})}
  \label{vit}
\end{figure}
Note that in \cref{vision_t}, we have only considered adapting the MLP head in  \cref{vit} since it is reasonable to assume that the transformer Encoder block in \cref{vit} has learnt a good representation of the CIFAR-10 input data and it is only necessary to increase the capacity of the classifier (MLP Head) to further improve the classification accuracy. However, it is noteworthy that our approach may be extended to adapting MLP in transformer encoder blocks instead of the MLP Head. In this case one needs to  make decisions on the following aspects: (i) out of different encoder blocks, which encoder block should be adapted; (ii) for the selected encoder block, where to add a new layer within the MLP; and (iii) how to initialize the added layer.  This multi-level adaptation strategy is beyond the scope of the present work and will be investigated in the future.

\paragraph{Possibility of adapting the attention module in a transformer architecture}
{
The attention mechanism in transformers relies on computing the Query (${\boldsymbol{Q}}$), Key (${\boldsymbol{K}}$), and Value (${\boldsymbol{V}}$)  matrices by applying three separate linear projections to an input sequence $\boldsymbol{X}\in \real^{N\times D}$, where $N$ is the number of vectors in the sequence and $D$ is the embedded dimension of the vector. Specifically, one computes ${\boldsymbol{Q}}={\boldsymbol{X}}{\boldsymbol{W}}_q$, ${\boldsymbol{K}}={\boldsymbol{X}}{\boldsymbol{W}}_k$, ${\boldsymbol{V}}={\boldsymbol{X}}{\boldsymbol{W}}_v$, where ${\boldsymbol{W}}_q,\ {\boldsymbol{W}}_k,\ {\boldsymbol{W}}_v\in \real^{D\times D}$ are the learnable parameters of the attention module. The output representation is then computed as: $\mathrm{softmax}\LRp{{\boldsymbol{Q}}{\boldsymbol{K}^T}/\sqrt{D}}{\boldsymbol{V}}$, where $\mathrm{softmax}(.)$ operation is applied row-wise, normalizing each row of the resulting matrix so that it sums to $1$.}

% {\it{Since this mechanism relies on linear transformations of the input, adapting  ${\boldsymbol{Q}}, \ {\boldsymbol{K}},\   {\boldsymbol{V}}  $ projections by inserting additional non-linear layers is not a meaningful strategy; therefore, our proposed procedure is not directly applicable to the attention module. }}
{{\em Since our approach is, after all, a sensitivity method, it can be applied to any architecture, including transformer, as long as a user-specific update strategy needs sensitivity. For transformer, one can derive the sensitivy of the loss function with respect to  (${\boldsymbol{Q}}$), (${\boldsymbol{K}}$), and  (${\boldsymbol{V}}$). Together with the user-specific update choice, we can adapt (${\boldsymbol{Q}}$), (${\boldsymbol{K}}$), and  (${\boldsymbol{V}}$), respectively.}}

{On the other hand, if one considers a nonlinear modification to the attention module such as modifying the Query  as $\tilde{\bQ}=f_{\btheta}\LRp{\bQ}$, where $f_{\btheta}(.)$ is an MLP with residual connections (with all layers having the same  dimension), our proposed approach could potentially be used to add additional layers to the MLP during the adaptation process. The  procedure requires detailed investigation with user-specific choice of adaptivity.}

\paragraph{Possibility of layers that change the dimensionality of the network features}
{Note that for the MLP in each encoder block of a ViT (see \cref{vit}), one could consider a propagation of the form ${\boldsymbol{y}_1}={\boldsymbol{P}_1}({\boldsymbol{Q}_1}{\boldsymbol{x}}+{\boldsymbol{g}}({\boldsymbol{Q}_1}{\boldsymbol{x}}))$, where ${\boldsymbol{Q}_1}$ is the up-sampling/downsampling in the input layer of MLP, ${\boldsymbol{g}}(.)$ is the nonlinear function acting componentwise on the input ${\boldsymbol{Q}_1}{\boldsymbol{x}}\in \mathbb{R}^k$, and ${\boldsymbol{P}_1}$ denotes the upsampling/downsampling layer (the output layer of MLP). For this two hidden layer MLP, it is clear that the hidden dimension is $k$. Now the output ${\boldsymbol{y}}_1$ can be used as an input to a new MLP block with hidden dimension $l$  whose output is given as ${\boldsymbol{y}_2}={\boldsymbol{P}_2}({\boldsymbol{Q}_2}{\boldsymbol{y}_1}+{\boldsymbol{g}}({\boldsymbol{Q}_2}{\boldsymbol{y}_1}))$, where ${\boldsymbol{Q}_2}{\boldsymbol{x}}\in \mathbb{R}^l$. Consequently, we obtain two MLP blocks with different hidden dimensions.  Our approach could then be used to address the following questions: (i) which MLP block should be adapted; (ii) for the selected MLP block, where to add a new layer; and (iii) how to initialize the added layer. A detailed investigation of this multi-level adaptation strategy is left for future work.}

\subsection{Topological derivative informed transfer learning approach: Application in parameter efficient fine tuning}
\label{param_eff}

In order to further back our claim that topological derivative is indeed a good indicator for determining where to add a new layer (refer \cref{sci_tran}), we consider adding a new layer at different locations and the results are tabulated in \cref{correlation}. We found that the computed topological derivative correlates well with the observed mean squared error
and the best performance is achieved by adding a layer at the location of the highest topological derivative. Interestingly, we observe that retraining the first hidden layer yields the best results, contrary to traditional transfer learning, where the last few layers are typically retrained. Other details of hyperparameters used for the problem are provided in \cref{hyper_parameter_n}.

\begin{table}[h!]
\caption{Correlation of topological derivative with mean squared error achieved by proposed method}
\centering
\begin{tabular}{|c | c | c | c}
        \hline
   Hidden layer number&  Topological derivative& Mean squared error  \\ 
    (Added layer)  &  &   \\ \hline
         1 & ${\bf{0.974}}$ &  ${\bf{2.576}}$  \\ \hline
          3 & $0.941$ &  $2.716$  \\ \hline
           5 & $0.907$ &  $3.120$  \\ \hline
           7 & $0.851$ &  $3.280$  \\ \hline
           8 & \textcolor{red}{\textbf{0.821}} &  \textcolor{red}{\textbf{3.524}}  \\ \hline
	\end{tabular} 
 \label{correlation}
\end{table}

\end{section}

\begin{section}{Details of hyperparameter values for different problems}{}
   \label{hyper_parameter_n}
   Details of hyperparameters used in  \cref{Algo_full} and \cref{Algo_full_CNN} is provided in  \cref{Input_values_different_problems}.   \cref{Input_values_different_problems} additionally provides details on the hyperparameters used for the optimizer. 
The description of each problem is also provided below.
\begin{table}[h!]
	\caption{Details of hyperparameters  for \cref{Algo_full}} \label{Input_values_different_problems}
	\renewcommand{\arraystretch}{1.3}
	\centering	
	\begin{tabular}{|c | c | c | c | c | c | c | c | c |c|c|c|c|c|}
 \hline
		\centering
		 Problem &$n$ &$m$ &$N_n$ &   $ T_b$ & $i_s$ ($\%$)& $E_e$ & $\kappa_e$ & $b_s$ & $\ell_r$&  $\epsilon^t$& $\epsilon$& $\tau_1$\\
		\hline
  I (S=1000)& 10 &5 & 7& 2 & $30$ $\%$&2000 & 1000 &1000& 0.001& 0.01& 0.001 & 0.001\\
   I (S=1500)& 10 &5 & 7& 2 & $30$ $\%$ &2000 &1000 &1500& 0.001&  0.01 & 0.001 & 0.001\\ \hline
  II (S=250)& 20 &10 & 6& 2 & $25$ $\%$&2000 & 1000 &125& 0.001& 0.01& 0.001 & 0.001\\
   II (S=500)& 20 &10 & 6& 2 & $25$ $\%$ &2000 &1000 &250& 0.001& 0.01&  0.001& 0.001\\ \hline
  III (S=500)& 20 &10 & 6& 2 & $25$ $\%$&500 & 500 &500& 0.01& 0.01 &0.001& 0.001\\
   III (S=1000)& 20 &10 & 6& 2 & $25$ $\%$ &500 &500 &1000& 0.01&  0.01 & 0.001 & 0.001\\ \hline
  IV (S=1000)& 20 &10 & 6& 2 & $10$ $\%$&500 & 500 &1000& 0.001& 0.01 & 0.001 & 0.001\\ 
   IV (S=5000)& 20 &10 & 6& 2 & $10$ $\%$ &500 &500 &5000& 0.001&  0.01&  0.001 & 0.001\\ \hline
     VI & 10 &5&20& $0$  & $0$ $\%$ &1 &1&512& 0.001&  0.01&0.001 & 0.1\\ \hline
      VII & 50 &50& 1  & 8 &0 $\%$ &2000&0& 20& 0.001& 0.01& 0.001 & 0.001\\ 
  \hline
	\end{tabular} 
\end{table}
In \cref{Input_values_different_problems}, $b_s$ denotes the batch size, $\ell_r$ denotes the learning rate,  and $\sigma_n=0.01$ is chosen as the standard deviation of Gaussian noise for initializing the weights/biases. In order to reduce the number of parameters in the first layer, we introduce the sparsity parameter denoted as $i_s$. In our experiments, $i_s\%$ parameters in the first layer are initialized as zeros and the corresponding connections are removed (or those parameters are untrainable/frozen throughout the procedure). Further, as more layers are added to capture more complex features in the solution, it is also natural to consider training the bigger network for a larger number of epochs. We consider the following scheduler for choosing the training epochs at each iteration,
$E(i)=E_e+\kappa_e \times (i-1)$, where $i$ denotes the $i^{th}$ iteration of the algorithm. For  \cref{Algo_full_auto}, we use the same hyperparameters in  \cref{Input_values_different_problems}, \cref{Input_values_CNN} except that no scheduler is employed (i.e $E_e,\ \kappa_e$ is not employed) and we do not fix the parameter $m$. Further, the parameter $N_k$ in  \cref{Algo_full_auto} is chosen as $100$ for all problems. For MNIST classification problem we fix $N_k$ as 20.
\begin{table}[h!]
	\caption{Details of hyperparameters  for \cref{Algo_full_CNN}} \label{Input_values_CNN}
	\renewcommand{\arraystretch}{1.3}
	\centering	
	\begin{tabular}{|c | c | c | c | c | c | c | c | c |c|c|c|c|c|c|}
 \hline
		\centering
		 Problem &$n$ &$m$ &$N_n$ & $T_b$ &  $f$ & $ s$& $E_e$ & $\kappa_e$& $b_s$& $\ell_r$&  $\epsilon^t$ & $\epsilon$& $\tau_1$ \\
		\hline
  V  (S=600)& $10$ &$5$ & $6$& $2$ & $5$ &$1$ & $15$ & 0& $600$&$0.001$& 0.01& 0.0001& 0.00001\\
  V  (S=60000)& $10$ &$10$ & $7$& $6$ & $5$ &$1$ & $5$ &0 & $600$&$0.001$&  0.01& 0.0001& 0.00001\\ \hline
	\end{tabular} 
\end{table}

\begin{equation*}
 \begin{aligned}
% \text{I}&: \text{1D regression with RBF neural network}.\\
\text{I}&: \text{Learning the observable to parameter map for 2D heat equation}.\\
\text{II }&: \text{Learning the observable to parameter map for Navier Stokes equation}.\\
\text{III }&: \text{Regression on  California housing dataset.}\\
\text{IV }&: \text{Wind velocity reconstruction problem.}\\
\text{V }&: \text{Image classification on MNIST dataset with CNN architecture.}\\
\text{VI }&: \text{Transfer learning for vision transformer (\cref{vision_t}).}\\
\text{VII }&: \text{Transfer learning involving scientific data-set ( \cref{sci_tran}).}
%\text{VI }&: \text{Image classification on CIFAR10 dataset with CNN architecture.}\\
\end{aligned}  
\end{equation*}

\end{section}

\end{document}